\documentclass[twoside]{article}

\usepackage[accepted]{aistats2022}
\usepackage{natbib}
\usepackage[utf8]{inputenc}
\usepackage[ruled,linesnumbered]{algorithm2e}
\usepackage{pdflscape}
\usepackage{fullpage}
\usepackage[noend]{algpseudocode}
\usepackage{amsmath,amsthm,amsfonts,amssymb}
\usepackage{hyperref}
\usepackage{color}
\usepackage{mathrsfs}
\usepackage{enumitem}
\usepackage{bm}
\usepackage{multirow}
\usepackage{booktabs}
\usepackage{makecell}
\usepackage{subfigure}
\usepackage{caption}
\usepackage{thm-restate}
\usepackage{cite}
\usepackage{xcolor}
\definecolor{light-gray}{gray}{0.9}
\usepackage{rotating} 
\usepackage{natbib}
\usepackage{diagbox}

\renewcommand{\epsilon}{\varepsilon}

\newcommand{\trans}{^{\top}}

\newcommand{\cA}{\mathcal{A}}
\newcommand{\cS}{\mathcal{S}}

\newcommand{\cL}{\mathcal{L}}
\newcommand{\ncL}{\nabla\mathcal{L}}
\newcommand{\cP}{\mathcal{P}}
\newcommand{\EE}{\mathbb{E}}

\newcommand\numberthis{\addtocounter{equation}{1}\tag{\theequation}}

\let\hat\widehat
\let\tilde\widetilde

\newtheorem{theorem}{Theorem}[section]
\newtheorem{lemma}[theorem]{Lemma}

\newtheorem{remark}[theorem]{Remark}

\newtheorem{property}[theorem]{Property}

\newtheorem{proposition}[theorem]{Proposition}
\theoremstyle{definition}
\newtheorem{definition}[theorem]{Definition}
\newtheorem{condition}[theorem]{Condition}

\newtheorem{assumption}{Assumption}

\newcount\Comments  
\newcommand{\vect}[1]{\ensuremath{\mathbf{#1}}}

\DeclareMathOperator*{\argmax}{arg\,max}
\DeclareMathOperator*{\argmin}{arg\,min}

\newcommand{\mat}[1]{\ensuremath{\mathbf{#1}}}
\newcommand{\A}{\mat{A}}
\newcommand{\B}{\mat{B}}
\newcommand{\X}{\mat{X}}
\newcommand{\BoldS}{\mat{S}}
\newcommand{\Y}{\mat{Y}}
\newcommand{\Z}{\mat{Z}}
\newcommand{\I}{\mat{I}}
\newcommand{\K}{\mat{K}}
\newcommand{\M}{\mat{M}}

\newcommand{\bP}{\mat{P}}
\newcommand{\bLambda}{\Lambda}

\newcommand{\hK}{\hat{\K}}
\newcommand{\hM}{\hat{\M}}

\newcommand{\cW}{{\mathcal{W}}}
\newcommand{\cQ}{{\mathcal{Q}}}
\newcommand{\nt}{\nabla_{\theta}}

\newcommand{\nz}{\nabla_{\zeta}}
\newcommand{\nx}{\nabla_{\xi}}

\newcommand{\et}{\eta_{\theta}}
\newcommand{\ez}{\eta_{\zeta}}
\newcommand{\ex}{\eta_{\xi}}
\newcommand{\gt}{g_{\theta}}

\newcommand{\Czx}{C_{\zeta,\xi}}
\newcommand{\Lzx}{L_{\zeta,\xi}}

\newcommand{\expect}[1]{\EE[#1]}
\newcommand{\Expect}[1]{\EE\Big[#1\Big]}

\newcommand{\bL}{\bar{L}}
\newcommand{\dia}{d}

\newcommand{\pt}{\pi_\theta}
\newcommand{\wz}{w_\zeta}
\newcommand{\qx}{Q_\xi}

\newcommand{\mR}{\mathbb{R}}
\newcommand{\mz}{\mu_\zeta}
\newcommand{\mx}{\mu_\xi}

\newcommand{\cF}{\mathcal{F}}

\newcommand{\rc}{c}
\newcommand{\batch}{N}
\newcommand{\dims}{d}

\newcommand{\bphi}{\bm{\phi}}
\newcommand{\bPhi}{\bm{\Phi}}

\newcommand{\eig}{v}
\newcommand{\uR}{\vect{u}_R}
\newcommand{\unu}{\vect{u}_\nu}
\newcommand{\huR}{\hat{\vect{u}}_R}
\newcommand{\hunu}{\hat{\vect{u}}_\nu}

\newcommand{\rad}{R}

\newcommand{\vk}{\mat{v}_k}
\newcommand{\uk}{\mat{u}_k}
\newcommand{\vmat}{\mat{v}}
\newcommand{\umat}{\mat{u}}

\newcommand{\ntz}{\nabla_{\theta,\zeta}}

\allowdisplaybreaks

\begin{document}

%

%

\twocolumn[
  \aistatstitle{On the Convergence Rate of Off-Policy Policy Optimization Methods with Density-Ratio Correction}

  \aistatsauthor{ Jiawei Huang \And Nan Jiang }

  \aistatsaddress{ 
  Department of Computer Science \\
  University of Illinois at Urbana-Champaign \\
  Urbana, IL 61801 \\
  \texttt{jiaweih@illinois.edu} 
  \And Department of Computer Science \\
  University of Illinois at Urbana-Champaign \\
  Urbana, IL 61801 \\
  \texttt{nanjiang@illinois.edu}} 
]

\begin{abstract}
In this paper, we study the convergence properties of off-policy policy optimization algorithms with state-action density ratio correction under function approximation setting, where the objective function is formulated as a max-max-min problem.
We first clearly characterize the bias of the learning objective, and then present two strategies with finite-time convergence guarantees.
In our first strategy, we propose an algorithm called P-SREDA with convergence rate $O(\epsilon^{-3})$, whose dependency on $\epsilon$ is optimal.
Besides, in our second strategy, we design a new off-policy actor-critic style algorithm named O-SPIM. We prove that O-SPIM converges to a stationary point with total complexity $O(\epsilon^{-4})$, which matches the convergence rate of some recent actor-critic algorithms under on-policy setting.
\end{abstract}

\section{INTRODUCTION}
Policy improvement is a popular class of methods in empirical reinforcement-learning (RL) research, and has attracted significant attention from the theoretical community recently \citep{agarwal2019optimality}, with many results analyzing the convergence properties of policy gradient-style algorithms
\citep{xu2019sample, xu2019improved, yuan2020stochastic, pmlr-v119-huang20a} and actor-critic algorithms \citep{fu2020singletimescale,wu2020finite}.
Most existing results require on-policy roll-outs, which are not available in offline RL, a paradigm considered crucial to applying RL in real-world problems \citep{levine2020offline}.  
While offline policy optimization algorithms also exist \citep{liu2019OPPG, ImaniGW18, zhang2019provably}, the understanding on their finite-time convergence property is still limited.


To close this gap, in this paper, we propose two strategies for off-policy policy optimization with convergence guarantess, based on recent development in marginalized importance sampling (MIS) methods \citep{liu2018breaking, zhang2019gendice, uehara2019minimax, yang2020lagrangian}. We assume the agent can only get access to a fixed dataset $D$ collected by some unknown policies, and formulate the off-policy learning problem as a max-max-min objective function below (similar objectives have been considered by \citet{nachum2019algaedice, jiang2020minimax, yang2020lagrangian} without convergence rate analyses):
\begin{align*}\label{eq:problem}
    &\max_{\pi\in\Pi} \max_{w\in\cW} \min_{Q\in\cQ} \cL^D(\pi,w,Q)\\
    :=&\max_{\theta\in \Theta} \max_{\zeta \in Z} \min_{\xi \in \Xi} \cL^D(\pt,\wz,\qx)\\
    :=& (1-\gamma) \EE_{s_0\sim \nu_D}[\qx(s_0,\pi_\theta)] \\
    &+ \EE_{d^D}[\wz(s,a)\Big(r+\gamma \qx(s',\pi_\theta)-\qx(s,a)\Big)]\\
    &+\frac{\lambda_Q}{2} \EE_{d^D}[\qx^2(s,a)]-\frac{\lambda_w}{2} \EE_{d^D}[\wz^2(s,a)].\numberthis
\end{align*}
In this objective, we optimize a parameterized policy $\pi_\theta\in\Pi$ with $\theta\in\Theta$ being its parameters, and the policy class $\Pi$ can be non-convex. We do so with the help of linearly parameterized functions $w_\zeta\in\cW$ and $Q_\xi\in\cQ$, which are respectively parameterized by $(\zeta,\xi)\in Z\times\Xi$ and serve as approximators of the density ratio and the value functions. 
We assume the parameter spaces $\Theta\in \mR^{{\rm dim}(\Theta)},Z\subset \mR^{{\rm dim}(Z)}$ and $\Xi\subset \mR^{{\rm dim}(\Xi)}$ are all convex sets. $\nu_D$ and $d^D$ are the empirical approximations of the initial state distribution and the state-action distribution in the data; see Sec.~\ref{sec:MDP} for a formal definition.  $\qx(s,\pi_\theta)$ is short for $\EE_{a\sim \pi_\theta(\cdot|s)}[\qx(s,a)]$. \footnote{In this paper, we only give analysis for $\max_{\pi\in\Pi}\max_{w\in\cW}\min_{Q\in\cQ}\cL_D(\pi, w, Q)$. In fact, we may swap the role of $w$ and $Q$ function and obtain another objective function $\max_{\pi\in\Pi}\max_{Q\in\cQ}\min_{w\in\cW}\cL'_D(\pi, w, Q)$, where $\cL'_D$ is the same as $\cL$ except that both $\lambda_Q$ and $\lambda_w$ are negative. 
In general, the solutions of Eq.\eqref{eq:problem_general} and those of this new objective are not the same (see \citep{jiang2020minimax} for their connections), but we can give a similar anlysis for $\cL'$ with the techniques in this paper and establish similar convergence results.}

The main contributions of this paper is three-fold, which we summarize below:
\paragraph{A Detailed Analysis of Bias}
We identify the inevitable bias between the stationary points of $\cL^D(\pi_\theta, w_\zeta, Q_\xi)$ and $J(\pi_\theta)$, where $J(\pi_\theta)$ is the expected return of $\pi_\theta$. We separate out and characterize the bias terms due to regularization ($\epsilon_{reg}$), generalization ($\epsilon_{func}$), and mis-specification ($\epsilon_{data}$), as will be detailed in Sec.~\ref{sec:biased_stationary_points}. Besides, our analysis also reveals how $\lambda_Q$ and $\lambda_w$ affect the trade-off between the convergence speed and the magnitude of bias.
Since the bias is unavoidable, we focus on the convergence to the points satisfying the following inequality, which we call the \emph{biased stationary points} of $J(\pi_\theta)$ up to $\epsilon$ error.
\begin{equation} \label{eq:biased_cond}
    \EE[\|\nt J(\pi_{\theta})\|]\leq \epsilon + \epsilon_{data}+ \epsilon_{func}+\epsilon_{reg}
\end{equation}
where $O(\cdot)$ only suppresses absolute constants. \footnote{All norms $\|\cdot\|$ in this paper is $\ell_2$ norm for vectors and operator norm for matrices unless specified otherwise. The expectation is over the randomness of the algorithm (e.g., the randomness in SGD) and not that of the data.} The rest of the paper then presents two optimization strategies for solving Eq.\eqref{eq:problem} and provides their convergence analyses.

\paragraph{First Strategy: P-SREDA}
In our first strategy, we convert the original max-max-min problem to a standard non-convex-strongly-concave problem $\min_{(\theta,\zeta)\in\Theta\times Z} \max_{\xi\in\Xi} -\cL (\pi_\theta, w_\zeta, Q_\xi)$ by simultaneously optimizing $\theta$ and $\zeta$ in the outer min.
Unfortunately, most of the existing non-convex-strongly-concave optimization algorithms can not be adapted directly here, because they focus on $\min_{x\in\mR^d}\max_{y\in \mathcal{Y}} f(x,y)$ where the first player can play an arbitrary vector in $\mR^d$, while our objective function requires boundedness of $\zeta$ and $\xi$ 
in order to guarantee the smoothness of $\cL^D$, as we will show in the proof of Property \ref{prop:detailed_scscs} in Appendix \ref{appx:properties}. 
To tackle this challenge, we propose P-SREDA, which is adapted from the SREDA algorithm \citep{luo2020stochastic} by including a projection step every time after updating $\zeta$. 
The proof of P-SREDA is non-trivial because the projection step will incur additional error and some crucial steps in the original proof no longer hold in our case. We overcome these difficulties by leveraging the properties of $\cL^D$ and carefully choosing the projection sets for $Z$ and $\Xi$, and prove that its convergence rate remains $O(\epsilon^{-3})$, which matches the lower bound in non-convex optimization \citep{arjevani2019lower}.

\paragraph{Second Strategy: O-SPIM}
In our second strategy for solving Eq.\eqref{eq:problem}, we study a novel actor-critic style framework called O-SPIM (\textbf{O}racle-based \textbf{S}tochastic \textbf{P}olicy \textbf{I}mprovement with \textbf{M}omentum). In O-SPIM, we alternate between updating $\zeta$ and $\xi$ (the critic update) and updating $\theta$ with momentum to achieve a better trade-off (the actor update).
The main technique difficulty is that  
different from policy gradient and actor-critic algorithms in the on-policy setting, here we need to coordinate the updates of three objects $\theta,\zeta$ and $\xi$, 
and the loss $\cL^D(\pi_{(\cdot)}, w_{\zeta_t}, Q_{\xi_t})$---when viewed solely as a function of the policy parameter $\theta$ with fixed parameters for $w$ and $Q$---varies across iterations because $\zeta_t, \xi_t$ are updated in the critic step. 
We handle this difficulty by considering a family of critic update rules that satisfy a general condition (Condition \ref{def:oracle_alg}), which enables us to relate the variations $\|\zeta_{t+1}-\zeta_{t}\|$ and $\|\xi_{t+1}-\xi_{t}\|$ with $\|\theta_{t+1}-\theta_t\|$ and is crucial to establishing convergence guarantee. 
We use \textsc{Oracle} to refer to all critic update subroutines satisfying such a condition, 
and present two concrete instantiations 
and analyze their convergence properties in Appendices \ref{appx:PLSO} and \ref{appx:SVRE_Oracle}. The first one is a least-square algorithm, and the second one is a first-order algorithm by extending the SVRE \citep{chavdarova2019reducing} from finite-sum setting to stochastic setting. 
Using either instantiation (or possibly other subroutines that satisfy our conditions), the convergence rate 
of our second strategy is $O(\epsilon^{-4})$, which matches the rate of recent actor-critic algorithms in the on-policy setting \citep{xu2020nonasymptotic, fu2020singletimescale}


\subsection{Related works}
Along with the progress of variance reduction techniques for non-convex optimization, there are emerging works analyzing convergence rates for on-policy policy gradient methods \citep{papini2018stochastic, xu2019improved, xu2019sample, yuan2020stochastic, Huang2020MomentumBasedPG}. Besides, in another line of work \citep{fu2020singletimescale, wu2020finite, hong2020twotimescale,xu2020nonasymptotic}, finite-time guarantees are established for actor-critic style algorithms. 
However, all of these works require online interaction with the environment, whereas our focus is the offline setting.

Turning to the offline setting, there have been substantial empirical developments and success for off-policy PG or actor-critic algorithms \citep{DegrisOPAC, lillicrap2015continuous, haarnoja2018soft, fujimoto2018addressing}, but 
most of these works ignore the distribution-shift issue and do not provide convergence guarantees. 
Recently, there has been a lot of interest in MIS methods for off-policy evaluation \citep{liu2018breaking, uehara2019minimax,zhang2019gendice, NachumCD019} and turning them into off-policy policy-optimization algorithms. Among them,  \citet{liu2019OPPG} presented OPPOSD with convergence guarantees, but the convergence relies on accurately estimating the density ratio and the value function via MIS, which were treated as a black box without further analysis. \citet{nachum2019algaedice,jiang2020minimax} discussed policy optimization given arbitrary off-policy dataset, but no convergence analysis was performed. 
Besides, \citep{zhang2019provably} presented a provably convergent algorithm under a similar linear setting with emphatic weighting \citep{ImaniGW18}, but only showed asymptotic convergence behavior and did not establish finite convergence rate.

In another work concurrent to ours \citep{xu2021doubly}, the authors designed DR-Off-PAC, which is motivated by the doubly-robust estimator and their objective function is similar to ours. 
However, they fixed the coefficient of regularization terms and their bias analysis is much coarser than ours.
As a result, their analysis cannot distinguish between the errors resulting from
regularization and mis-specification, and do not characterize how the regularization weights affect the trade-off between bias and convergence speed. Besides, the convergence rate of their algorithm is $O(\epsilon^{-4})$, which matches our second strategy but worse than our first strategy.
Recently,  \citep{lyu2020variancereduced} developped VOMPS/ACE-STORM based on Geoff-PAC \citep{zhang2019generalized} and the STORM estimator \citep{cutkosky2020momentumbased}. However, as also pointed by \citep{xu2021doubly}, the unbiasness of their estimator only holds asymptotically, and the convergence property of their algorithms is still unclear. 

We summarize the comparison to closely related works in Table \ref{tab:comparison} in Appendix.



\section{PRELIMINARY}\label{sec:prelimiary}

\subsection{Markov Decision Process}\label{sec:MDP}
We consider an infinite-horizon discounted MDP $(\cS, \cA, R, P, \gamma, \nu_0)$, where $\cS$ and $\cA$ are the state and action spaces, respectively, which we assume to be finite but can be arbitrarily large. $R:\cS\times\cA\rightarrow \Delta([0, 1])$ is the reward function. $P:\cS\times\cA\rightarrow \Delta(\cS)$ is the transition function, $\gamma$ is the discount factor and $\nu_0$ denotes the initial state distribution. 

Fixing an arbitrary policy $\pi$, we use $d^\pi(s,a)=(1-\gamma)\EE_{\tau\sim \pi, s_0\sim \nu_0}[\sum_{t=0}^\infty \gamma^t p(s_t=s,a_t=a)]$ to denote the normalized discounted state-action occupancy, where $\tau\sim\pi,s_0\sim \nu_0$ means a trajectory $\tau=\{s_0,a_0,s_1,a_1,...\}$ is sampled according to the rule that $s_0\sim \nu_0,a_0\sim\pi(\cdot|s_0),s_1\sim P(\cdot|s_0,a_0),a_1\sim \pi(\cdot|s_1),...$, and $p(s_t=s,a_t=a)$ denotes the probability that the $t$-th state-action pair are exactly $(s,a)$. 
We also use $Q^\pi(s,a)=\EE_{\tau\sim \pi, s_0=s,a_0=a}[\sum_{t=0}^\infty \gamma^t r(s_t,a_t)]$ to denote the Q-function of $\pi$. It is well-known that $Q^\pi$ satisfies the Bellman Equation: 
\begin{align*}
Q^\pi(s,a)=&\mathcal{T}^\pi Q^\pi(s,a):=\\
&\EE_{r\sim R(s,a), s'\sim P(\cdot|s,a),a'\sim \pi(\cdot|s')}[r+\gamma Q^\pi(s',a')].
\end{align*}

Define $J(\pi)=\EE_{s\sim \nu_0, a\sim \pi}[Q^\pi(s,a)]$=$\frac{1}{1-\gamma}\EE_{s,a\sim d^\pi}[r(s,a)]$ as the expected return of policy $\pi$. If $\pi$ is parameterized by $\theta$ and differentiable, we have
\begin{align*}
\nt J(\pi_\theta)=&\frac{1}{1-\gamma}\EE_{s,a\sim d^\pi}[Q^\pi(s,a)\nt \log\pi(a|s)]\\
=&\frac{1}{1-\gamma}\EE_{s,a\sim \mu}[w^\pi(s,a)Q^\pi(s,a)\nt \log\pi(a|s)].    
\end{align*}
where the first step is the policy-gradient theorem \citep{NIPS1999_1713}, and in the second step, we replace $d^\pi$ with distribution $\mu$, a state-action distribution generated by some behavior policies, and introduce the density ratio $w^\pi(s,a):=\frac{d^\pi(s,a)}{\mu(s,a)}$.

In the rest of the paper, we assume we are only provided with a fixed off-line dataset $D=\{(s_i,a_i,r_i,s'_i)\}_{i=1}^{|D|}$, where each tuple is sampled according to $s_i,a_i\sim \mu, r_i\sim R(s_i,a_i), s'_i \sim P(\cdot|s_i,a_i)$. Besides, we do not require the knowledge of behavior policies that generate $\mu$. As we also mentioned in the introduction, we will use $d^D$ to denote the empirical state-action distribution induced from the dataset, which is defined by $d^D(s,a)=\frac{1}{|D|}\sum_{(s_i,a_i,r_i,s'_i)\in D}\mathbb{I}[s_i=s,a_i=a]$.

Since the state-action space can be very large, we consider generalization via linear function approximation:
\begin{definition}[Linear function classes]\label{def:linear_func}
Suppose we have two feature maps $\{\bphi_W:\cS\times\cA\to \mR^{{\rm dim}(Z)}\}$ and $\{\bphi_Q:\cS\times\cA\to \mR^{{\rm dim}(\Xi)}\}$ subject to $\|\bphi_w(\cdot, \cdot)\|\le 1, \|\bphi_Q(\cdot, \cdot)\|\leq 1$. Let $Z\subset \mathbb{R}^{{\rm dim}(Z)}, \Xi\subset \mathbb{R}^{{\rm dim}(\Xi)}$ be the corresponding parameter spaces, respectively, satisfying $C_\cW:=\max\{1, \max_{\zeta\in Z}\|\zeta\|\} < \infty$ and $C_\cQ:=\max_{\xi \in \Xi} \|\xi\|<\infty$. The approximated value function $Q_\xi$ and density ratio $w_\zeta$ are represented by
    \begin{align*}
        w(\cdot,\cdot)=\bphi_w(\cdot,\cdot)\trans \zeta,~~~~~Q(\cdot,\cdot)=\bphi_Q(\cdot,\cdot)\trans \xi.
    \end{align*}
\end{definition}

We use $\bPhi_w \in \mathbb{R}^{|S||A|\times {\rm dim}(Z)}$ to denote the matrix whose  $(s,a)$-th row is $\bphi_w(s,a)^\top$, and use $\K_w$ to denote $\bPhi_w\trans \bLambda^D \bPhi_w$, where $\bLambda^D$ is a diagonal matrix whose diagonal elements are $d^D(\cdot,\cdot)$. 
Similarly, we define $\bPhi_Q$ and $\K_Q$ for $\bphi_Q(\cdot,\cdot)$. Besides, $\M_\pi$ denotes $\bPhi_w\trans \bLambda^D (I-\gamma \bP_D^\pi)\bPhi_Q$, where $\bP_D^\pi$ is the empirical transition matrix induced from the data distribution. Notes that if we never see some $s$ in dataset, then the corresponding element in $\bLambda^D$ should be 0, and therefore, the corresponding row in $\bP^\pi_D$ can just be set arbitrarily without having any effects on the loss function. 
Under linear function classes, we can rewrite $\cL^D$ in Eq.\eqref{eq:problem} as:
\begin{align}
\cL^D(\pi, \zeta, \xi)
=& (1-\gamma)(\nu_D^\pi)\trans \bPhi_Q \xi+\zeta\trans \bPhi_w\trans \bLambda^D R\nonumber\\
& -\zeta\trans \M_\pi \xi+\frac{\lambda_Q}{2}\xi\trans \K_Q \xi - \frac{\lambda_w}{2}\zeta\trans \K_w\zeta.\label{eq:linear_formulation}
\end{align}
We will also denote $\Lambda$ as the diagonal matrix whose diagonal elements are $\mu(\cdot,\cdot)$, and denote $P^\pi$ as the transition matrix of $\pi$.

\subsection{Assumptions}\label{sec:assumptions}
Our first assumption is about the policy class, which is practical and frequently considered in the literature \citep{zhang2019provably, wu2020finite}.
In this paper, we choose $\Theta=\mathbb{R}^{{\rm dim}(\Theta)}$ and we will interchangeably use $\Theta$ and $\mR^{{\rm dim}(\Theta)}$ to denote the  space of policy parameter $\theta$.
\begin{assumption}[Smoothness of policy function]\label{assump:bound_pi_grad}\label{assump:smooth}\label{assump:bound_w_Q}$ $
For any $s,a\in\cS\times\cA$ and $\theta\in\Theta=\mathbb{R}^{{\rm dim}(\Theta)}$, $\pi_\theta(s,a)$ is second-order differentiable w.r.t. $\theta$, and there exist  constants $L_\Pi$, $G$ and $H$, s.t. for arbitrary $\theta_1, \theta_2, \theta\in\Theta$:
\begin{align}
    &\|\pi_{\theta_1}(\cdot|s)-\pi_{\theta_2}(\cdot|s)\|_1\leq L_\Pi \|\theta_1-\theta_2\|,\\
    &\|\nt \log\pi_\theta(a|s)\|\leq G,\quad \|\nt^2 \log\pi_\theta(a|s)\|\leq H.
\end{align}
\end{assumption}

Note that this assumption is equivalent to Assump.~4 in \citep{xu2021doubly} and Assump.~3 in \citep{zhang2020provably}.

We also assume $\K_w, \K_Q$ and $\M_\pi$ are non-singular. 
\begin{assumption}\label{assump:feature_matrices}
Denote $v_{\min}(\cdot)$ as the function to return the minimum singular value. We assume that $v_{w} := v_{\min}(\K_w)>0$ and $v_{Q}:=v_{\min}(\K_Q)>0$. Beisdes, there exists a positive constant $v_{M}$ that, for arbitrary $\pi\in\Pi$ (i.e. $\forall \theta\in\Theta$), the matrix $\M_\pi$ is invertible, and $v_{\min}(\M_\pi)\geq v_{\M}>0$.
\end{assumption}
Our assumption on $v_Q$ and $v_w$ is frequently considered in the literature, see e.g.~Assump.~2 in \citet{zhang2020provably}. 
Besides, $v_w$ and $v_Q$ can be computed from a holdout dataset. 
Comparing with \citet{xu2021doubly}, our $M_\pi$ is equivalent to matrix $A$ in their Assump.~3 if we use the same feature to approximate $Q$ and $w$. Although \citet{xu2021doubly} do not consider $v_Q$ and $v_w$, they have additional constraints on another two matrices in their Assump.~3.

The next assumption characterizes the coverage of $\mu$ using a constant $C$, often known as the concentrability coefficient in the literature \citep{szepesvari2005finite, chen2019information}. 
\begin{assumption}[Exploratory Data]\label{assump:dataset}
    Recall the behavior policy is denoted as $\mu$. We assume there exists a constant $C>0$, for arbitrary $\pi\in\Pi$ and any $(s,a)\in\cS\times\cA$, we have
    \begin{align*}
    w^\pi(s,a):=\frac{d^\pi(s,a)}{\mu(s,a)}\leq C,~~~~w^\pi_{\mu}(s,a):=\frac{d^\pi_{\mu}(s,a)}{\mu(s,a)}\leq C
    \end{align*}
    where $
    d^\pi_{\mu}(s,a):=(1-\gamma)\EE_{\tau\sim \pi, s_0,a_0\sim \mu}[\sum_{t=0}^\infty \gamma^t p(s_t=s,a_t=a)]
    $
    is the normalized discounted state-action occupancy by treating $\mu$ as initial distribution.
\end{assumption}
Our Assump.~C is slightly stronger than Assump.~1 in \citet{xu2021doubly} because we have additional constraints on $\|\frac{d_\mu^\pi}{\mu}\|_\infty$, but it is still weaker than assuming the behavior policy is exploration enough, which is used by \citet{zhang2020provably} as their Assump.~1.

Finally, since $\phi_w$ and $\phi_Q$ has bounded norm, the following variance terms are therefore bounded, and we use $\sigma_{(\cdot)}$ to denote their minimal upper bounds: 
\begin{assumption}\label{assump:detailed_variance}[Variance]
There are constants $\sigma_\K, \sigma_\M, \sigma_R$ and $\sigma_\nu$ satisfying
\begin{align*}
    &\EE_{s,a\sim d^D}[\|\K_w - \bphi_w(s,a)\bphi_w(s,a)\trans\|^2]\leq \sigma_{\K}^2,\\ 
    &\EE_{s,a\sim d^D}[\|\K_Q - \bphi_Q(s,a)\bphi_Q(s,a)\trans\|^2]\leq \sigma_{\K}^2,\\
    &\EE_{s,a\sim d^D}[\|\bPhi_w\trans\bLambda^D R - \bphi_w(s,a)r(s,a)\|^2]\leq \sigma_{R}^2,\\ 
    &\EE_{s\sim \nu^D_0,a\sim \pi(\cdot|s)}[\|\bPhi_Q\trans \nu^\pi_D-\bphi_Q(s,a)\|]\leq \sigma^2_{\nu},\\
    &\EE_{s,a,s'\sim d^D,a'\sim\pi(\cdot|s')}[\|\M_\pi - \bphi_w(s,a)(\bphi_Q(s,a)\\
    &\quad\quad\quad\quad\quad-\gamma \bphi_Q(s',a'))\trans\|]\leq \sigma^2_{\M}.
\end{align*}
\end{assumption}

\subsection{Useful Properties}
Under Assumptions above, we have the following important property regarding the convexity, concavity and smoothness of our objective function, which will be central to our proofs. Due to space limits,  we defer a more detailed version (Property \ref{prop:detailed_scscs}) in Appendix \ref{appx:properties}, in addition to some other useful properties and their proofs.
\begin{property}[Convexity, Concavity and Smoothness]\label{prop:scscs}
    For arbitrary $\theta\in\Theta,\zeta\in Z$, $\cL^D$ is $\mu_\xi$-strongly convex w.r.t. $\xi\in \Xi$, and for arbitrary $\theta\in\Theta,\xi\in\Xi$, $\cL^D$ is $\mu_\zeta$-strongly concave w.r.t. $\zeta\in Z$, where $\mu_\xi=\lambda_Q \eig_Q$ and $\mu_\zeta = \lambda_w \eig_w$.
    Besides, given that $C_\cQ$ and $C_\cW$ in Definition \ref{def:linear_func} is finite, $\cL^D$ is $L$-smooth with finite $L$.
\end{property}
In the rest of paper, we will use $\kappa_\zeta:=\lambda_\zeta/L$ and $\kappa_\xi:=\lambda_\xi/L$ as a short note of the condition number.

\section{THE BIASED STATIONARY POINTS}\label{sec:biased_stationary_points}
As alluded to in the introduction, even if we optimize Eq.\eqref{eq:problem}, we do not expect that it converges to the a true stationary point w.r.t.~$J(\pi_\theta)$ due to regularization, function-class mis-specification, and finite-sample effects, which result in the three error terms in Eq.\eqref{eq:biased_cond}. 
More concretely, by the triangular inequality,
\begin{align*}
    &\|\nt J(\pi_\theta)\|\leq\|\nt\max_{w\in\cW}\min_{Q\in \cQ}\cL^D(\pi_\theta, w, Q)\|\\
    &+ \|\nt J(\pi_\theta)-\nt\max_{w\in\cW}\min_{Q\in \cQ}\cL^D(\pi_\theta, w, Q)\|.
\end{align*}
Optimizing our loss function $\cL^D(\pi, w, Q)$ reduces the first term but leaves the second term intact. In the rest of this section we bound the second term by breaking it down to more basic quantities that can be bounded under further standard assumptions. To do so, we need some additional definitions: 
We use $\cL(w, \pi, Q)$ (as well as $\cL(\pi_\theta, w_\zeta, Q_\xi)$) to denote the asymptotic version of $\cL^D$ as $|D|\rightarrow \infty$:
\begin{align*}
    &\max_{\pi\in\Pi} \max_{w\in\cW} \min_{Q\in\cQ} \cL(\pi,w,Q)\\
    :=& \max_{\theta\in \Theta} \max_{\zeta \in Z} \min_{\xi \in \Xi} \cL(\pt,\wz,\qx) \\
    :=& (1-\gamma) \EE_{s_0\sim \nu_0}[\qx(s_0,\pi_\theta)] \\
    &+ \EE_{\mu}[\wz(s,a)\Big(r+\gamma \qx(s',\pi_\theta)-\qx(s,a)\Big)]\\
    &+\frac{\lambda_Q}{2} \EE_{\mu}[\qx^2(s,a)]-\frac{\lambda_w}{2} \EE_{\mu}[\wz^2(s,a)].
    \numberthis\label{eq:problem_general}
\end{align*}

\begin{definition}[Generalization Error]\label{def:gen_error}
We define $\bar{\epsilon}_{data}$ to be the minimal value satisfying that, for arbitrary $\pi_\theta, w_\zeta, Q_\xi\in\Pi\times\cW\times\cQ$, we have:
\begin{align*}
|\cL(\pi_\theta, w_\zeta, Q_\xi)-\cL^D(\pi_\theta, w_\zeta, Q_\xi)|\leq \bar{\epsilon}_{data},\\ 
\|\nt\cL(\pi_\theta, w^*_\mu, Q^*_\mu)-\nt\cL^D(\pi_\theta, w^*_\mu, Q^*_\mu)\|^2\leq \bar{\epsilon}_{data}
\end{align*}
where $(w^*_\mu, Q^*_\mu):=\arg\max_{w\in\cW}\min_{Q\in\cQ}\cL(\pi, w, Q)$.
\end{definition}
$\bar{\epsilon}_{data}$ characterizes the uniform deviation between $\cL$ and $\cL^D$ ($\nt\cL$ and $\nt\cL^D$, resp.) and can be bounded by the sample size and the complexity of the function classes. Such uniform convergence bounds are standard and we do not further analyze it in this paper. 

\begin{definition}[Mis-specification Error]\label{def:misspeci}
Denote $\|\cdot\|_\Lambda$ as the $\ell_2$-norm weighted by $\Lambda$, and define $w^\pi_\cL=\arg\max_{w\in\mR^{|\cS||\cA|}} \min_{Q\in\mR^{|\cS||\cA|}}\cL(\pi, w, Q)$ given arbitrary $\pi\in\Pi$. We define
\begin{align*}
    \epsilon_1:=&\max_{\pi\in\Pi}\min_{w\in \cW}\|w-w^\pi_\cL\|^2_{\Lambda}, \\ 
    \epsilon_2:=&\max_{w\in \cW, \pi\in\Pi}~~\min_{Q\in\cQ}\|Q-\argmin_{Q\in\mR^{|\cS||\cA|}}\cL(\pi, w, Q)\|^2_{\Lambda}
\end{align*}
\end{definition}
Based on these definitions and Assumptions in Section \ref{sec:assumptions}, in the following theorem, we provide an upper bound for the bias between the stationary points of our objective function $\cL^D$ and $J(\pi)$, which is one of our main contributions. The proofs are deferred to Appendix \ref{appx:bound_error}.

\begin{restatable}{theorem}{Biasedness}[Bias]\label{thm:biasedness}
Based Assumptions in Section \ref{sec:assumptions} and Definitions above, for arbitrary $\theta\in\Theta$, we have:
\begin{align*}
    &\|\nt \max_{w\in\cW}\min_{Q\in \cQ} \cL^D(\pi_\theta,w,Q)-\nt J(\pi_\theta)\|\\
    \leq& \epsilon_{reg}+\epsilon_{func}+\epsilon_{data}
\end{align*}
The bias terms $\epsilon_{reg},\epsilon_{func}$ and $\epsilon_{data}$ are defined by 
\begin{align*}
    &\epsilon_{func}:=\frac{G}{1-\gamma}\Big(\sqrt{C\epsilon_\cQ}+  C_\cW\sqrt{\frac{\gamma\epsilon_\cQ C}{1-\gamma}}\\
    &\qquad\qquad+\sqrt{\frac{\gamma\epsilon_\cQ \epsilon_\cW C}{1-\gamma}}+\gamma C_\cQ \sqrt{\epsilon_\cW}\Big)\\
    &\epsilon_{reg}:= \frac{G}{1-\gamma}\Big(\frac{C^2}{(1-\gamma)}(\frac{\lambda_w\lambda_Q }{1-\gamma}+\lambda_w)\\
    &\qquad\quad +\frac{\gamma C(\lambda_Q + \lambda_Q \lambda_w C)}{(1-\gamma)^3}\\
    &\qquad\quad +\frac{C^2 (\lambda_Q + \lambda_Q \lambda_w C)}{(1-\gamma)^3}(\frac{\lambda_w\lambda_Q }{1-\gamma}+\lambda_w)\sqrt{\frac{\gamma C}{1-\gamma}}\Big)\\
    &\epsilon_{data}:=(2\kappa_\zeta\kappa_\xi+2\kappa_\zeta+2\kappa_\xi+\sqrt{2}/2)\sqrt{2\bar{\epsilon}_{data}}
\end{align*}
where $\kappa_\zeta$ and $\kappa_\xi$ is the condition number, and
\begin{align*}
    \epsilon_\cW:=&4\frac{\lambda_{\max}^2}{\lambda_Q\lambda_w}\epsilon_1+2\frac{\lambda_{\max}}{\mu_\zeta}\epsilon_2,\\ 
    \epsilon_\cQ :=&8\frac{\lambda_{\max}^3}{\lambda^2_Q\lambda_w}\epsilon_1+(2+4\frac{\lambda_{\max}^2}{\lambda_Q\mu_\zeta})\epsilon_2
\end{align*}
with $\lambda_{\max}:=\max\{\lambda_Q,\lambda_w\}$.
\end{restatable}

As we can see, $\|\nt \max_{w\in\cW}\min_{Q\in \cQ}\cL^D(\pi_\theta, w, Q)-\nt J(\pi_{\theta})\|$ can be controlled by three terms. $\epsilon_{data}$ reflects the generalization error and decreases with sample size (assuming all function classes have bounded capacity and allow uniform convergence). $\epsilon_{reg}$ depends on the magnitude of regularization, and will decrease if $\lambda_w$ and $\lambda_Q$ decreases. As for $\epsilon_{func}$, it depends on the approximation error $\epsilon_\cW$ and $\epsilon_{\cQ}$, which are proportional to $\epsilon_1$ and $\epsilon_2$. Besides, because $\mu_\zeta$ should be proportional to $\lambda_w$, the coefficients before $\epsilon_1$ and $\epsilon_2$ will not vary significant as we change $\lambda_w$ and $\lambda_Q$, as long as $\lambda_w \approx \lambda_Q$. ($\epsilon_1$ and $\epsilon_2$ themselves may change with $\lambda_w$ and $\lambda_Q$). 
Overall, a large dataset, well-specified function classes, and small $\lambda_w$ and $\lambda_Q$ will result in a small total bias, while small $\lambda_w$ and $\lambda_Q$ can lead to weaker strong-concavity or strong-convexity of the loss function, resulting in slower convergence.

Based on the discussion above, our goal is to find stochastic optimization algorithms, which return a policy $\pi_{\theta}$ after consuming ${\rm poly}(\epsilon^{-1})$ samples from dataset (we omit the dependence on quantities such as $\mz$ and $\mx$), satisfying the biased stationary condition in Eq.\eqref{eq:biased_cond}

\paragraph{Global Convergence Under Polyak-Łojasiewicz Condition}
If we assume the policy function class satisfies certain additional conditions, we can establish the global convergence guarantee. Take the Polyak-Łojasiewicz condition as an example, which requires that, there exists a constant $c_{PŁ}>0$, s.t., for arbitrary $\theta\in\Theta$
\begin{align}
    \frac{1}{2}\|\nabla_\theta J(\pi_\theta)\|^2\geq c_{PŁ} (J^* - J(\pi_{\theta})).
\end{align}
where $J^*$ is the value of the optimal policy. 
As a result, if the policy class satisfies the PŁ inequality, a biased stationary point $\theta$ in Eq.\eqref{eq:biased_cond} indicates that,
\begin{align*}
    \EE[J(\pi_\theta)] \geq J^* - \frac{1}{c_{PŁ}}(\epsilon^2+\epsilon_{data}^2+\epsilon_{func}^2+\epsilon_{reg}^2).
\end{align*}
Such condition is provably weaker than the strongly convexity/concavity \citep{karimi2020linear}.
\section{Strategy 1: Projected-Stochastic Recursive Gradient Descent Ascent}\label{sec:alg:Direct_SGDA}
In our first optimization strategy, we rewrite the original max-max-min problem as a max-min:
\begin{align*}
    &\argmax_{\theta\in\mR^{\dims_\theta}}\max_{\zeta\in Z}\min_{\xi\in\Xi} \cL^D(\theta, \zeta,\xi)\\
    \rightarrow &\argmax_{\theta, \zeta\in\mR^{\dims_\theta}\times Z}\min_{\xi\in\Xi} \cL^D(\theta, \zeta, \xi)=\argmin_{\theta, \zeta\in\mR^{\dims_\theta}\times Z}\max_{\xi\in\Xi} \cL^D_-(\theta, \zeta, \xi)
\end{align*}
where we use $\cL^D_-$ as a shorthand of $-\cL^D$. 
Given Assumptions in Sec.~\ref{sec:assumptions}, we know $\argmin_{\theta, \zeta\in\mR^{\dims_\theta}\times Z}\max_{\xi\in\Xi} \cL_-^D(\theta, \zeta, \xi)$ is a non-convex-strongly-concave problem, which has been the focus of some prior works \citep{lin2019gradient,luo2020stochastic}. 
That said, most of them target at $\min_{x\in\mathbb{R}^d}\max_{y\in\mathcal{Y}} f(x, y)$ where the first player can play an arbitrary vector in $\mR^d$, while in our setting, $\zeta$ and $\xi$ have to be constrained in bounded sets to guarantee that $\cL^D$ (and $\cL^D_-$) is $L$-smooth with finite $L$ (see the proof of Property \ref{prop:detailed_scscs} in Appendix \ref{appx:properties} for details).

Therefore, we introduce Projected-SREDA (Algorithm \ref{alg:P_SREDA}), where we project $\zeta$ back to the convex set $Z$ every time after update.
However, the proofs for original SREDA \citep{luo2020stochastic} cannot be adapted directly because the projection step will incur extra error.
To overcome this difficulty, we first study the following property of our objective function, which illustrate that the saddle-points of $\cL^D$ given arbitrary $\pi\in\Pi$ have bounded $l_2$-norm.
\begin{restatable}{property}{PropDiameter}\label{prop:radius}
    Denote $Z_0=\{\zeta|\|\zeta\|\leq R_\zeta\}$ and $\Xi_0=\{\xi|\|\xi\|\leq R_\xi\}$ with $\rad_\zeta:=\frac{1}{\lambda_w\lambda_Q\eig_w + \eig_\M^2}(\frac{1-\gamma^2}{\eig_Q}+\lambda_Q)$ and $\rad_\xi:=\frac{1}{\lambda_w\lambda_Q \eig_Q+\eig_\M^2}((1-\gamma)\lambda_w + \frac{1+\gamma}{\eig_w})$. For all $\pi\in\Pi$, we have $(\zeta^*_\pi, \xi^*_\pi):=\arg\max_{\zeta\in \mR^{{\rm dim}(Z)}}\min_{\xi\in \mR^{{\rm dim}(\Xi)}}\cL^D(\pi_\theta, w_\zeta,Q_\xi)\in Z_0\times\Xi_0$.
\end{restatable}
As a result, we have Lemma \ref{lem:small_angle} below, which is a key step to control the projection error. 
We carefully choose the projection set $Z$ and $\Xi$; as we increase the diameter of $Z$ (i.e. $R'$) and fix the diameter of $\Xi$ as $R_0$, the angle between $\nz \cL_-^D(\theta_k, \zeta_k, \xi_k)$ and the projection direction $\zeta_{k+1}-\zeta_{k+1}^+$ increases, which reflects the decrease of the relative projection error.


\begin{restatable}{lemma}{projectionError}\label{lem:small_angle}
According to Alg \ref{alg:P_SREDA}, at iteration $k$, we have $\zeta_k\in Z$, $\zeta_{k+1}^+=\zeta_k- \eta_k \vk^\zeta$ and $\zeta_{k+1}=P_Z(\zeta_{k+1}^+)$. Denote $R_0:=\frac{1}{\lambda_w\eig_w}(1 + (1+\gamma)\rad_\xi)$ where $\rad_\xi$ is defined in Property \ref{prop:radius}. If we choose $\Xi=\Xi_0$ and $Z=\{\zeta|\|\zeta\|\leq R'\}$ with $R'\geq R_0$, we have:
\begin{align*}
    &\langle \nz \cL^D_-(\theta_k, \zeta_k, \xi_k), \zeta_{k+1} - \zeta_{k+1}^+\rangle \\
    \leq&\eta_k\frac{R_0+\eta_k \|\vk^\zeta\|}{\rad' + \eta_k\|\vk^\zeta\|}\|\nz \cL_-^D(\theta_k, \zeta_k, \xi_k)\|\|\vk^\zeta\|.
\end{align*}
\end{restatable}


Based on this lemma and an appropriate choice of $R'$, we can show that Algorithm \ref{alg:P_SREDA} converges to the biased stationary points. 
We omit the detail of the hyper-parameter choices due to space limits, and a detailed version of the following theorem and its proof are deferred to Appendix \ref{appx:P_SREDA}.
\begin{theorem}[Informal]\label{them:converge_rate_P_SREDA_informal}
For $\epsilon<1$, under Assumptions in Section \ref{sec:assumptions}, by choosing $\Xi=\Xi_0$ and $Z=\{\zeta|\|\zeta\|\leq R':=8\max\{R_0, 1\}\}$, Algorithm \ref{alg:P_SREDA} will return $\hat \theta$ satisfy the following with $O(\epsilon^{-3})$ stochastic gradient evaluations.
\begin{align*}
    \EE[\nt\|J(\pi_{\hat \theta})\|] \leq \epsilon+\epsilon_{reg}+\epsilon_{func}+\epsilon_{gen}
\end{align*}
\end{theorem}

\begin{algorithm}[!htb]
    \SetAlgoLined
    \caption{Projected SREDA}\label{alg:P_SREDA}
    \textbf{Input}: initial point $(\theta_0,\zeta_0)$, learning rates $\eta_k$, $\lambda>0$, batch size $S_1, S_2 > 0$; periods $q, m > 0$, number of initial iterations $K_0$; Convex Sets $\Xi$ and $Z$; Two functions $\rm \bm {PiSARAH}$ and $\rm \bm{ConcaveMaximizer}$ (Alg. 4) in \citep{luo2020stochastic}.\\
    $\xi_0 = {\rm \bm {PiSARAH}}(-\cL^D_-(\theta_0, \zeta_0, \cdot), K_0)$\\
    \For{$k=0, ..., K-1$}{
        \If{${\rm mod}(k,q)=0$}{
            draw $S_1$ samples and compute:\\
            $\vk:=(\vk^\theta, \vk^\zeta)=\nabla_{\theta, \zeta}\cL^{S_1}_-(\theta_k, \zeta_k, \xi_k),$\\
            $\uk=\nabla_{\xi}\cL^{S_1}_-(\theta_k, \zeta_k, \xi_k)$\\
        }
        \Else{
            $\vk=\vk',\quad \uk=\uk'$\\
        }
        $\theta_{k+1}=\theta_k-\eta_k \vk^\theta,\quad\zeta_{k+1}^+=\zeta_k - \eta_k \vk^\zeta,$ \\
        $\zeta_{k+1}=P_Z(\zeta_{k+1}^+)$\\
        ($\xi_{k+1}, \vmat_{k+1}', \umat_{k+1}') =\rm \bm{ConcaveMaximizer}(k, m,$ $S_2, (\theta_k,\zeta_k), (\theta_{k+1},\zeta_{k+1}),\xi_k, \uk,\vk,\lambda$)\\
    }
    \textbf{Output}: $(\hat{\theta}, \hat\zeta)$ chosen uniformly at random from $\{(\theta_k, \zeta_k)\}_{k=0}^{K-1}$
\end{algorithm}



\section{Strategy 2: Oracle-based Stochastic Policy Improvement with Momentum}\label{sec:SRM_with_Oracle}
A natural question is that, can we solve max-max-min objective directly without merging $\max_\theta\max_\zeta$ together? We answer it firmly in this section by proposing an off-policy actor-critic style algorithm named O-SPIM in Algorithm \ref{alg:SPIM_Oracle}, where we solve $\zeta$ and $\xi$ with an abstract subroutine \textsc{Oracle} in the critic step, and update $\theta$ with momentum in the actor step.
Our analysis of O-SPIM also contributes to the understanding of convergence properties of actor-critic algorithms in the off-policy setting.

Different from the on-policy setting, where distribution shift is not a problem, our objective function has three variables to optimize, 
which causes difficulty in analyzing convergence property. Therefore, we enforce the following condition on \textsc{Oracle} to coordinate the actor and critic steps.
\begin{condition}\label{def:oracle_alg}
    For any strongly-concave-strongly-convex objective $\max_{\zeta\in Z}\min_{\xi\in\Xi}\cL^D(\theta, \zeta,\xi)$ with saddle point $(\zeta^*,\xi^*)\in Z\times\Xi$, and arbitrary $0\leq \beta\leq 1$ and $c>0$, starting from a random initializer $(\bar \zeta,\bar \xi)\in Z\times \Xi$, \textsc{Oracle} can return a solution $(\hat\zeta,\hat\xi)$ satisfying
    \begin{align*}
        \EE&[\|\hat\zeta-\zeta^*\|^2+\|\hat\xi-\xi^*\|^2]\\
        \leq& \frac{\beta}{2} \EE[\|\bar\zeta-\zeta^*\|^2+\|\bar\xi-\xi^*\|^2]+\rc. \numberthis\label{eq:oracle_condition}
    \end{align*}
\end{condition}

\begin{algorithm}[h!]
    \SetAlgoLined
    \caption{Stochastic Momentum Policy Improvement with Oracle}\label{alg:SPIM_Oracle}
    \textbf{Input}: Total number of iteration $T$; Learning rate $\et,\ez,\ex$; Dataset distribution $d^D$; \textsc{Oracle} parameter $\beta$. \\
    Set $Z=Z_0$ and $\Xi=\Xi_0$ with $Z_0$ and $\Xi_0$ defined in Property \ref{prop:radius}.\\ 
    Initialize $\theta_0, \zeta_{-1}, \xi_{-1}.$\\
    $\zeta_{0},\xi_{0}\gets \text{Oracle.init}(T_1, \ez,\ex, \theta_{0}, \zeta_{-1},\xi_{-1}, d^D).$\\
    Sample $B_0\sim d^D$ with batch size $|B_0|$ and estimate batch gradient $g^0_\theta=\nt \cL^{B_0}(\theta_{0},\zeta_0,\xi_0).$\\
    \For{$t=0,1,2,...T-1$}{
        $\theta_{t+1} \gets \theta_{t} + \et \gt^t;$\\
        $\zeta_{t+1},\xi_{t+1}\gets  \text{Oracle}(\beta, \theta_{t+1}, \zeta_t,\xi_t, d^D);$\\
        Sample $B\sim d^D$;\\
        $\gt^{t+1}\gets (1-\alpha)\gt^t+\alpha\nt\cL^B(\theta_{t+1},\zeta_{t+1},\xi_{t+1})$ // update $g_\theta$ with batch gradient\\
    }
    \textbf{Output}: Sample $\hat\theta \sim {\rm Unif}\{\theta_0, \theta_1,...,\theta_T\}$ and output $\pi_{\hat\theta}$.
\end{algorithm}

\subsection{Analysis of O-SPIM}
The key insight of Condition \ref{def:oracle_alg} is the following lemma, in which we make a connection between the actor step and the critic step. 
As a result, in addition to the smoothness of $\cL^D$, we can control the shift between $\nabla \cL^D(\theta_t, \zeta_t, \xi_t)$ and $\nabla \cL^D(\theta_{t+1}, \zeta_{t+1}, \xi_{t+1})$ only with the shift of $\theta$ (i.e. $\et g_\theta^\tau$) and hyper-parameters $\beta$ and $c$,
which paves the way to comparing $\nt\cL^D(\theta,\zeta_t,\xi_t)$ with $\nt J(\pi_\theta)$ and finally establishing the convergence guarantee.
\begin{restatable}{lemma}{RelateShift}[Relate the shift of $\zeta_{t}$ and $\xi_t$ with $\theta_t$]\label{lem:recursive_split}
    Denote $(\theta_t,\zeta_t,\xi_t)$ as the parameter value at the beginning at the step $t$ in Algorithm \ref{alg:SPIM_Oracle}, and denote $(\zeta^*_t,\xi^*_t)=\arg\max_{\zeta\in Z}\min_{\xi\in\Xi}\cL^D(\theta_t, \zeta,\xi)$ as the only saddle point given $\theta_t$. 
    Under Assumptions in Section \ref{sec:assumptions}, we have:
    \begin{align*}
        &\EE[\|\zeta_{t+1}-\zeta_t\|^2+\|\xi_{t+1}-\xi_t\|^2]\\
        \leq& 6\beta^{t+1}d^2+6\et^2\Czx\sum_{\tau=0}^{t}\beta^{t-\tau}\expect{\|\gt^\tau\|^2}+\frac{6\rc}{1-\beta}.
    \end{align*}
    where $\dia:=\max\{C_\cW, C_\cQ\}$ is the maximum of diameters of $Z$ and $\Xi$, $\Czx$ is a short note of $\kappa^2_\mu(\kappa_\xi+1)^2+\kappa^2_\xi(\kappa_\mu+1)^2$, $g_\theta^\tau:=\frac{1}{\et}(\theta^{\tau+1}-\theta^\tau)$, and $\et$ is the step size of $\theta$.
\end{restatable}
In the actor step, we introduce another hyper-parameter $\alpha$ and adopt a momentum-based update rule, aiming at a better trade-off between the variance of the gradient estimation in the current step and the bias of using accumulative gradient in previous steps. As we will show in the proof of our main theorem (Theorem \ref{thm:alg2_converage_rate} in Appendix \ref{appx:proofs_alg2}), by choosing $\alpha$ appropriately, although the trade-off cannot improve the dependence of $\epsilon$ in the convergence rate, it can indeed reduce the upper bound of $\EE[\|\nt J(\hat\theta)\|]$ comparing with the case without momentum (i.e. $\alpha=1$)

Now, we are ready to state the main theorem of our second strategy. We defer the formal version including the hyper-parameter choices and its proof to Theorem \ref{thm:alg2_converage_rate} in Appendix \ref{appx:proofs_alg2}.
\begin{theorem}[Informal]
    Under Assumptions in Section \ref{sec:assumptions}, given arbitrary $\epsilon$, with appropriate hyper-parameter choices, by using either Algorithm \ref{alg:PLSO} or Algorithm \ref{alg:SVREB} as Oracle, Algorithm \ref{alg:SPIM_Oracle} will return us a policy $\pi_{\hat\theta}$ with total complexity $O(\epsilon^{-4})$, satisfying
    \begin{align*}
        \EE[\|\nt J(\pi_{\hat\theta})\|]\leq \epsilon+\epsilon_{reg} + \epsilon_{data}+\epsilon_{func}.
    \end{align*}
\end{theorem}


\subsection{Concrete Examples of \textsc{Oracle}}
We provide two concrete examples of \textsc{Oracle}, and defer the algorithm details and the discussion about related works to Appendix \ref{appx:oracle_examples}.

\subsubsection{The Least-Square Oracle}
In the linear setting, the saddle-point $(\zeta^*, \xi^*)$ has a closed form, which can be regarded as the regularized LSTD-Q solution \citep{Kolter2009RegularizationAF} and can be represented by $\K_w, \K_Q, \M_\pi$ and so on. Therefore, a natural solution is to estimate these matrices from (a subsample of) the dataset and plug them into the closed-form solution. We describe this idea in Alg.~\ref{alg:PLSO} and show that it satisfies our definition of \textsc{Oracle} in Appendix \ref{appx:PLSO}.

In Alg. \ref{alg:PLSO}, we use $\uR$ to denote $\bPhi_w\trans \bLambda^D R$ and use $\unu^\pi$ to denote $\bPhi_Q\trans\nu_D^\pi$. Besides, we use $\hat{(\cdot)}$ to denote the empirical version of the matrices/vectors estimated via samples. Besides, $P_\Omega(x)$ means projected $x$ into set $\Omega$.
\begingroup
\begin{algorithm}
    \SetAlgoLined
    \caption{Projected Least-Square Oracle}\label{alg:PLSO}
    \textbf{Input}: Distribution $d^D$; Batch Size $\batch_{all}$;\\
    $\zeta \gets \Big(\lambda_w\lambda_Q  \hK_w +  \hM_\pi \hK_Q^{-1} \hM_\pi\trans\Big)^{-1}\Big(-(1-\gamma)\hM_\pi \hK_Q^{-1} \hunu^\pi+\lambda_Q\huR\Big) .$\\
    $\xi\gets  \Big(\lambda_w\lambda_Q  \hK_Q +  \hM_\pi\trans \hK_w^{-1}\hM_\pi \Big)^{-1}\Big((1-\gamma)\lambda_w \hunu^\pi +\hM\trans_\pi \hK^{-1}_w \huR\Big).$\\
    \textbf{Output}: $P_Z(\zeta)$, $P_\Xi(\xi)$
\end{algorithm}
\endgroup

\subsubsection{Stochastic Variance-Reduced Extragradient with Batch Data}
Similar to the LSTD, the per-step computational complexity of the first solver is quadratic in the dimension of the feature, which can be expensive for high dimensional features. Therefore, in Alg.~\ref{alg:SVREB}, we present a first-order algorithm inspired by Stochastic Variance Reduced Extra-gradient \citep{chavdarova2019reducing}, which reduces the per-step complexity to $O(d)$ and also satisfies the \textsc{Oracle} condition. Besides, Alg.~\ref{alg:SVREB} can also handle general strongly-convex-strongly-concave problem beyond the linear setting.

\begin{algorithm}
\SetAlgoLined
\caption{Stochastic Variance-Reduced Extragradient with Mini Batch Data (SVREB)}\label{alg:SVREB}
\textbf{Input}: Stopping time $K$; learning rates $\ez,\ex$; Initial wegihts $\zeta_0,\xi_0$; Distribution $d^D$; Batch size $|\batch|.$\\
Sample mini batch $\batch_{\zeta},\batch_{\xi}\sim d^D$ with batch size $|\batch|$ \\    
${g}_0^\zeta \gets \nz \cL^{\batch_\zeta} (\theta,\zeta_0, \xi_0),\quad {g}_0^\xi \gets  \nx \cL^{\batch_\xi}(\theta, \zeta_0, \xi_0).$\\
$\zeta_1\gets \cP_\zeta(\zeta_0 + \ez g_0^\zeta).$\\
$\xi_1\gets \cP_\xi(\xi_0 - \ex g_0^\xi).$\\
$m_1^\zeta, m_1^\xi\gets \nz \cL^{\batch_\zeta} (\theta,\zeta_0, \xi_0),  \nx \cL^{\batch_\xi}(\theta, \zeta_0, \xi_0).$\\
\For{$k=1,2,...K+1$}{
    Sample mini batch $\batch_\zeta,\batch_\xi\sim d^D$ with batch size $|\batch|$. \\
    ${g}_{k}^{\zeta}=m_k^\zeta + d^{\batch_\zeta}_\zeta(\zeta_k,\xi_k,\zeta_{k-1},\xi_{k-1}).$ \label{step:compute_cv_zeta}\\
    ${g}_{k}^{\xi}=m_k^\xi + d^{\batch_\xi}_\xi(\zeta_k,\xi_k,\zeta_{k-1},\xi_{k-1}).$ \label{step:compute_cv_xi}\\
    $\zeta_{k+1/2}=\cP_\zeta(\zeta_k+\ez {g}_{k}^{\zeta}).$\\
    $\xi_{k+1/2}=\cP_\xi(\xi_k-\ex {g}_{k}^{\xi}).$\\
    Sample mini batch $\batch'_\zeta,\batch'_\xi\sim d^D$ with batch size $|\batch|$. \\
    ${g}_{k+1/2}^\zeta=m_k^\zeta + d^{\batch'_\zeta}_\zeta(\zeta_{k+1/2},\xi_{k+1/2},\zeta_{k-1},\xi_{k-1}).$\\
    ${g}_{k+1/2}^\xi=m_k^\xi + d^{\batch'_\xi}_\xi(\zeta_{k+1/2},\xi_{k+1/2},\zeta_{k-1},\xi_{k-1}).$\\
    $\zeta_{k+1}=\cP_\zeta(\zeta_{k}+\ez {g}_{k+1/2}^\zeta).$\\
    $\xi_{k+1}=\cP_\xi(\xi_{k}-\ex {g}_{k+1/2}^\xi).$\\
    // The following has been computed in step \ref{step:compute_cv_zeta} and \ref{step:compute_cv_xi}\\
    $m_{k+1}^\zeta, m_{k+1}^\xi\gets \nz \cL^{\batch_\zeta}(\theta, \zeta_k, \xi_k), \nx \cL^{\batch_\xi}(\theta, \zeta_k, \xi_k)$ \\
    $k\gets k+1.$\\
}
\textbf{Output}: $\zeta_K,\xi_K$
\end{algorithm}
In the algorithm, $\cP_\zeta$ and $\cP_\xi$ are projection operators; $\nabla \cL^{\batch}(\theta,\zeta,\xi)$ denotes the average gradient over samples from mini batch data $\batch$. We also define:
\begin{align*}
    d^\batch_\zeta(\zeta_1,\xi_1,\zeta_2,\xi_2)=&\nz \cL^{\batch}(\theta,\zeta_1,\xi_1)-\nz \cL^{\batch}(\theta,\zeta_2,\xi_2)\\
    d^\batch_\xi(\zeta_1,\xi_1,\zeta_2,\xi_2)=&\nx \cL^{\batch}(\theta,\zeta_1,\xi_1)-\nx \cL^{\batch}(\theta,\zeta_2,\xi_2).
\end{align*}
It is clear that
\begin{align*}
    &\EE[g^\zeta_{k}]=\nz \cL^D(\theta, \zeta_k,\xi_k),\\
    &\EE[g^\zeta_{k+1/2}]=\nz \cL^D(\theta, \zeta_{k+1/2},\xi_{k+1/2}),
\end{align*}
where the expectation only concerns the randomness of sample when computing $g$. The above relationship also holds if we consider gradient w.r.t. $\xi$.

For this Algorithm \ref{alg:SVREB}, we have the following theorem:
\begin{restatable}{theorem}{ConvergeRateAlgFour}\label{thm:oracle_thm}
Under Assumptions in Section \ref{sec:assumptions}, in Algorithm \ref{alg:SVREB}, if step sizes satisfy
\begin{align*}
    \ez\leq \frac{1}{50 \max\{\bar{L}_\zeta, \mu_\zeta\}},~~~~~~~\ex\leq \frac{1}{50 \max\{\bar{L}_\xi, \mu_\xi\}}
\end{align*}
after $K$ iterations, the algorithm will output $(\zeta_K, \xi_K)$:
\begin{align*}
    &\EE[\|\zeta_K-\zeta^*\|^2+\|\xi_K-\xi^*\|^2]\\
    \leq& \frac{201}{100}\Big(1-\frac{\mu\eta}{4}\Big)^{K}\EE[\|\zeta_0-\zeta^*\|^2+\|\xi_0-\xi^*\|^2]\\
    &+\frac{8\sigma^2}{\min\{\frac{\mz\ez}{4},\frac{\mx\ex}{4}\}|\batch|}(\frac{\ez}{\mz}+\frac{\ex}{\mx})
\end{align*}
where $(\zeta^*,\xi^*)$ is the saddle point of $\cL^D(\theta,\zeta,\xi)$ given input $\theta$. 
\end{restatable}
The theorem implies that Alg.~\ref{alg:SVREB} can be used as an oracle for arbitrary $\beta$ and $c$ as long as $K$ and $|N|$ are chosen appropriately.
We defer the proofs to Appx.~\ref{appx:proofs_alg4}.

\section{Conclusion}\label{sec:conclusion}
In this paper, we study two natural optimization strategies for density-ratio based off-policy policy optimization, establish their convergence rates, and characterize the quality of the output policies. In the future, there are several potentially interesting directions to study. (\textbf{i}) It would be interesting to investigate the possibility of improving the dependence on $\epsilon$ on the convergence rate of our second strategy. (\textbf{ii}) Some of our assumptions are still quite strong, such as Assump.~\ref{assump:feature_matrices} and \ref{assump:dataset}, where we assume that some condition holds for \textit{any} policy $\pi\in\Pi$. An open question is whether we can derive convergence rate when those conditions holds only for the optimal policy. (\textbf{iii}) In this paper, we consider the linear approximation for $\cW$ and $\cQ$ in order to ensure the strongly concavity/convexity property of the loss function w.r.t.~$\zeta$ and $\xi$ so that the analysis is more tractable. It is still an open problem whether we can show some convergence property if $\cW$ and $\cQ$ are more complicated function classes.

\newpage
\section*{Acknowledgements}
JH's research activities on this work at UIUC were completed by December 2021. NJ acknowledges funding support from ARL Cooperative Agreement W911NF-17-2-0196, NSF IIS-2112471, and Adobe Data Science Research Award.

\bibliographystyle{plainnat}
\bibliography{references}

\clearpage
\appendix

\onecolumn

\begin{table}[ht]
  \begin{footnotesize}
  \begin{center}
  \caption{Comparison of Convergence Guarantee of Recent Methods}\label{tab:comparison}
  \begin{tabular}{ccccc}
  \toprule
  Algorithms & Convergence Rate & Off-Policy? & 
  Detailed Bias Analysis? \\
  \toprule
  SVRPG \citep{xu2019improved} &$O(\epsilon^{-10/3})$& $\times$ & \\ \cmidrule{1-3}
  SRVR-PG \citep{xu2019sample} &  $O(\epsilon^{-3})$  & $\times$ & \\ \cmidrule{1-3}
  STORM-PG \citep{yuan2020stochastic} & $O(\epsilon^{-3})$ & $\times$ & N.A. &\\ \cmidrule{1-3}
  MBPG \citep{pmlr-v119-huang20a} & $O(\epsilon^{-3})$ & $\times$ & \\ \cmidrule{1-3}
  \multicolumn{1}{m{4cm}}{\shortstack{On Policy AC/NAC \\ \citep{fu2020singletimescale,xu2020nonasymptotic}}} & $O(\epsilon^{-4})$ & $\times$ & \\ \midrule
  DR-Off-PAC \citep{xu2021doubly} & $O(\epsilon^{-4})$ & \checkmark  & $\times$ \\ \midrule
  P-SREDA (Ours) & $O(\epsilon^{-3})$  & \checkmark & \checkmark \\ \midrule
  O-SPIM (Ours) & $O(\epsilon^{-4})$ & \checkmark & \checkmark \\
  \bottomrule
  \end{tabular}
  \end{center}
  \end{footnotesize}
  \end{table}

\section{Useful Lemma}\label{appx:useful_lemma}

\begin{lemma}[Lemma B.2 in \citep{lin2020near}]\label{lem:optimum_function}
Define
\begin{align*}
    \Phi_\theta(\zeta) = \min_{\xi\in\Xi} \cL^D(\theta, \zeta, \xi)&~~~~~~~\phi_\theta(\zeta) =\argmin_{\xi\in\Xi} \cL^D(\theta, \zeta, \xi),~~~~for~\zeta\in \mR^{dim(Z)}\\
    \Psi_\theta(\xi)=\max_{\zeta\in Z}\cL^D(\theta, \zeta, \xi)&~~~~~~~    \psi_\theta(\xi)=\argmax_{\zeta\in Z} \cL^D(\theta, \zeta, \xi),~~~~for~\xi\in\mR^{dim(\Xi)}\\
\end{align*}
Under Assumption \ref{assump:smooth} and \ref{assump:feature_matrices}, for fixed $\theta$, we have:

(1) The function $\phi_\theta(\cdot)$ is $\kappa_\xi=\frac{L}{\mu_\xi}$-Lipschitz.

(2) The function $\Phi_\theta(\cdot)$ is $2\kappa_\xi L=2\frac{L^2}{\mu_\xi}$-smooth and $\mu_\zeta$-strongly concave with $\nabla \Phi_\theta(\cdot):=\nz \cL^D(\theta,\zeta, \phi_\theta(\zeta))$.

(3) The function $\psi_\theta(\cdot)$ is $\kappa_\zeta=\frac{L}{\mu_\zeta}$-Lipschitz.

(4) The function $\Psi_\theta(\cdot)$ is $2\kappa_\zeta L=2\frac{L^2}{\mu_\zeta}$-smooth and $\mu_\xi$-strongly convex with $\nabla \Psi_\theta(\cdot):=\nx \cL^D(\theta, \psi_\theta(\xi), \xi)$.
\end{lemma}
\begin{remark}[For clarification]
    According to Danskin's Theorem, 
    $$\nabla \Phi_\theta(\cdot):=\nz \cL^D(\theta,\zeta, \phi_\theta(\zeta))=\nz \cL^D(\theta,\zeta, \xi)|_{\xi=\phi_\theta(\zeta)}$$
    Therefore, when we compute $\nz \cL^D(\theta,\zeta, \phi_\theta(\zeta))$, we can treat $\phi_\theta(\zeta)$ as a constant. Then, for arbitrary $\zeta',\xi'$, based on Assumption \ref{assump:smooth}, we always have:
    \begin{align*}
        \|\nabla \Phi_\theta(\cdot) - \nz \cL^D(\theta,\zeta', \xi')\|\leq L\|\zeta - \zeta'\|+L\|\phi_\theta(\zeta)-\xi'\|
    \end{align*}
    We have a similar clarification w.r.t. $\nx\Psi(\xi)$.
\end{remark}


\begin{lemma}\label{lem:grad_norm}
For $\alpha$-strongly-convex function $f(x)$ and $\beta$-strongly-concave function $g(x)$ w.r.t. $x\in X$, where $X\subseteq \mathbb{R}^n$ is a convex set, we have
\begin{align}
    \|x-x^*_f\|\leq \frac{1}{\alpha}\|\nabla_x f(x)\|,\qquad
    \frac{\alpha}{2}\|x-x^*_f\|^2\leq f(x)-f(x^*_f)\\
    \|x-x^*_g\|\leq \frac{1}{\beta}\|\nabla_x g(x)\|,\qquad
    \frac{\beta}{2}\|x-x^*_f\|^2\leq g(x^*_g)-g(x)
\end{align}
where $x^*_f$ and $x^*_g$ the minimum and maximum of $f(x)$ and $g(x)$, respectively.
\end{lemma} 
\begin{proof}
Since $f(x)$ is $\alpha$-strongly-convex, we have 
\begin{align*}
    (\nabla_x f(x)-\nabla_x f(x^*_f))\trans (x-x^*_f)  \geq \alpha \|x-x^*_f\|^2,\qquad f(x)\geq f(x^*_f)+\nabla_x f(x^*_f)\trans (x-x^*_f)+\frac{\alpha}{2}\|x-x^*_f\|^2.
\end{align*}
Since $x^*_f$ is the minimizer of $f(x)$, we know that $\nabla_x f(x^*_f)\trans (x-x^*_f) \geq 0,\forall x\in X$. Combining all the above inequalities together and we obtain
\begin{align*}
    \|x-x^*_f\|^2\leq \frac{1}{\alpha} \nabla_x f(x)\trans (x-x^*_f)\leq \frac{1}{\alpha}\|\nabla_x f(x)\|\|x-x^*_f\|,\qquad f(x)\geq f(x^*_f)+\frac{\alpha}{2}\|x-x^*_f\|^2.
\end{align*}
which implies
\begin{align*}
    \|x-x^*_f\|\leq \frac{1}{\alpha}\|\nabla_x f(x)\|,\qquad \frac{\alpha}{2}\|x-x^*_f\|^2\leq f(x)-f(x^*_f).
\end{align*}
By applying the above results for $-g(x)$ which is a $\beta$-strongly-convex function and we can complete the proof.
\end{proof}

\begin{lemma}\label{lem:non_decreasing_order}
    For positive definite matrix $\A$, and arbitrary $\alpha > 0$, we have:
    \begin{align*}
        (\A\trans \A)^{-1}\succ &\Big((\alpha \I+\A)\trans(\alpha \I+\A)\Big)^{-1}
    \end{align*}
\end{lemma}
\begin{proof}
    Suppose for symmetric matrix $\A$ and $\B$, we have the relationship $\A\succ \B\succ 0$.
    According to the inverse matrix lemma, we have
    \begin{align*}
        \B^{-1}-\A^{-1}=\B^{-1}-(\B+(\A-\B))^{-1}=(\B+\B(\A-\B)^{-1}\B)^{-1}
    \end{align*}
    Because $\A\succ \B\succ 0$, we have $(\B+\B(\A-\B)^{-1}\B)^{-1}\succ 0$, therefore $\B^{-1}\succ \A^{-1}$.

    Then, we only need to prove
    \begin{align*}
        (\alpha \I+\A)\trans(\alpha \I+\A) \succ & \A\trans \A
    \end{align*}
    We have
    \begin{align*}
        (\alpha \I+\A)\trans(\alpha \I+\A)=\alpha^2 \I+\alpha (\A+\A\trans ) + \A\trans \A
    \end{align*}
    Combining $\A=\A\trans \succ 0$ and $\alpha > 0$, we can finish the proof.
\end{proof}

\begin{lemma}[Non-negative Elements]\label{lem:all_non_negative}
    We use $\bP_*^\pi=(\bP^\pi)\trans\in \mathbb{R}^{|\cS||\cA|\times |\cS||\cA|}$ to denote the transpose of the transition kernel. All the elements in $(\I-\gamma \bP^\pi_*)^{-1}$ are non-negative. Moreover, the element indexed by $(s_i, a_j)$ in row and $(s_p, a_q)$ in column equals to the unnormalized discounted state-action occupancy of $(s_i,a_j)$ starting from $(s_p, a_q)$ and executing $\pi$.
\end{lemma}
\begin{proof}
    For arbitrary initial state-action distribution vector $\mu_0\in \mathbb{R}^{|\cS||\cA|\times 1}$, $(\I-\gamma \bP^\pi_*)^{-1} \mu_0$ is a vector whose elements are unnormalized state-action occupancy with $\mu_0$ as initial distribution, which is larger or equal to 0. As a result, by choosing standard basis vector as $\mu_0$, we can finish the proof.
\end{proof}

\section{Useful Properties Implied by Assumptions in \ref{sec:assumptions}}\label{appx:properties}
In this section, we first prove several properties implied by our basic assumptions in Section \ref{sec:assumptions}.

\begin{restatable}{property}{PropSCSCS}[A detailed version of Property \ref{prop:scscs}]\label{prop:detailed_scscs}
    Under Assumption \ref{assump:smooth} and \ref{assump:feature_matrices}, given that $C_\cW:=\max\{1,\max_{\zeta\in Z}\|\zeta\|\}$ and $C_\cQ:=\max_{\xi\in\Xi}\|\xi\|$ are finite, we have:
    \paragraph{(a)} For arbitrary $\theta\in\Theta,\zeta\in Z$, $\cL^D$ is $\mu_\xi$-strongly convex w.r.t. $\xi\in \Xi$, and for arbitrary $\theta\in\Theta,\xi\in\Xi$, $\cL^D$ is $\mu_\zeta$-strongly concave w.r.t. $\zeta\in Z$, where $\mu_\xi=\lambda_Q \eig_Q$ and $\mu_\zeta = \lambda_w \eig_w$.
    \paragraph{(b)} For any $\xi,\xi_1,\xi_2\in\Xi,\zeta,\zeta_1,\zeta_2\in Z,(s,a)\in\cS\times\cA$, 
    \begin{align*}
        |Q_\xi(s,a)|\leq C_\cQ;&\quad
        |Q_{\xi_1}(s,a)-Q_{\xi_2}(s,a)|\leq  \|\xi_1-\xi_2\|;\\
        |w_\zeta(s,a)|\leq C_\cW;&\quad
        |w_{\zeta_1}(s,a)-w_{\zeta_2}(s,a)|\leq  \|\zeta_1-\zeta_2\|.
    \end{align*}
    \paragraph{(c)} For any $\zeta_1,\zeta_2\in Z, \xi_1, \xi_2\in\Xi, \theta_1, \theta_2 \in \Theta$, $\cL^D$ defined in Eq.\eqref{eq:problem} is differentiable, and there exists constant $L$ s.t.
    \begin{align*}
    &\|\nt \cL^D(\theta_1, \zeta_1, \xi_1)-\nt \cL^D(\theta_2, \zeta_2, \xi_2)\|+\|\nz \cL^D(\theta_1, \zeta_1, \xi_1)-\nz \cL^D(\theta_2, \zeta_2, \xi_2)\|\\
    &+\|\nx \cL^D(\theta_1, \zeta_1, \xi_1)-\nx \cL^D(\theta_2, \zeta_2, \xi_2)\\\
    \leq& L\|\theta_1-\theta_2\|+L\|\zeta_1-\zeta_2\|+L\|\xi_1-\xi_2\|.
    \end{align*}
    In other words, $\cL^D$ is L-smooth when $\zeta\in Z,\xi\in\Xi, \theta\in\Theta$.
\end{restatable}    
\begin{proof}
$ $
\paragraph{Proof of (a)}
Since $\cL^D$ is second-order differentible w.r.t. arbitrary $\zeta\in\mR^{\dims_\zeta}$ and $\xi\in\mR^{\dims_\xi}$, under Assumption \ref{assump:feature_matrices}, we have:
\begin{align*}
    \nabla^2_\zeta \cL^D = -\lambda_w K_w \prec -\lambda_w \eig_w \I,~~~~~~~\nabla^2_\xi \cL^D = \lambda_Q K_Q \succ \lambda_Q \eig_Q\I.
\end{align*}
where $\I$ is the identity matrix.
\paragraph{Proof of (b)}
Because $w_\zeta$ and $Q_\xi$ are linear and features has bounded $l_2$-norm, and $Z$ and $\Xi$ are all convex sets with bounded radius, we have:
\begin{align*}
    |Q_\xi(s,a)|\leq \|\bphi_Q(s,a)\trans \xi\|\leq C_\cQ,~~~~~|Q_{\xi_1}(s,a)-Q_{\xi_2}(s,a)|\leq \|\xi_1 - \xi_2\|;\\
    |w_\zeta(s,a)|\leq \|\bphi_w(s,a)\trans \zeta\|\leq C_\cW,~~~~~|w_{\zeta_1}(s,a)-w_{\zeta_2}(s,a)|\leq \|\zeta_1 - \zeta_2\|;
\end{align*}
Therefore, Property \ref{prop:detailed_scscs}-(b) holds.
\paragraph{Proof of (c)}
We will use $,w_1, Q_1, \pi_1$ and $\cL^D_1$ as shortnotes of $w_{\zeta_1}, Q_{\xi_1}, \pi_{\theta_1}$ and $\cL^D(\theta_1, \zeta_1, \xi_1)$, and the meaning of $w_2, Q_2, \pi_2$ and $\cL^D_2$ are similar.
\begin{align*}
    &\|\nt \cL^D_1 - \nt \cL^D_2\|\\
    = &\|(1-\gamma) \EE_{ \nu_D, a\sim \pi_1(\cdot|s_0)}[Q_1 (s_0,a)\nt\log\pi_1(a|s_0)]-(1-\gamma) \EE_{\nu_D, a\sim \pi_2(\cdot|s_0)}[Q_2 (s_0,a)\nt\log\pi_2(a|s_0)]\\
    &+\gamma \EE_{d^D, a'\sim \pi_1(\cdot|s')}[w_1(s,a)Q_1(s', a')\nt\log\pi_1(a'|s')]-\gamma \EE_{d^D, a'\sim \pi_2(\cdot|s')}[w_2(s,a)Q_2(s', a')\nt\log\pi_2(a'|s')]\|\\
    \leq& (1-\gamma) \|\EE_{\nu_D, a\sim \pi_1(\cdot|s_0)}[\big(Q_1 (s_0,a)-Q_2(s_0,a)\big)\nt\log\pi_1(a|s_0)]\\
    &+\EE_{\nu_D, a\sim \pi_1(\cdot|s_0)-\pi_2(\cdot|s_0)}[Q_2(s_0,a)\nt\log\pi_1(a|s_0)]+\EE_{\nu_D, a\sim \pi_2(\cdot|s_0)}[Q_2(s_0,a)\Big(\nt\log\pi_1(a|s_0)-\nt\log\pi_2(a|s_0)\Big)]\|\\
    &+\gamma \|\EE_{d^D, a'\sim \pi_1(\cdot|s')}[\big(w_1(s,a)-w_2(s,a)\big)Q_1(s', a')\nt\log\pi_1(a'|s')]\\
    &\quad+\EE_{d^D, a'\sim \pi_1(\cdot|s')}[w_2(s,a)\big(Q_1(s', a')-Q_2(s', a')\big)\nt\log\pi_1(a'|s')]\\
    &\quad+\EE_{d^D, a'\sim \pi_1(\cdot|s')-\pi_2(\cdot|s')}[w_2(s,a)Q_2(s', a')\nt\log\pi_1(a'|s')]\\
    &\quad+\EE_{d^D, a'\sim \pi_2(\cdot|s')}[w_2(s,a)Q_2(s', a')\big(\nt\log\pi_1(a'|s')-\nt\log\pi_2(a'|s')\big)]\|\\
    \leq&(1-\gamma)\Big(G\|\xi_1-\xi_2\|+GC_\cQ \EE_{\nu_D}[\|\pi(\cdot|s_0)-\pi(\cdot|s_0)\|_1]+HC_\cQ \|\theta_1-\theta_2\|\Big)\\
    &+\gamma \Big(GC_\cQ \|\zeta_1-\zeta_2\|+GC_\cW  \|\xi_1-\xi_2\|+GC_\cW C_\cQ \EE_{d^D}[\|\pi_1(\cdot|s')-\pi_2(\cdot|s')\|_1]+H C_\cW C_\cQ\|\theta_1-\theta_2\| \Big)\\
    \leq& C_\cW C_\cQ (G L_\Pi+ H) \|\theta_1 - \theta_2\|+G C_\cQ \|\zeta_1-\zeta_2\| +G C_\cW  \|\xi_2-\xi_2\| 
\end{align*}
In the last inequality, we use $C_\cW \geq 1$ and $0<\gamma\leq 1$. Besides,
\begin{align*}
    \|\nz \cL^D_1 - \nz \cL^D_2\|=&\|(M_{\pi_1}-M_{\pi_2})\xi_1+ M_{\pi_2} (\xi_1 - \xi_2)-\lambda_w K_w (\zeta_1 -\zeta_2)\|\\
    \leq& \gamma C_\cQ \EE_{d^D}[\|\pi_1(\cdot|s')-\pi_2(\cdot|s')\|_1] +(1+\gamma) \|\xi_1-\xi_2\|+\lambda_w \|\zeta_1-\zeta_2\|\\
    \leq& \gamma C_\cQ L_\Pi \|\theta_1-\theta_2\| +(1+\gamma) \|\xi_1-\xi_2\|+\lambda_w \|\zeta_1-\zeta_2\|\\
    \|\nx \cL^D_1 - \nx \cL^D_2\|=& \|(M_{\pi_1}-M_{\pi_2})\trans\zeta_1+ M_{\pi_2}\trans (\zeta_1 - \zeta_2)+\lambda_Q K_Q (\xi_1 -\xi_2)\|\\
    \leq &\gamma C_\cW \EE_{d^D}[\|\pi_1(\cdot|s')-\pi_2(\cdot|s')\|_1]+(1+\gamma) \|\zeta_1-\zeta_2\|+\lambda_Q \|\xi_1-\xi_2\|\\
    \leq& \gamma C_\cW L_\Pi \|\theta_1-\theta_2\| + (1+\gamma) \|\zeta_1-\zeta_2\|+\lambda_Q \|\xi_1-\xi_2\|
\end{align*}
As a result,
\begin{align*}
&\|\nt \cL^D(\theta_1, \zeta_1, \xi_1)-\nt \cL^D(\theta_2, \zeta_2, \xi_2)\|+\|\nz \cL^D(\theta_1, \zeta_1, \xi_1)-\nz \cL^D(\theta_2, \zeta_2, \xi_2)\|\\
&+\|\nx \cL^D(\theta_1, \zeta_1, \xi_1)-\nx \cL^D(\theta_2, \zeta_2, \xi_2)\|\\
\leq& \Big(C_\cW C_\cQ (G L_\Pi+ H)+\gamma (C_\cQ + C_\cW) L_\Pi\Big)\|\theta_1-\theta_2\| + \Big(G C_\cQ +(1+\gamma) +\lambda_w\Big)\|\zeta_1-\zeta_2\|\\
&+ \Big(G C_\cW  +(1+\gamma) +\lambda_Q\Big)\|\xi_1-\xi_2\|\\
\end{align*}
Therefore, Condition \ref{assump:smooth}-(b) holds with
\begin{align}\label{eq:choice_of_L}
    L = \max\{C_\cW C_\cQ (G L_\Pi+ H)+\gamma (C_\cQ + C_\cW) L_\Pi, G C_\cQ +(1+\gamma) +\lambda_w, G C_\cW  +(1+\gamma) +\lambda_Q\}
\end{align}
\end{proof}


\PropDiameter*
\begin{proof}
Recall the definition of $\cL^D$ in Eq.\eqref{eq:linear_formulation}:
\begin{align*}
\cL^D(\pi, \zeta, \xi)
=& (1-\gamma)(\nu_D^\pi)\trans \bPhi_Q \xi+\zeta\trans \bPhi_w\trans \bLambda^D R -\zeta\trans \M_\pi \xi+\frac{\lambda_Q}{2}\xi\trans \K_Q \xi - \frac{\lambda_w}{2}\zeta\trans \K_w\zeta.\label{}
\end{align*}
by taking derivatives w.r.t. $\xi$ and setting it to be zero, we have:
\begin{align*}
    \xi = \frac{1}{\lambda_Q} \K_Q^{-1}  \Big( \M_\pi\trans \zeta-(1-\gamma) \bPhi_Q\trans \nu_D^\pi\Big)
\end{align*}
Plug it into $\cL^D$:
\begin{align*}
    -\frac{\lambda_w}{2}\zeta\trans  \K_w \zeta -\frac{1}{2\lambda_Q}\Big( \M_\pi\trans \zeta-(1-\gamma) \bPhi_Q\trans \nu_D^\pi\Big)\trans  \K_Q^{-1} \Big( \M_\pi\trans \zeta-(1-\gamma) \bPhi_Q\trans \nu_D^\pi\Big)+\zeta\trans \bPhi_w \trans\bLambda^D R
\end{align*}
Taking the derivative of $\zeta$ and set it to be zero, we have:
\begin{align*}
    \zeta^*_\pi =& \Big(\lambda_w\lambda_Q  \K_w +  \M_\pi \K_Q^{-1} \M_\pi\trans\Big)^{-1}\Big(-(1-\gamma)\M_\pi \K_Q^{-1} \bPhi_Q\trans  \nu_D^\pi+\lambda_Q\bPhi_w\trans \bLambda^D R\Big)
\end{align*}
and therefore,
\begin{align*}
    \xi^*_\pi =& \frac{1}{\lambda_Q} \K_Q^{-1}  \Big( \M_\pi\trans \zeta_\pi^*-(1-\gamma) \bPhi_Q\trans \nu_D^\pi\Big)\\
    =&\frac{1}{\lambda_Q} \K_Q^{-1} \M_\pi\trans \Big(\lambda_w\lambda_Q  \K_w +  \M_\pi \K^{-1}_Q \M_\pi\trans\Big)^{-1} \cdot \Big(- (1-\gamma)\M_\pi \K_Q^{-1} \bPhi_Q\trans  \nu_D^\pi+\lambda_Q\bPhi_w\trans \bLambda^D R\Big)\\
    &+(1-\gamma) \frac{1}{\lambda_Q} \K_Q^{-1} \bPhi_Q\trans(\nu_D^\pi)\trans\\
    =&  (1-\gamma)\lambda_w \Big(\lambda_w\lambda_Q  \K_Q +  \M_\pi\trans \K_w^{-1}\M_\pi \Big)^{-1}\bPhi_Q\trans  \nu_D^\pi+\K_Q^{-1} \M_\pi\trans \Big(\lambda_w\lambda_Q  \K_w +  \M_\pi \K_Q^{-1}\M_\pi \trans\Big)^{-1}\bPhi_w\trans \bLambda^D R\\
    =&\Big(\lambda_w\lambda_Q  \K_Q +  \M_\pi\trans \K_w^{-1}\M_\pi \Big)^{-1}\Big((1-\gamma)\lambda_w \bPhi_Q\trans  \nu_D^\pi +\M\trans_\pi \K^{-1}_w \bPhi_w\trans \bLambda^D R\Big)
\end{align*}
where in the third step, we use the inverse matrix lemma:
\begin{align*}
    (\lambda_w\lambda_Q \K_Q+\M_\pi\trans \K_w^{-1}\M_\pi)^{-1}=\frac{1}{\lambda_w\lambda_Q}\K_Q^{-1}-\frac{1}{\lambda_w\lambda_Q}\K_Q^{-1}\M_\pi\trans (\lambda_w\lambda_Q \K_w+\M_\pi\K_Q^{-1}\M_\pi\trans)\M_\pi\K_Q^{-1}
\end{align*}
Because $\|\bphi(\cdot, \cdot)\|\leq 1$, it's easy to prove that, for arbitrary vector $x\in \mathbb{R}^d$, 
\begin{align*}
    \max\{\|\M_\pi x\|, \|\M_\pi\trans x\| \} \leq (1+\gamma)\|x\|
\end{align*}

Therefore,
\begin{align*}
    \|\zeta^*_\pi\| \leq&  (1-\gamma)\|\Big(\lambda_w\lambda_Q  \K_w +  \M_\pi \K_Q^{-1} \M_\pi\trans\Big)^{-1}\M_\pi \K_Q^{-1} \|\cdot\|\bPhi_Q\trans  \nu_D^\pi\| \\
    &+ \|\Big(\lambda_w\lambda_Q  \K_w +  \M_\pi \K_Q^{-1} \M_\pi\trans\Big)^{-1}\|\cdot \|\lambda_Q\bPhi_w\trans \bLambda^D R\|\\
    \leq&\frac{1}{\lambda_w\lambda_Q\eig_w + \eig_\M^2}(\frac{1-\gamma^2}{\eig_Q}+\lambda_Q) 
    :=\rad_\zeta\numberthis\label{def:Rzeta}\\
    \|\xi^*_\pi\|\leq& (1-\gamma)\lambda_w \|\Big(\lambda_w\lambda_Q  \K_Q +  \M_\pi\trans \K_w^{-1}\M_\pi \Big)^{-1}\|\cdot\|\bPhi_Q\trans  \nu_D^\pi\|\\
    &+\|\Big(\lambda_w\lambda_Q  \K_Q +  \M_\pi\trans \K_w^{-1}\M_\pi \Big)^{-1}\M_\pi\trans\K_w^{-1}\|\|\bPhi_w\trans \bLambda^D R\|\\
    \leq&\frac{1}{\lambda_w\lambda_Q \eig_Q+\eig_\M^2}((1-\gamma)\lambda_w + \frac{1+\gamma}{\eig_w})
    :=\rad_\xi\numberthis\label{def:Rxi}
\end{align*}
\end{proof}

Given the special property of $Z_0$ and $\Xi_0$, we force $Z$ and $\Xi$ satisfying the following condition:
\begin{condition}\label{cond:subset_z0_xi0}
    $Z_0\subseteq Z$, $\Xi_0\subseteq \Xi$. 
\end{condition}
As a direct result, we have:
\begin{align*}
    &\|\nz \max_{\zeta\in\mR^u}\min_{\xi\in\mR^u}\cL^D(\theta, \zeta^*_\theta, \xi^*_\theta)\|=\|\nz \max_{\zeta\in Z}\min_{\xi\in \Xi}\cL^D(\theta, \zeta^*_\theta, \xi^*_\theta)\|=0,\\
    &\|\nx \max_{\zeta\in\mR^u}\min_{\xi\in\mR^u}\cL^D(\theta, \zeta^*_\theta, \xi^*_\theta)\|=\|\nx \max_{\zeta\in Z}\min_{\xi\in \Xi}\cL^D(\theta, \zeta^*_\theta, \xi^*_\theta)\|=0.
\end{align*}

\begin{restatable}{property}{CondVar}[Variance of Estimated Gradient]\label{cond:variance}
Under Assumption \ref{assump:smooth}, \ref{assump:feature_matrices} and \ref{assump:detailed_variance}, given convex sets $\Theta, Z, \Xi$, where $Z$ and $\Xi$ have finite diameter $C_\cW$ and $C_\cQ$, then there exists constants $\sigma_\theta, \sigma_\zeta, \sigma_\xi$, such that, for arbitrary $\theta,\zeta,\xi\in \Theta\times Z\times\Xi$, we have:
\begin{align*}
\EE_{s,a,r,s',a',s_0,a_0}[\|\nt& \cL^{(s,a,r,s',a',s_0,a_0)}(\theta,\zeta,\xi)-\nt \cL^D(\theta,\zeta,\xi)\|^2] \leq \sigma_\theta^2;\\
\EE_{s,a,r,s',a',s_0,a_0}[\|\nz &\cL^{(s,a,r,s',a',s_0,a_0)}(\theta,\zeta,\xi)-\nz \cL^D(\theta,\zeta,\xi)\|^2] \leq \sigma_\zeta^2;\\
\EE_{s,a,r,s',a',s_0,a_0}[\|\nx &\cL^{(s,a,r,s',a',s_0,a_0)}(\theta,\zeta,\xi)-\nx \cL^D(\theta,\zeta,\xi)\|^2] \leq \sigma_\xi^2.
\end{align*}

Here we use $\EE_{s,a,r,s',a',s_0,a_0}[\cdot]$ as a shorthand of $\EE_{(s,a,r,s')\sim d^D, a'\sim \pi(\cdot|s'), s_0\sim \nu^D_0,a_0\sim \pi(\cdot|s_0)}[\cdot]$,
and use $\nabla\cL^{(s,a,r,s',a',s_0,a_0)}(\theta, \zeta, \xi)$ to denote the stochastic gradient estimated using a single data point:
\begin{align*}
    \nt&\cL^{(s,a,r,s',a',s_0,a_0)}(\theta, \zeta, \xi)=(1-\gamma)\qx(s_0,a_0)\nt\log\pi_\theta(a_0|s)+\gamma \wz(s,a)\qx(s',a')\nt\log\pi_\theta(a'|s'),\\
    \nz&\cL^{(s,a,r,s',a',s_0,a_0)}(\theta, \zeta, \xi)=\Big(r+\gamma \qx(s',a')-\qx(s,a)\Big)\nz \wz(s,a)-\lambda_w w_\zeta(s,a)\nz w_\zeta(s,a),\\
    \nx&\cL^{(s,a,r,s',a',s_0,a_0)}(\theta, \zeta, \xi)=(1-\gamma)\nx\qx(s_0,a_0)+\wz(s,a)\nx\Big(\gamma \qx(s',a')-\qx(s,a)\Big)+\lambda_Q Q_\xi(s,a)\nx Q_\xi(s,a).\numberthis\label{eq:one_sample_estimation}
\end{align*}
\end{restatable}
\begin{proof}
Under Linear case and Assumption \ref{assump:detailed_variance}, we should have
\begin{align*}
    &\EE_{s,a,r,s',a',s_0,a_0}[\|\nt \cL^{(s,a,r,s',a',s_0,a_0)}(\theta,\zeta,\xi)-\nt \cL^D(\theta,\zeta,\xi)\|^2] \\
    \leq& 2(1-\gamma)^2\EE[\|Q_\xi(s_0,a_0)\nt\log\pi_\theta(a_0|s_0)-\EE[Q_\xi(s_0,a_0)\nt\log\pi_\theta(a_0|s_0)]\|^2]\\
    &+2\gamma^2 \EE[\|w_\zeta(s,a)Q_\xi(s',a')\nt\log\pi_\theta(a'|s')-\EE[w_\zeta(s,a)Q_\xi(s',a')\nt\log\pi_\theta(a'|s')]\|^2]\\
    \leq& 2(1-\gamma)^2\EE[\|Q_\xi(s_0,a_0)\nt\log\pi_\theta(a_0|s_0)\|^2]+2\gamma^2 \EE[\|w_\zeta(s,a)Q_\xi(s',a')\nt\log\pi_\theta(a'|s')]\|^2]\\
    \leq& 2(1-\gamma)^2 C_\cQ^2 G^2+2\gamma^2 C_\cW^2 C_\cQ^2 G^2\\
    &\EE_{s,a,r,s',a',s_0,a_0}[\|\nz \cL^{(s,a,r,s',a',s_0,a_0)}(\theta,\zeta,\xi)-\nz \cL^D(\theta,\zeta,\xi)\|^2] \\
    \leq& 3\EE[\|\bphi_w(s,a)(\gamma \bphi_Q(s',a')-\bphi_Q(s,a))\trans\xi-\EE[\bphi_w(s,a)(\gamma \bphi_Q(s',a')-\bphi_Q(s,a))\trans\xi]\|^2]\\
    &+3\EE[\|\bphi_w(s,a)r-\EE[\bphi_w(s,a)r]\|^2]+3\lambda_w^2\EE[\|\bphi_w(s,a)\bphi_w(s,a)\trans\zeta-\EE[\bphi_w(s,a)\bphi_w(s,a)\trans\zeta]\|^2]\\
    \leq& 3\sigma_R^2 + 3\sigma^2_\M C_\cQ^2 + 3\lambda_w^2\sigma_\K^2 C_\cW^2\\
    &\EE_{s,a,r,s',a',s_0,a_0}[\|\nx \cL^{(s,a,r,s',a',s_0,a_0)}(\theta,\zeta,\xi)-\nx \cL^D(\theta,\zeta,\xi)\|^2]\\
    \leq& 3(1-\gamma)^2\EE[\|\bphi_Q(s_0,a_0)-\EE[\bphi_Q(s_0,a_0)]\|^2]+3_Q^2\EE[\|\bphi_Q(s,a)\bphi_Q(s,a)\trans\xi-\EE[\bphi_Q(s,a)\bphi_Q(s,a)\trans\xi]\|^2]\\
    &+3\EE[\|\zeta\bphi_w(s,a)(\gamma \bphi_Q(s',a')-\bphi_Q(s,a))\trans-\EE[\zeta\bphi_w(s,a)(\gamma \bphi_Q(s',a')-\bphi_Q(s,a))\trans]\|^2]\\
    \leq& 3(1-\gamma)^2\sigma^2_{\nu}+3\sigma^2_\M C_\cW^2 + 3\lambda_Q^2\sigma^2_\K C_\cQ^2
\end{align*}

which finishes the proof.
\end{proof}

In Linear setting, $\sigma_\theta^2, \sigma_\zeta^2, \sigma_\xi^2$ can be chosen as:
\begin{align*}
    \sigma_\theta^2=&2(1-\gamma)^2\sigma_\nu^2 G^2 C_\cQ^2+ 2\gamma^2\sigma^2_\M G^2 C_\cW^2 C_\cQ^2,\\
    \sigma_\zeta^2=& 3\sigma_R^2 + 3\sigma^2_\M C_\cQ^2 + 3\lambda_w^2\sigma_\K^2 C_\cW^2,\\
    \sigma_\xi^2=& 3(1-\gamma)^2\sigma^2_{\nu}+3\sigma^2_\M C_\cW^2 + 3\lambda_Q^2\sigma^2_\K C_\cQ^2.
\end{align*}
In the following, we will use $\sigma$ to refer to the $\max\{\sigma_{\theta},\sigma_{\zeta},\sigma_\xi\}$ value satisfying Property \ref{cond:variance}. 

Finally, we prove a condition which is useful in the analysis of our second strategy.
We first introduce some new notations. Suppose we have a mini batch data $\batch$ sampled according to $d^D$ whose batch size is constant $|\batch|$. Then, we denote the average batch gradients as
\begin{align*}
    \nabla\cL^\batch(\theta, \zeta, \xi)=&\frac{1}{|\batch|}\sum_{i=1}^{|\batch|}\nabla \cL^{(s^i,a^i,r^i,{s'}^i,{a'}^i,s_0^i,a_0^i)}(\theta, \zeta, \xi)
\end{align*}
where $\nabla \cL^{(s^i,a^i,r^i,{s'}^i,{a'}^i,s_0^i,a_0^i)}(\theta, \zeta, \xi)$ is defined in Eq.\eqref{eq:one_sample_estimation}. 

\begin{restatable}{property}{CondLips}\label{cond:smooth_SVRE}    
    Under Assumption \ref{assump:smooth} and \ref{assump:feature_matrices}, there exists two constants two constants $\bar{L}_\zeta$ and $\bar{L}_\xi$, such that:
    \begin{align*}
        &\EE_{\batch\sim d^D}[\|\nz \cL^\batch(\theta, \zeta_1, \xi_1)-\nz \cL^\batch(\theta, \zeta_2, \xi_2)\|^2 + \|\nx \cL^\batch(\theta, \zeta_1, \xi_1)-\nx \cL^\batch(\theta, \zeta_2, \xi_2)\|^2]\\
        \leq& \EE_{\batch\sim d^D}\Big[\bar{L}_\zeta \Big(\nz \cL^\batch(\theta, \zeta_1, \xi_1)-\nz \cL^\batch(\theta, \zeta_2, \xi_2)\Big)\trans(\zeta_2-\zeta_1) + \bar{L}_\xi \Big(\nx \cL^\batch(\theta, \zeta_1, \xi_1)-\nx \cL^\batch(\theta, \zeta_2, \xi_2)\Big)\trans(\xi_1-\xi_2)\Big],\\
        &\EE_{\batch\sim d^D}[\|\nz \cL^\batch(\theta, \zeta_1, \xi_1)-\nz \cL^\batch(\theta, \zeta_2, \xi_2)\|^2 + \|\nx \cL^\batch(\theta, \zeta_1, \xi_1)-\nx \cL^\batch(\theta, \zeta_2, \xi_2)\|^2]\\
        \leq& \bar{L}^2_\zeta \|\zeta_1-\zeta_2\|^2 + \bar{L}^2_\xi\|\xi_1-\xi_2\|^2.
    \end{align*}    
\end{restatable}

\begin{proof}
For simplicity, we use $\K_w^\batch$ to denote matrix $\EE_\batch[\bphi_w(s,a)\bphi_w(s,a)\trans]$ ($\K_Q^\batch$ is similar) and use ${\M^\batch_\pi}$ to denote $\EE_\batch[\bphi(s,a)\bphi(s,a)\trans-\gamma\bphi(s,a)\bphi(s',\pi)\trans]$
\begin{align*}
    \nz \cL^\batch(\theta, \zeta_1, \xi_1)-\nz \cL^\batch(\theta, \zeta_2, \xi_2)=-\lambda_w \K_w^\batch(\zeta_1-\zeta_2)-{\M^\batch_\pi}(\xi_1-\xi_2)\\
    \nx \cL^\batch(\theta, \zeta_1, \xi_1)-\nx \cL^\batch(\theta, \zeta_2, \xi_2)=\lambda_Q \K_Q^\batch(\xi_1-\xi_2)-{\M^\batch_\pi}\trans(\zeta_1-\zeta_2)
\end{align*}
Therefore,
\begin{align*}
    &\EE_{\batch\sim d^D}[\|\nz \cL^\batch(\theta, \zeta_1, \xi_1)-\nz \cL^\batch(\theta, \zeta_2, \xi_2)\|^2 + \|\nx \cL^\batch(\theta, \zeta_1, \xi_2)-\nx \cL^\batch(\theta, \zeta_2, \xi_2)\|^2]\\
    \leq& 2\EE_{\batch\sim d^D}[(\zeta_1-\zeta_2)\trans (\lambda_w^2 (\K_w^\batch)\trans \K_w^\batch+{\M^\batch_\pi}\trans {\M^\batch_\pi}) (\zeta_1-\zeta_2)]\\
    &+2 \EE_{\batch\sim d^D}[(\xi_1-\xi_2)\trans (\lambda_Q^2 (\K_Q^\batch)\trans \K_Q^\batch+{\M^\batch_\pi}\trans {\M^\batch_\pi}) (\xi_1-\xi_2)]\\
    \leq& 2\EE_{\batch\sim d^D}[(\zeta_1-\zeta_2)\trans (\lambda_w^2 (\K_w^\batch)^2+(1+\gamma)^2I) (\zeta_1-\zeta_2)]+2 \EE_{\batch\sim d^D}[(\xi_1-\xi_2)\trans (\lambda_Q^2 (\K_Q^\batch)^2+(1+\gamma)^2I) (\xi_1-\xi_2)]\\
    \leq& 2\EE_{\batch\sim d^D}[(\zeta_1-\zeta_2)\trans (\lambda_w^2 (\K_w^\batch)^2+(1+\gamma)^2I) (\zeta_1-\zeta_2)]+2 \EE_{\batch\sim d^D}[(\xi_1-\xi_2)\trans (\lambda_Q^2 (\K_Q^\batch)^2+(1+\gamma)^2I) (\xi_1-\xi_2)]\\
    =&(\zeta_1-\zeta_2)\trans (2\lambda_w^2 \K_w+2(1+\gamma)^2I) (\zeta_1-\zeta_2)+(\xi_1-\xi_2)\trans (2\lambda_Q^2 \K_Q+2(1+\gamma)^2I) (\xi_1-\xi_2)
\end{align*}
In the first inequality, we use Young's inequality; in the second one, we use the fact that the largest singular value of ${\M^\batch_\pi}$ is less than $(1+\gamma)$; the third one is because all eigenvalues of $\K_w^\batch$ and $\K_Q^\batch$ locate in [0, 1], and we should have $I\succ \K_w^\batch\succ (\K_w^\batch)^2$ and $I\succ \K_Q^\batch\succ (\K_Q^\batch)^2$. Notice that, 
\begin{align*}
    &\EE_{\batch\sim d^D}\Big[- \Big(\nz \cL^\batch(\theta, \zeta_1, \xi_1)-\nz \cL^\batch(\theta, \zeta_2, \xi_2)\Big)\trans(\zeta_1-\zeta_2) \\
    &~~~~~~~~~~+ \Big(\nx \cL^\batch(\theta, \zeta_1, \xi_1)-\nx \cL^\batch(\theta, \zeta_2, \xi_2)\Big)\trans(\xi_1-\xi_2)\Big]\\
    =&\lambda_w (\zeta_1-\zeta_2)\trans \K_w (\zeta_1-\zeta_2)+\lambda_Q (\xi_1-\xi_2)\trans \K_Q(\xi_1-\xi_2)
\end{align*}
Therefore,
\begin{align*}
    &(\zeta_1-\zeta_2)\trans (2\lambda^2_w \K_w+2(1+\gamma)^2I) (\zeta_1-\zeta_2)+(\xi_1-\xi_2)\trans (2\lambda^2_Q \K_Q+2(1+\gamma)^2I) (\xi_1-\xi_2)\\
    \leq&\frac{2\max\{\lambda^2_w, \lambda^2_Q\} + 2(1+\gamma)^2}{\min\{\lambda_w\eig_w, \lambda_Q\eig_Q\}}\Big(\lambda_w (\zeta_1-\zeta_2)\trans \K_w (\zeta_1-\zeta_2)+\lambda_Q (\xi_1-\xi_2)\trans \K_Q(\xi_1-\xi_2)\Big)
\end{align*}
Moreover,
\begin{align*}
    &(\zeta_1-\zeta_2)\trans (2\lambda_w^2 \K_w+2(1+\gamma)^2I) (\zeta_1-\zeta_2)+(\xi_1-\xi_2)\trans (2\lambda_Q^2 \K_Q+2(1+\gamma)^2I) (\xi_1-\xi_2)\\
    \leq& (2\max\{\lambda_w^2, \lambda^2_Q\} + 2(1+\gamma)^2)\Big((\zeta_1-\zeta_2)\trans (\zeta_1-\zeta_2)+ (\xi_1-\xi_2)\trans (\xi_1-\xi_2)\Big)
\end{align*}
As a result, Assumption \ref{cond:smooth_SVRE} holds with 
$$
\bL_\zeta=\bL_\xi=\max\{\frac{2\max\{\lambda^2_w, \lambda^2_Q\} + 2(1+\gamma)^2}{\min\{\lambda_w\eig_w, \lambda_Q\eig_Q\}},\sqrt{2\max\{\lambda_w^2, \lambda^2_Q\} + 2(1+\gamma)^2}\}
$$

\end{proof}

\section{The Analysis of Bias}\label{appx:bound_error}
We first prove two propositions, which are crucial for analyzing the biases due to the finite dataset and mis-specified function classes.

\begin{restatable}{proposition}{PropEpsData}\label{prop:eps_data}
    For abitrary $\pi\in\Theta$, we have:
    \begin{align*}
        \|\nt\max_{w\in\cW}\min_{Q\in \cQ}\cL(\pi_\theta, w, Q)-\nt\max_{w\in\cW}\min_{Q\in \cQ}\cL^D(\pi_\theta, w, Q)\|\leq \underbrace{(2\kappa_\zeta\kappa_\xi+2\kappa_\zeta+2\kappa_\xi+\sqrt{2}/2)\sqrt{2\bar{\epsilon}_{data}}}_{denoted~as~\epsilon_{data}}
    \end{align*}
    where $\bar{\epsilon}_{data}$ is defined in Definition \ref{def:gen_error}.
\end{restatable}
\begin{proof}
For the simplicity of notation, we give the proof for a fixed $\pi$. 

Denote $(w^*_\mu, Q^*_\mu)$ parameterized by $(\zeta^*_\mu, \xi^*_\mu)$ as $\arg\max_{w\in\cW}\min_{Q\in\cQ} \cL(\pi, w, Q)$ and denote $(w^*, Q^*)$ parameterized by $(\zeta^*, \xi^*)$ as $\arg\max_{w\in\cW}\min_{Q\in\cQ} \cL^D(\pi, w, Q)$. First, we try to bound $\zeta^*-\zeta^*_\mu$. We use $Q_w$ and $Q^D_w$ (parameterized by $\xi_w$ and $\xi_w^D$) as the short notes of $\arg\min_{Q\in\cQ} \cL(\pi, w, Q)$ and $\arg\min_{Q\in\cQ} \cL^D(\pi, w, Q)$, respectively. Then,
\begin{align*}
&|\cL(\pi, w, Q_w)-\cL^D(\pi, w, Q_w^D)|\leq \max\{\cL(\pi, w, Q^D_w)-\cL^D(\pi, w, Q^D_w), \cL^D(\pi, w, Q_w)-\cL(\pi, w, Q_w)\}\leq \bar{\epsilon}_{data}
\end{align*}
As a result,
\begin{align*}
    0\leq&\cL^D(\pi, w^*, Q^*)-\min_{Q\in\cQ}\cL^D(\pi, w^*_\mu, Q)\\
    \leq& \cL^D(\pi, w^*, Q^*)-\min_{Q\in\cQ} \cL(\pi, w^*, Q)+\cL(\pi, w^*_\mu, Q^*_\mu)-\min_{Q\in\cQ}\cL^D(\pi, w^*_\mu, Q)\\
    \leq&2\bar{\epsilon}_{data}
\end{align*}
According to Lemma \ref{lem:optimum_function}, $\min_{Q\in\cQ}\cL^D(\pi, w, Q)$ is $\mz$-strongly concave. Therefore,
\begin{align*}
    \|\zeta^*-\zeta^*_\mu\|\leq \frac{2}{\mz}\sqrt{\cL^D(\pi, w^*, Q^*)-\min_{Q\in\cQ}\cL^D(\pi, w^*_\mu, Q)}\leq \frac{2}{\mz}\sqrt{2\bar{\epsilon}_{data}}
\end{align*}
Next, we bound $\|\xi^*-\xi^*_\mu\|$. For arbitrary $\pi\in\Pi$ and $w\in\cW$, we have:
\begin{align*}
    0\leq \cL^D(\pi, w, Q_w)-\cL^D(\pi, w, Q_w^D)\leq \cL^D(\pi, w, Q_w)-\cL(\pi, w, Q_w)+\cL(\pi, w, Q_w^D)-\cL^D(\pi, w, Q_w^D)\leq 2\bar{\epsilon}_{data}
\end{align*}
Since $L^D$ is $\mx$ strongly-convex, as a result of Lemma \ref{lem:grad_norm}, for arbitrary $w$,
\begin{align}\label{eq:bound_xiw_xiwD}
    \|\xi_w-\xi_w^D\|\leq \frac{2}{\mx}\sqrt{2\bar{\epsilon}_{data}}
\end{align}
Then, we have
\begin{align*}
    \|\xi^*-\xi^*_\mu\|\leq& \|\xi^*-\arg\min_{\xi\in\Xi}\cL^D(\pi, w^*_\mu, Q_\xi)\|+\|\arg\min_{\xi\in\Xi}\cL^D(\pi, w^*_\mu, Q_\xi)-\xi^*_\mu\|\\
    =&\|\xi^*-\arg\min_{\xi\in\Xi}\cL^D(\pi, w^*_\mu, Q_\xi)\|+\|\arg\min_{\xi\in\Xi}\cL^D(\pi, w^*_\mu, Q_\xi)-\arg\min_{\xi\in\Xi}\cL(\pi, w^*_\mu, Q_\xi)\|\\
    \leq& \frac{L}{\mx} \|\zeta^*-\zeta^*_\mu\|+\frac{2}{\mx}\sqrt{2\bar{\epsilon}_{data}}\\
    \leq& (\frac{2L}{\mx\mz}+\frac{2}{\mx})\sqrt{2\bar{\epsilon}_{data}}
\end{align*}
where in the last but two step, we use Lemma \ref{lem:optimum_function}-(1).

As a directly application of Property \ref{prop:detailed_scscs}, we have:
\begin{align*}
    &\|\nt\max_{w\in\cW}\min_{Q\in \cQ}\cL(\pi_\theta, w, Q)-\nt\max_{w\in\cW}\min_{Q\in \cQ}\cL^D(\pi_\theta, w, Q)\|\\
    =&\|\nt\cL(\pi_\theta, w^*_\mu, Q^*_\mu)-\nt\cL^D(\pi_\theta, w^*_\mu, Q^*_\mu)\|+\|\nt\cL^D(\pi_\theta, w^*_\mu, Q^*_\mu)-\nt\cL^D(\pi_\theta, w^*, Q^*)\|\\
    \leq& \sqrt{\bar{\epsilon}_{data}} + L\|\zeta^*-\zeta^*_\mu\|+L\|\xi^*-\xi^*_\mu\|\\
    \leq & (2\kappa_\zeta\kappa_\xi+2\kappa_\zeta+2\kappa_\xi+\sqrt{2}/2)\sqrt{2\bar{\epsilon}_{data}}
\end{align*}
\end{proof}
    
\begin{restatable}{proposition}{PropEpsWEpsQ}\label{prop:eps_cW_eps_cQ}
    For arbitrary $\pi\in\Pi$, we have:
    \begin{align*}
        \EE_{d^\mu}&[|w^*_\mu(s,a)-w^{\pi}_\cL(s,a)|^2] \leq \epsilon_\cW:=4\frac{\lambda_{\max}^2}{\lambda_Q\lambda_w}\epsilon_1+2\frac{\lambda_{\max}}{\mu_\zeta}\epsilon_2\\
        \EE_{d^\mu}&[|Q^*_\mu(s,a)-Q^{\pi}_\cL(s,a)|^2]\leq \epsilon_\cQ :=8\frac{\lambda_{\max}^3}{\lambda^2_Q\lambda_w}\epsilon_1+(2+4\frac{\lambda_{\max}^2}{\lambda_Q\mu_\zeta})\epsilon_2
    \end{align*}
    where $(w^*_\mu, Q^*_\mu)$ denotes the saddle point of $ \cL(\pi, w, Q)$ constrained by $w,Q\in\cW\times\cQ$ (i.e. $\zeta\in Z, \xi\in\Xi$), $(w^\pi_\cL, Q^\pi_\cL)$ denotes the saddle point of $\cL(\pi, w, Q)$ without any constraint on $w$ and $Q$ (i.e. $w$ and $Q$ can be arbitrary vectors in $\mR^{|\cS||\cA|}$), $\lambda_{\max}=\max\{\lambda_Q, \lambda_w\}$, $\mu_\zeta$ is defined in Property \ref{prop:scscs}.
\end{restatable}
\begin{proof}
In the following, we will frequently consider two loss functions. The first one is $\cL(\pi, w, Q)$ defined in Eq.\eqref{eq:problem_general}, where $w$ and $Q$ are parameterized by $\zeta$ and $\xi$, respectively, and we will write $(w, Q)\in\cW\times\cQ$. The second one is $\cF(\pi, x, y)$ defined by:
\begin{align*}
    \cF(\pi, x, y)=&(1-\gamma)(\nu_0^\pi)\trans \bLambda^{-1/2}y+x\trans \Big(\bLambda^{1/2}R-(\I-\gamma \bLambda^{1/2}P^\pi\bLambda^{-1/2})y\Big) +\frac{\lambda_Q}{2}y\trans y-\frac{\lambda_w}{2}x\trans x
\end{align*}
where $(x,y)\in \mR^{|\cS||\cA|}\times \mR^{|\cS||\cA|}$. For simplification, in the following, we will use $\max_{x} \min_{y}$ as a short note of $\max_{x\in\mR^{|\cS||\cA|}}\min_{y\in\mR^{|\cS||\cA|}}$.

As we can see, the difference between $\cL(\pi, w, Q)$ and $\cF(\pi, x,y)$ is not only that we don't have any constraint on $x$ and $y$, but also that we absorb one $\bLambda^{1/2}$ into vector $x$ and $y$. In another word, for arbitrary $\pi, w, Q$, we have 
$$
\cL(\pi, w, Q)=\cF(\pi, \bLambda^{1/2}w, \bLambda^{1/2}Q).
$$
Obviously, $\cF(\pi, x, y)$ is $\lambda_w$-strongly-concave-$\lambda_Q$-strongly-convex and $\lambda_{\max}$-smooth w.r.t. $x,y\in\mR^{|\cS||\cA|}$.

In the following, we use $w^*_\mR$ parameterized by $\zeta^*_\mR$ to denote $\arg\max_{w\in\cW}\min_{y}\cF(\pi,\bLambda^{1/2}w, y)$. According to Lemma \ref{lem:optimum_function}, $\min_{y}\cF(\pi,x, y)$ is a $2\frac{\lambda_{\max}^2}{\lambda_Q}$-smooth and $\lambda_w$-strongly-concave function with gradient $\nabla_x \min_{y}\cF(\pi,x, y)$. Since $\nabla_x \cF(\pi, \bLambda^{1/2}w^\pi_\cL, \bLambda^{1/2}Q^\pi_\cL)=0$, we have,
\begin{align*}
    &\frac{\lambda_w}{2}\|\bLambda^{1/2}w^*_\mR-\bLambda^{1/2}w^\pi_\cL\|^2\\
    \leq&\cF(\pi,\bLambda^{1/2}w^\pi_\cL, \bLambda^{1/2}Q^\pi_\cL)-\min_{y}\cF(\pi,\bLambda^{1/2}w^*_\mR, y)\tag{Strong concavity of $\min_y \cF(\pi, x, y)$}\\
    =&\cF(\pi,\bLambda^{1/2}w^\pi_\cL, \bLambda^{1/2}Q^\pi_\cL)-\max_{w\in\cW}\min_{y}\cF(\pi,\bLambda^{1/2}w, y)\\
    \leq& \cF(\pi,\bLambda^{1/2}w^\pi_\cL, \bLambda^{1/2}Q^\pi_\cL)-\min_{y}\cF(\pi,\bLambda^{1/2}w_{\zeta^\pi}, y)\tag{$w_{\zeta^\pi}$ is defined in Def. \ref{def:misspeci}}\\
    \leq& \frac{\lambda_{\max}^2}{\lambda_Q}\|\bLambda^{1/2}w_{\zeta^\pi}-\bLambda^{1/2}w^\pi_\cL\|^2\tag{Smoothness of $\min_y \cF(\pi, x, y)$}\\
    =&\frac{\lambda_{\max}^2}{\lambda_Q}\|w_{\zeta^\pi}-w^\pi_\cL\|_\bLambda^2=\frac{\lambda_{\max}^2}{\lambda_Q}\epsilon_1 \tag{see definition of $\epsilon_1$ in Def.\ref{def:misspeci}}
\end{align*}
which implies
\begin{align*}
    \|\bLambda^{1/2}w^*_\mR-\bLambda^{1/2}w^\pi_\cL\|^2\leq 2\frac{\lambda_{\max}^2}{\lambda_Q\lambda_w}\epsilon_1\numberthis\label{eq:bound_wstarR_wcLpi}
\end{align*}
Applying Lemma \ref{lem:optimum_function} for $(w,Q)\in\cW\times\cQ$, we know $\min_{\xi\in\Xi}\cL(\pi, \wz, \qx)$ is $\mu_\zeta$-strongly-concave w.r.t. $\zeta$. Since $\zeta^*$ is the minimizer of $\min_{\xi\in\Xi}\cL(\pi, \wz, \qx)$ and $Z$ is a convex set, we have

\begin{align*}
    \frac{\mu_\zeta}{2}\|\zeta^*-\zeta^*_\mR\|^2\leq& \cL(\pi, w^*_\mu, Q^*_\mu)-\min_{Q\in \cQ} \cL(\pi, w^*_\mR, Q)\tag{Strong concavity of $\min_{Q\in\cQ}\cL(\pi, w, Q)$; Lemma \ref{lem:grad_norm}}\\
    =& \cF(\pi,\bLambda^{1/2}w^*_\mu, \bLambda^{1/2}Q^*_\mu)-\min_{Q\in\cQ} \cF(\pi,\bLambda^{1/2}w^*_\mR, \bLambda^{1/2}Q)\\
    \leq& \cF(\pi,\bLambda^{1/2}w^*_\mu, \bLambda^{1/2}Q^*_\mu)-\min_{y} \cF(\pi,\bLambda^{1/2}w^*_\mR, y)\\
    \leq& \cF(\pi,\bLambda^{1/2}w^*_\mu, \bLambda^{1/2}Q^*_\mu)-\min_{y} \cF(\pi,\bLambda^{1/2}w^*_\mu, y)\tag{Because $w_\mR^*=\arg\max_{w\in\cW}\min_{y} \cF(\pi,\bLambda^{1/2}w, y)$}\\
    \leq&\frac{\lambda_{\max}}{2}\|\bLambda^{1/2}Q^*_\mu-\arg\min_{y}\cF(\pi,\bLambda^{1/2} w^*_\mu, y)\|^2\tag{Smoothness of $\cF(\pi, x, y)$ for fixed $x$ and $\nabla_y \min_y \cF = 0$}\\
    \leq& \frac{\lambda_{\max}}{2}\epsilon_2
\end{align*}

In the last but two inequality, we use the fact that $\cF(\pi,\bLambda^{1/2}w^*_\mu, \cdot)$ is $\lambda_{\max}$-smooth and $\nabla_y \min_{y}\cF(\pi,\bLambda^{1/2}w^*_\mu, Q)=0$; in the last equality, we use the definition of $\epsilon_2$ in Def. \ref{def:misspeci}. Combing (b) in Property \ref{prop:detailed_scscs} with $L_w=1$, for arbitrary $s,a\in\cS\times\cA$, we have:
\begin{align}
    |w^*_\mu(s,a)-w^*_\mR(s,a)|^2\leq \|\zeta^*-\zeta^*_\mR\|^2\leq \frac{\lambda_{\max}}{\mu_\zeta}\epsilon_2\label{eq:bound_wstar_wstarR}
\end{align}
Therefore, as a result of Eq.\eqref{eq:bound_wstarR_wcLpi} and Eq.\eqref{eq:bound_wstar_wstarR}:
\begin{align*}
    \EE_{d^\mu}[|w^*_\mu-w^\pi_\cL|^2]\leq& 2\EE_{d^\mu}[|w^*_\mR-w^\pi_\cL|^2]+2\EE_{d^\mu}[|w^*_\mR-w^*_\mu|^2]\\
    =&2\|\bLambda^{1/2}w^*_\mR-\bLambda^{1/2}w^\pi_\cL\|^2+2\EE_{d^\mu}[|w^*_\mR-w^*_\mu|^2]\\
    \leq&4\frac{\lambda_{\max}^2}{\lambda_Q\lambda_w}\epsilon_1+2\frac{\lambda_{\max}}{\mu_\zeta}\epsilon_2
\end{align*}
According to Lemma \ref{lem:optimum_function} again, $\arg\min_{y}\cF(\pi,x, y)$ is $\frac{\lambda_{\max}}{\lambda_Q}$-Lipschitz w.r.t. $x$, we have
\begin{align*}
    &\EE_{d^\mu}[|Q^*_\mu-Q^\pi_\cL|^2]=\|\bLambda^{1/2}Q^*_\mu-\bLambda^{1/2}Q^\pi_\cL\|^2\\
    \leq& 2\underbrace{\|\bLambda^{1/2}Q^*_\mu-\arg\min_{y}\cF(\pi,\bLambda^{1/2}w^*_\mu, Q)\|^2}_{bounded~by~\epsilon_2}+2\|\arg\min_{y}\cF(\pi,\bLambda^{1/2}w^*_\mu, y)-\bLambda^{1/2}Q^\pi_\cL\|^2\\
    \leq& 2\epsilon_2 + 2\frac{\lambda_{\max}}{\lambda_Q}\|\bLambda^{1/2}w^*_\mu-\bLambda^{1/2}w^\pi_\cL\|^2\leq 8\frac{\lambda_{\max}^3}{\lambda^2_Q\lambda_w}\epsilon_1+(2+4\frac{\lambda_{\max}^2}{\lambda_Q\mu_\zeta})\epsilon_2
\end{align*}
As a result,
\begin{align*}
    \epsilon_\cW = 4\frac{\lambda_{\max}^2}{\lambda_Q\lambda_w}\epsilon_1+2\frac{\lambda_{\max}}{\mu_\zeta}\epsilon_2;~~~\epsilon_\cQ=8\frac{\lambda_{\max}^3}{\lambda^2_Q\lambda_w}\epsilon_1+(2+4\frac{\lambda_{\max}^2}{\lambda_Q\mu_\zeta})\epsilon_2
\end{align*}


\end{proof}
    
\begin{theorem}[Bias resulting from regularization]\label{thm:biased_solution}
Let's rewrite Eq.\eqref{eq:problem_general} in a vector-matrix form:

\begin{align*}
    \max_{w\in\cW} \min_{Q\in\cQ} \cL(\pi,w,Q):= (1-\gamma)(\nu_0^\pi)\trans Q+w\trans \bLambda \Big(R-(I-\gamma \bP^\pi)Q\Big) +\frac{\lambda_Q}{2}Q\trans \bLambda Q-\frac{\lambda_w}{2}w\trans \bLambda w
\end{align*}
where $\nu_0^\pi$ and $\bP^\pi$ denotes the initial state-action distribution and the transition matrix w.r.t. policy $\pi$, respectively; $\bLambda\in\mR^{|\cS||\cA|\times|\cS||\cA|}$ denotes the diagonal matrix whose diagonal elements are $d^\mu(\cdot,\cdot)$.
Denote $(w^\pi_\cL, Q^\pi_\cL)$ as the saddle point of $\cL(\pi, w, Q)$ without any constraint on $w$ and $Q$ (i.e. $\cW=\cQ=\mathbb{R}^{|\cS||\cA|}$), then we have:
\begin{align*}
w^\pi_\cL=&w^\pi + \Big(\lambda_w \lambda_Q I + (I-\gamma \bP^\pi) \bLambda^{-1}(I-\gamma \bP^\pi_*)\bLambda\Big)^{-1}\Big( \lambda_Q R-  \lambda_Q \lambda_w  w^\pi\Big)\\
Q^\pi_\cL=&Q^\pi - \Big(\lambda_w \lambda_Q I +  \bLambda^{-1}(I-\gamma \bP^\pi_*)\bLambda (I-\gamma \bP^\pi) \Big)^{-1}\Big(\lambda_w \lambda_Q Q^\pi+ \lambda_w (1-\gamma)\bLambda^{-1}\nu_0^\pi)\Big)
\end{align*}
where $w^\pi=\frac{d^\pi}{d^\mu}$ is the density ratio and $Q^\pi$ is the Q function of $\pi$. we use $\bP_*^\pi=(\bP^\pi)\trans$ to denote the transpose of the transition matrix.
\end{theorem}
\begin{proof}
Recall the loss function
\begin{align*}
    &\cL(\pi,w,Q)= (1-\gamma) (\nu_0^\pi)\trans Q + w\trans \bLambda R - w\trans\bLambda(I-\gamma \bP^\pi)Q+\frac{\lambda_Q}{2}Q\trans \bLambda Q-\frac{\lambda_w}{2}w\trans \bLambda w 
\end{align*}
By taking the derivatives w.r.t. $Q$, since $\Lambda$ is invertible, the optimal choice of $Q$ should be:
\begin{align*}
    Q =  \frac{1}{\lambda_Q} \bLambda^{-1} ( (I-\gamma \bP^\pi_*)\bLambda w- (1-\gamma) \nu_0^\pi)
\end{align*}
Plug this result in, and we have
\begin{align*}
    \cL(\pi,w,Q)=& -\frac{1}{2\lambda_Q}\Big((1-\gamma)\nu_0^\pi -(I-\gamma \bP^\pi_*)\bLambda w\Big)\trans  \bLambda^{-1} \Big((1-\gamma) (\nu_0^\pi)- (I-\gamma \bP^\pi_*)\bLambda w\Big) + w\trans \bLambda R - \frac{\lambda_w}{2}w\trans \bLambda w    
\end{align*}
Taking the derivative w.r.t. $w$, and set it to 0:
\begin{align*}
    0 = \frac{1}{\lambda_Q}\bLambda(I-\gamma \bP^\pi)  \bLambda^{-1} \Big((1-\gamma) (\nu_0^\pi)- (I-\gamma \bP^\pi_*)\bLambda w\Big) + \bLambda R -  \lambda_w \bLambda w
\end{align*}
As a result,
\begin{align*}
    w^\pi_\cL =& \Big(\lambda_w I + \frac{1}{\lambda_Q}  (I-\gamma \bP^\pi) \bLambda^{-1}(I-\gamma \bP^\pi_*)\bLambda\Big)^{-1}\Big(\frac{1}{\lambda_Q} (I-\gamma \bP^\pi) \bLambda^{-1}(1-\gamma) \nu_0^\pi+ R\Big)\\
    =& \Big(\lambda_w \lambda_Q I + (I-\gamma \bP^\pi)  \bLambda^{-1}(I-\gamma \bP^\pi_*)\bLambda\Big)^{-1}\Big( (I-\gamma \bP^\pi)  \bLambda^{-1}(I-\gamma \bP^\pi_*)\bLambda\bLambda^{-1}(I-\gamma \bP^\pi_*)^{-1}(1-\gamma) \nu_0^\pi+ \lambda_Q R\Big)\\
    =&w^\pi + \Big(\lambda_w \lambda_Q I + (I-\gamma \bP^\pi) \bLambda^{-1}(I-\gamma \bP^\pi_*)\bLambda\Big)^{-1}\Big( \lambda_Q R-  \lambda_Q \lambda_w  w^\pi\Big)\\
\end{align*}
and
\begin{align*}
    Q^\pi_\cL=& \frac{1}{\lambda_Q} \bLambda^{-1}\Big((I-\gamma \bP^\pi_*)\bLambda w^\pi_\cL- (1-\gamma) \nu_0^\pi \Big)\\
    =&\frac{1}{\lambda_Q} \bLambda^{-1}\Big((I-\gamma \bP^\pi_*)\bLambda w^\pi_\cL-  (I-\gamma \bP^\pi_*)\bLambda w^\pi \Big)\\
    =&\frac{1}{\lambda_Q} \bLambda^{-1}(I-\gamma \bP^\pi_*)\bLambda\Big(\lambda_Q\lambda_w \bLambda +  \bLambda(I-\gamma \bP^\pi)  \bLambda^{-1}(I-\gamma \bP^\pi_*)\bLambda\Big)^{-1}\Big(\lambda_Q\bLambda R-\lambda_Q\lambda_w \bLambda w^\pi\Big)\\
    =&\Big(\lambda_w\lambda_Q (I-\gamma \bP^\pi_*)^{-1}\bLambda + \bLambda(I-\gamma \bP^\pi) \Big)^{-1}\Big(\bLambda R-\lambda_w \bLambda w^\pi\Big)\\
    =&\Big(\lambda_w\lambda_Q (I-\gamma \bP^\pi_*)^{-1}\bLambda + \bLambda(I-\gamma \bP^\pi) \Big)^{-1}\Big(\bLambda (I-\gamma \bP^\pi) Q^\pi-\lambda_w \bLambda w^\pi\Big)\\
    =&Q^\pi - \Big(\lambda_w\lambda_Q (I-\gamma \bP^\pi_*)^{-1}\bLambda + \bLambda(I-\gamma \bP^\pi) \Big)^{-1}\Big(\lambda_w \lambda_Q (I-\gamma \bP^\pi_*)^{-1}\bLambda Q^\pi+\lambda_w \bLambda w^\pi\Big)\\
    =&Q^\pi - \Big(\lambda_w \lambda_Q I +  \bLambda^{-1}(I-\gamma \bP^\pi_*)\bLambda (I-\gamma \bP^\pi) \Big)^{-1}\Big(\lambda_w \lambda_Q Q^\pi+ \lambda_w (1-\gamma)\bLambda^{-1}\nu_0^\pi)\Big)
\end{align*}
\end{proof}

\begin{lemma}\label{lem:bound_bias_weighted_norm}
Under Assumption \ref{assump:dataset}:
\begin{align*}
    \|w^\pi - w^\pi_\cL\|_\bLambda^2\leq& \frac{C^2(\lambda_Q + \lambda_Q \lambda_w C)^2}{(1-\gamma)^4},\quad\|Q^\pi-Q^\pi_\cL\|_\bLambda^2\leq \frac{C^2}{(1-\gamma)^2}(\frac{\lambda_w\lambda_Q }{1-\gamma}+\lambda_w)^2
\end{align*}
where $(w^\pi, Q^\pi)$ and $(w^\pi_\cL, Q^\pi_\cL)$ are defined in Theorem \ref{thm:biased_solution}. $\|x\|_\bLambda=x\trans\bLambda x$ denotes the norm of column vector $x$ weighted by $\bLambda$.
\end{lemma}
\begin{proof}
From Theorem \ref{thm:biased_solution}, we have
\begin{align*}
    w^\pi_\cL=&w^\pi+\Big(\lambda_w\lambda_Q I + (I-\gamma \bP^\pi) \bLambda^{-1}(I-\gamma \bP^\pi_*)\bLambda\Big)^{-1}\Big( \lambda_QR-\lambda_Q\lambda_w w^\pi\Big)\\
    Q^\pi_\cL=&Q^\pi - \Big(\lambda_w\lambda_Q I + \bLambda^{-1}(I-\gamma \bP^\pi_*)\bLambda(I-\gamma \bP^\pi) \Big)^{-1}\Big(\lambda_w\lambda_Q Q^\pi+ \lambda_w(1-\gamma)\bLambda^{-1}\nu_0^\pi)\Big)
\end{align*}
We use $\bm{1}\in\mR^{|\cS||\cA|\times 1}$ to denote a vector whose all elements are 1. Then, we have
\begin{align*}
    \|w^\pi - w^\pi_\cL\|_\bLambda^2 =& \|\Big(\lambda_w \lambda_Q I + (I-\gamma \bP^\pi) \bLambda^{-1}(I-\gamma \bP^\pi_*)\bLambda\Big)^{-1}\Big( \lambda_Q R-  \lambda_Q \lambda_w  w^\pi\Big)\|^2_\bLambda\\
    =&\|\Big(\lambda_w \lambda_Q I + \bLambda^{1/2} (I-\gamma \bP^\pi) \bLambda^{-1}(I-\gamma \bP^\pi_*)\bLambda^{1/2}\Big)^{-1}\bLambda^{1/2}\Big( \lambda_Q R-  \lambda_Q \lambda_w  w^\pi\Big)\|^2\\
    \leq & \|\bLambda^{-1/2}(I-\gamma \bP^\pi_*)^{-1}\bLambda(I-\gamma \bP^\pi)^{-1}\Big( \lambda_Q R-\lambda_Q \lambda_w w^\pi\Big)\|^2\\
    =& \|\bLambda^{-1/2}(I-\gamma \bP^\pi_*)^{-1}\bLambda \tilde{Q}^\pi\|^2\\
    \leq & \frac{(\lambda_Q + \lambda_Q \lambda_w C)^2}{(1-\gamma)^2} \|\bLambda^{-1}(I-\gamma \bP^\pi_*)^{-1}\bLambda \bm{1}\|^2_\bLambda\\
    =& \frac{(\lambda_Q + \lambda_Q \lambda_w C)^2}{(1-\gamma)^2} \|\bLambda^{-1}(I-\gamma \bP^\pi_*)^{-1} d^\mu\|^2_\bLambda\\
    =& \frac{(\lambda_Q + \lambda_Q \lambda_w C)^2}{(1-\gamma)^4} \|w^\pi_{d^\mu}\|^2_\bLambda\leq  \frac{C^2(\lambda_Q + \lambda_Q \lambda_w C)^2}{(1-\gamma)^4}
\end{align*}
where in the first inequality, we use Lemma \ref{lem:non_decreasing_order}; in the third equality, we use $\tilde{Q}^\pi$ to denote the Q function after replacing true rewards with $\lambda_Q R - \lambda_Q \lambda_w w^\pi$; in the second inequality, we use Lemma \ref{lem:all_non_negative} and the result that $|\lambda_Q R - \lambda_Q \lambda_w w^\pi|\leq \lambda_Q + \lambda_Q \lambda_w C$ given Assumption \ref{assump:dataset}; in the last inequality, we use Assumption \ref{assump:dataset} again. Similarly,
\begin{align*}
    \|Q^\pi-Q^\pi_\cL\|_\bLambda^2\leq&\|\Big(\lambda_w \lambda_Q I +  \bLambda^{-1}(I-\gamma \bP^\pi_*)\bLambda (I-\gamma \bP^\pi) \Big)^{-1}\Big(\lambda_w \lambda_Q Q^\pi+ \lambda_w (1-\gamma)\bLambda^{-1}\nu_0^\pi)\Big)\|^2_\bLambda\\
    =&\|\Big(\lambda_Q\lambda_w I + \bLambda^{-1/2}(I-\gamma \bP^\pi_*)\bLambda (I-\gamma \bP^\pi)\bLambda^{-1/2} \Big)^{-1}\bLambda^{1/2}\Big(\lambda_Q\lambda_w Q^\pi+ \lambda_w(1-\gamma) \bLambda^{-1}\nu_0^\pi)\Big)\|^2\\
    \leq&\|\bLambda^{1/2}(I-\gamma \bP^\pi)^{-1}\bLambda^{-1}(I-\gamma \bP^\pi_*)^{-1}\Big(\lambda_w \lambda_Q \bLambda Q^\pi+ \lambda_w(1-\gamma)\nu_0^\pi)\Big)\|^2\\
    = & \|\lambda_w \lambda_Q\bLambda^{1/2}(I-\gamma \bP^\pi)^{-1}\bLambda^{-1}(I-\gamma \bP^\pi_*)^{-1} \bLambda Q^\pi+\lambda_w \bLambda^{1/2}(I-\gamma \bP^\pi)^{-1}w^\pi)\|^2 \\
    \leq & \|\frac{\lambda_w \lambda_Q}{1-\gamma}\bLambda^{1/2}(I-\gamma \bP^\pi)^{-1}\bLambda^{-1}(I-\gamma \bP^\pi_*)^{-1} \bLambda \bm{1}+\lambda_w \bLambda^{1/2}(I-\gamma \bP^\pi)^{-1}w^\pi)\|^2 \\
    \leq & \|(I-\gamma \bP^\pi)^{-1}\big(\frac{\lambda_w \lambda_Q}{1-\gamma}w^\pi_{d^\mu}+\lambda_w w^\pi\big)\|_\bLambda^2\\
    \leq&\frac{C^2}{(1-\gamma)^2}(\frac{\lambda_w\lambda_Q }{1-\gamma}+\lambda_w)^2
\end{align*}
where in the last but third inequality, we use Lemma \ref{lem:all_non_negative} and the fact that $w^\pi$ is also non-negative.
\end{proof}

\begin{lemma}\label{lem:formula_transform}
Under Assumption \ref{assump:dataset}, for arbitrary function $f(s,a)$,
\begin{align*}
    &(1-\gamma)\EE_{s_0\sim \nu_0, a_0\sim \pi}[f(s_0,a_0)]+\gamma\EE_{s,a,s'\sim d^\mu, a'\sim \pi}[w^{\pi}(s,a)f(s',a')]=\EE_{d^\mu}[w^{\pi}(s,a)f(s,a)]\numberthis\label{eq:property_fsa}\\
    &\gamma\EE_{s,a,s'\sim d^\mu, a'\sim \pi}[f^2(s',a')]
    \leq  \frac{1}{1-\gamma} \EE_{s,a\sim d^\pi_{d^\mu}}[f^2(s,a)]
    \leq\frac{C}{1-\gamma} \EE_{s,a\sim d^\mu}[f^2(s,a)]\numberthis\label{eq:property_gsa}
\end{align*}
where $d^\pi_{d^\mu}:=(1-\gamma)\EE_{\tau\sim \pi, s_0,a_0\sim d^\mu(\cdot, \cdot)}[\sum_{t=0}^\infty \gamma^t p(s_t=s,a_t=a)]$ is the normalized discounted state-action occupancy by treating $d^\mu(\cdot,\cdot)$ as initial distribution; $s,a,s'\sim d^\mu, a'\sim \pi$ is a short note of $s,a\sim d^\mu, s'\sim P(s'|s,a), a'\sim \pi(\cdot|s')$.
\end{lemma}
\begin{proof}
Eq.\eqref{eq:property_fsa} can be proved by the equation:
\begin{align*}
    d^\pi(s,a)=(1-\gamma)\nu_0(s)\pi(a|s)+\gamma \sum_{s',a'} p(s|s',a')d^\pi(s',a')\pi(a|s)
\end{align*}
For Eq.\eqref{eq:property_gsa}, the first step is because $\gamma \sum_{s',a'} d^\mu(s',a')p(s|s',a')\pi(a|s)
\leq \frac{1}{1-\gamma}d^\pi_{d^\mu}(s,a)$, and the second step is the result of Assumption \ref{assump:dataset}.
\end{proof}

\Biasedness*

\begin{proof}
Firstly, by applying the triangle inequality:
\begin{align*}
    \|\nt \max_{w\in\cW}\min_{Q\in\cQ}\cL^D(\pi_\theta,w,Q)-\nt J(\pi_\theta)\|  \leq&\underbrace{\|\nt\max_{w\in\cW}\min_{Q\in \cQ}\cL^D(\pi_\theta, w, Q)-\nt\max_{w\in\cW}\min_{Q\in \cQ}\cL(\pi_\theta, w, Q)\|}_{Bounded~in~Proposition~\ref{prop:eps_data}}\\
    &+\underbrace{\|\nt\max_{w}\min_{Q}\cL(\pi_\theta, w, Q)-\nt\max_{w\in\cW}\min_{Q\in \cQ}\cL(\pi_\theta, w, Q)\|}_{t_1}\\
    &+\underbrace{\|\nt J(\pi_\theta)-\nt\max_{w}\min_{Q}\cL(\pi_\theta, w, Q)\|}_{t_2}
\end{align*}
where we use $\max_w \min_Q$ as a short note of $\max_{w\in \mR^{|\cS||\cA|}}\min_{Q\in\mR^{|\cS||\cA|}}$.

In the following, we again use $(w^{\pi_\theta}_\cL, Q^{\pi_\theta}_\cL)$ to denote the saddle point of $ \cL(\pi_\theta,w,Q)$ without any constraint on $w$ and $Q$, and use $(w^*_\mu, Q^*_\mu)$ to denote the saddle point of $\cL(\pi_\theta,w,Q)$. Next, we upper bound $t_1$ and $t_2$ one by one. For simplicity, we use $s,a,s'\sim d^\mu, a'\sim \pi_\theta$ as a short note of $s,a\sim d^\mu, s'\sim P(s'|s,a), a'\sim \pi_\theta(\cdot|s')$.
\paragraph{Upper bound $t_1$}
With misspecification Definition \ref{def:misspeci}, we can easily bound $t_1$:

\begin{align*}
    t_1 =& \|\nt \cL(\pi_\theta,w^*_\mu, Q^*_\mu)-\nt \cL(\pi_\theta,w^{\pi_\theta}_\cL,Q^{\pi_\theta}_\cL)\|\\
    \leq& \frac{1}{1-\gamma}\|(1-\gamma)\EE_{\nu_0^{\pi_\theta}}[\Big(Q^*_\mu(s_0,a_0)-Q_\cL^{\pi_\theta}(s_0,a_0)\Big)\nt \log\pi_\theta(a_0|s_0)]\|\\
        &+\frac{\gamma}{1-\gamma}\|\EE_{s,a,s'\sim d^\mu, a'\sim \pi}[w^*_\mu(s,a)\Big(Q^*_\mu(s',a')-Q^{\pi_\theta}_\cL(s',a')\Big)\nt\log\pi(a'|s')]\|\\
        &+\frac{\gamma}{1-\gamma}\|\EE_{s,a,s'\sim d^\mu, a'\sim \pi}[(w^*_\mu(s,a)-w_\cL^{\pi_\theta}(s,a))\Big(Q^*_\mu(s',a')-Q^{\pi_\theta}_\cL(s',a')\Big)\nt\log\pi(a'|s')]\|\\
        &+\frac{\gamma}{1-\gamma}\|\EE_{s,a,s'\sim d^\mu, a'\sim \pi_\theta}[(w^*_\mu(s,a)-w_\cL^{\pi_\theta}(s,a))Q^*_\mu(s',a')\nt\log\pi(a'|s')]\|\\
    \leq & \frac{G}{1-\gamma}\EE_{\nu_0^{\pi_\theta}}[|Q^*_\mu(s,a)-Q_\cL^{\pi_\theta}(s,a)|] 
        +\frac{\gamma C_\cW G}{1-\gamma}\EE_{s,a,s' \sim d^\mu, a'\sim \pi_\theta}[|Q^*_\mu(s',a')-Q^{\pi_\theta}_\cL(s',a')|]\tag{$(1-\gamma)\nu_0^\pi(s,a)\leq d^\pi(s,a)\leq C d^\mu(s,a)$}\\
        &+\frac{\gamma G}{1-\gamma}\EE_{s,a,s'\sim d^\mu, a'\sim \pi_\theta}[|(w^*_\mu(s,a)-w_\cL^{\pi_\theta}(s,a))\Big(Q^*_\mu(s',a')-Q^{\pi_\theta}_\cL(s',a')\Big)|]\\
        &+\frac{\gamma C_\cQ G}{1-\gamma}\EE_{s,a,s'\sim d^\mu, a'\sim \pi_\theta}[|w^*_\mu(s,a)-w_\cL^{\pi_\theta}(s,a)|]\\
    \leq& \frac{G}{1-\gamma}\sqrt{\EE_{\nu_0^{\pi_\theta}}[|Q^*_\mu(s,a)-Q_\cL^{\pi_\theta}(s,a)|^2]}
    +\frac{\gamma C_\cW G}{1-\gamma}\sqrt{\EE_{s,a,s'\sim d^\mu,a'\sim \pi_\theta}[|Q^*_\mu(s',a')-Q^{\pi_\theta}_\cL(s',a')|^2]}\\
        &+ \frac{\gamma G}{1-\gamma}\sqrt{\EE_{ d^\mu}[|w^{\pi_{\theta}}_\cL(s,a)-w^*_\mu(s,a)|^2] \EE_{s,a,s'\sim d^\mu, a'\sim \pi_\theta}[|Q^{\pi_{\theta}}(s',a')- Q^{\pi_{\theta}}_\cL(s',a')|^2|]} \\
        &+\frac{\gamma C_\cQ G}{1-\gamma} \sqrt{\EE_{d^\mu}[|w^*_\mu(s,a)-w^{\pi_{\theta}}_\cL(s,a))|^2]}\\
    \leq& \frac{G}{1-\gamma}\sqrt{C\EE_{d^\mu}[|Q^*_\mu(s,a)-Q_\cL^{\pi_\theta}(s,a)|^2]}
        +\frac{ C_\cW G}{1-\gamma}\sqrt{\frac{\gamma C}{1-\gamma}\EE_{d^\mu}[|Q^*_\mu(s,a)-Q^{\pi_\theta}_\cL(s,a)|^2]}\\
        &+ \frac{ G}{1-\gamma}\sqrt{\frac{\gamma C}{1-\gamma} \EE_{ d^\mu}[|w^{\pi_{\theta}}_\cL(s,a)-w^*_\mu(s,a)|^2] \EE_{ d^\mu}[|Q^{\pi_{\theta}}(s,a)- Q^{\pi_{\theta}}_\cL(s,a)|^2|]} \\
        &+\frac{\gamma C_\cQ G}{1-\gamma} \sqrt{\EE_{d^\mu}[|w^*_\mu(s,a)-w^{\pi_{\theta}}_\cL(s,a))|^2]}\\
    \leq & \frac{G}{1-\gamma}\Big(\sqrt{C\epsilon_\cQ}+  C_\cW\sqrt{\frac{\gamma\epsilon_\cQ C}{1-\gamma}}+\sqrt{\frac{\gamma\epsilon_\cQ \epsilon_\cW C}{1-\gamma}}+\gamma C_\cQ \sqrt{\epsilon_\cW}\Big)
\end{align*}
In the last equation, we first use Eq.\eqref{eq:property_gsa} in Lemma \ref{lem:formula_transform}, and then apply Proposition \ref{prop:eps_cW_eps_cQ}.

\paragraph{Upper bound $t_2$}
Similarly, we can give a bound for $t_2$:
\begin{align*}
    t_2 =& \|\nt J(\pi_{\theta})-\nt \cL(\pi_\theta,w^{\pi_\theta}_\cL,Q^{\pi_\theta}_\cL))\|\\
    \leq& \frac{1}{1-\gamma}\|(1-\gamma)\EE_{\nu_0^{\pi_\theta}}[\Big(Q^{\pi_\theta}(s_0,a_0)-Q_\cL^{\pi_\theta}(s_0,a_0)\Big)\nt \log\pi_\theta(a_0|s_0)]\\
        &+\gamma\EE_{d^\mu}[w^{\pi_\theta}(s,a)\Big(Q^{\pi_\theta}(s',a')-Q^{\pi_\theta}_\cL(s',a')\Big)\nt\log\pi(a'|s')]\|\\
        &+\frac{\gamma}{1-\gamma}\|\EE_{d^\mu}[(w^{\pi_\theta}(s,a)-w_\cL^{\pi_\theta}(s,a))\Big(Q^{\pi_\theta}(s',a')-Q^{\pi_\theta}_\cL(s',a')\Big)\nt\log\pi(a'|s')]\|\\
        &+\frac{\gamma}{1-\gamma}\|\EE_{d^\mu}[(w^{\pi_\theta}(s,a)-w_\cL^{\pi_\theta}(s,a))Q^{\pi_\theta}(s',a')\nt\log\pi(a'|s')]\|\\
    =& \frac{1}{1-\gamma}\|\EE_{d^\mu}[w^{\pi_\theta}(s,a)\Big(Q^{\pi_\theta}(s,a)-Q^{\pi_\theta}_\cL(s,a)\Big)\nt\log\pi(a|s)]\|\tag{Eq.\eqref{eq:property_fsa} in Lemma \ref{lem:formula_transform}}\\
        &+\frac{\gamma}{1-\gamma}\|\EE_{s,a,s'\sim d^\mu, a'\sim \pi_\theta}[(w^{\pi_\theta}(s,a)-w_\cL^{\pi_\theta}(s,a))\Big(Q^{\pi_\theta}(s',a')-Q^{\pi_\theta}_\cL(s',a')\Big)\nt\log\pi(a'|s')]\|\\
        &+\frac{\gamma}{1-\gamma}\|\EE_{s,a,s'\sim d^\mu, a'\sim \pi_\theta}[(w^{\pi_\theta}(s,a)-w_\cL^{\pi_\theta}(s,a))Q^{\pi_\theta}(s',a')\nt\log\pi(a'|s')]\|\\
    \leq& \frac{CG}{1-\gamma}\EE_{d^\mu}[|Q^{\pi_\theta}(s,a)-Q^{\pi_\theta}_\cL(s,a)|]\\
        &+\frac{\gamma G}{1-\gamma}\EE_{s,a,s'\sim d^\mu, a'\sim \pi_\theta}[|(w^{\pi_\theta}(s,a)-w_\cL^{\pi_\theta}(s,a))\Big(Q^{\pi_\theta}(s',a')-Q^{\pi_\theta}_\cL(s',a')\Big)|]\\
        &+\frac{\gamma G}{(1-\gamma)^2}\EE_{s,a,s'\sim d^\mu, a'\sim \pi_\theta}[|w^{\pi_\theta}(s,a)-w_\cL^{\pi_\theta}(s,a)|]\\
    \leq& \frac{CG}{1-\gamma}\sqrt{\EE_{ d^\mu}[|Q^{\pi_{\theta}}- Q^{\pi_{\theta}}_\cL|^2]}+ \frac{\gamma G}{(1-\gamma)^2}\sqrt{\EE_{d^\mu}[|(w^{\pi_{\theta}}(s,a)-w^{\pi_{\theta}}_\cL(s,a)|^2]}\\
    &+ \frac{\gamma G}{1-\gamma}\sqrt{\EE_{ d^\mu}[|w^{\pi_{\theta}}_\cL(s,a)-w^{\pi_\theta}(s,a)|^2] \EE_{s,a,s'\sim d^\mu, a'\sim \pi_\theta}[|Q^{\pi_{\theta}}(s',a')- Q^{\pi_{\theta}}_\cL(s',a')|^2|]} \\
    \leq& \frac{CG}{1-\gamma}\sqrt{\EE_{ d^\mu}[|Q^{\pi_{\theta}}- Q^{\pi_{\theta}}_\cL|^2]}+ \frac{\gamma G}{(1-\gamma)^2}\sqrt{\EE_{d^\mu}[|(w^{\pi_{\theta}}(s,a)-w^{\pi_{\theta}}_\cL(s,a)|^2]}\\
    &+ \frac{G}{1-\gamma}\sqrt{\frac{\gamma C}{1-\gamma}\EE_{ d^\mu}[|w^{\pi_{\theta}}_\cL(s,a)-w^{\pi_\theta}(s,a)|^2] \EE_{ d^\mu}[|Q^{\pi_{\theta}}(s,a)- Q^{\pi_{\theta}}_\cL(s,a)|^2|]} \tag{Eq.\ref{eq:property_gsa} in Lemma \ref{lem:formula_transform}}\\
    \leq & \frac{G}{1-\gamma}\Big(\frac{C^2}{(1-\gamma)}(\frac{\lambda_w\lambda_Q }{1-\gamma}+\lambda_w)+\frac{\gamma C(\lambda_Q + \lambda_Q \lambda_w C)}{(1-\gamma)^3}+\frac{C^2 (\lambda_Q + \lambda_Q \lambda_w C)}{(1-\gamma)^3}(\frac{\lambda_w\lambda_Q }{1-\gamma}+\lambda_w)\sqrt{\frac{\gamma C}{1-\gamma}}\Big)
\end{align*}

\end{proof}

\section{Missing Examples and Proofs for Strategy 1}\label{appx:P_SREDA}

\subsection{Equivalence between Stationary Points}\label{appx:alg:Direct_SGDA}
\begin{restatable}{theorem}{EquiStationary}[Equivalence Between Stationary Points]\label{thm:equi_stationary}
Under Assumptions in Section \ref{sec:assumptions}, given $Z$ and $\Xi$ with finite $C_\cQ$ and $C_\cW$, suppose there is an Algorithm provides us with one stationary point $(\theta_T, \zeta_T, \xi_T)$ of the non-concave-strongly-convex objective $\max_{\theta, \zeta}\min_{\xi}\cL^D(\theta, \zeta, \xi)$ after running $T$ iterations, statisfying the following conditions in expectation over the randomness of algorithm.
\begin{align}
    \EE[\|\nabla_{\theta,\zeta} \cL^D(\theta_T, \zeta_T, \phi_{\theta_T}(\zeta_T))\|]:=&\EE[\|\nt \cL^D(\theta_T, \zeta_T, \phi_{\theta_T}(\zeta_T))\|+\|\nz \cL^D(\theta_T, \zeta_T, \phi_{\theta_T}(\zeta_T))\|]\nonumber\\
    \leq&\frac{\epsilon}{(\kappa_\xi +1)(\kappa_\zeta+1)} \label{eq:bound_gt_gz}
\end{align}
where $\phi_\theta(\zeta)=\argmin_{\xi\in \Xi} \cL^D(\theta, \zeta, \xi)$ and $\kappa_\zeta$ and $\kappa_\xi$ are the condition numbers, then in expectation $\theta_T$ is a biased stationary point satisfying Eq.\eqref{eq:biased_cond}.
\end{restatable}
\begin{proof}

Eq.\eqref{eq:bound_gt_gz} implies that
\begin{align*}
    \max\{\EE[\|\nt \cL^D(\theta_T, \zeta_T, \phi_{\theta_T}(\zeta_T))\|], \EE[\|\nz \cL^D(\theta_T, \zeta_T, \phi_{\theta_T}(\zeta_T))\|]\} \leq \frac{\epsilon}{(\kappa_\xi +1)(\kappa_\zeta+1)} \numberthis\label{eq:bound_exp_nt_cL}
\end{align*}
We can upper bounded $\EE[\|\nt J(\pi_{\theta_T})\|]$ with the triangle inequality:
\begin{align*}
\EE[\|\nt J(\pi_{\theta_T})\|]\leq& \underbrace{\EE[\|\nt \cL^D(\theta_T,\zeta_T,\phi_{\theta_T}(\zeta_T))\|]}_{Bounded~in~Eq.\eqref{eq:bound_exp_nt_cL}} +\EE[\|\nt \cL^D(\theta_T,\zeta^*,\xi^*)-\nt \cL^D(\theta_T,\zeta_T,\phi_{\theta_T}(\zeta_T)))\|]\\
&+\underbrace{\EE[\|\nt \cL^D(\theta_T,\zeta^*,\xi^*)-\nt J(\pi_{\theta_T})\|]}_{Bounded~in~Theorem \ref{thm:biasedness}}\\
\leq &\frac{\epsilon}{(\kappa_\xi +1)(\kappa_\zeta+1)} +\epsilon_{func} + \epsilon_{reg}+\epsilon_{data}+\EE[\|\nt \cL^D(\theta_T,\zeta^*,\xi^*)-\nt \cL^D(\theta_T,\zeta_T,\phi_{\theta_T}(\zeta_T)))\|]
\end{align*}
where we use $\zeta^*, \xi^*$ to denote the saddle-point of $\max_{\zeta\in Z}\min_{\xi\in\Xi} \cL^D(\theta_T, \zeta,\xi)$; in the last inequality we use Eq.\eqref{eq:bound_exp_nt_cL} and Theorem \ref{thm:biasedness}. 

Next, we try to bound the last term. According to the definition, $\zeta^*$ is also the maximum of function $\Phi_{\theta_T}(\cdot)=\min_{\xi\in\Xi} \cL^D(\theta_T, \cdot, \xi)$ defined in Lemma \ref{lem:optimum_function}. Applying Property (2) in Lemma \ref{lem:optimum_function}, Lemma \ref{lem:grad_norm}, and inequality \eqref{eq:bound_exp_nt_cL}, we obtain that
\begin{align*}
    \|\zeta_T - \zeta^*\| \leq \frac{1}{\mu_\zeta}\|\Phi_{\theta_T}(\zeta_T)\| = \frac{1}{\mu_\zeta}\|\nz \cL^D(\theta_T, \zeta_T, \phi_{\theta_T}(\zeta_T))\| \leq \frac{\epsilon}{\mu_\zeta(\kappa_\xi+1)(\kappa_\zeta+1)}
\end{align*}
Then we can bound:
\begin{align*}
    &\|\nt \cL^D(\theta_T,\zeta^*,\xi^*)-\nt \cL^D(\theta_T,\zeta_T,\phi_{\theta_T}(\zeta_T))\|\\
    \leq & L\|\zeta_T -\zeta^*\|+L\|\xi^*-\phi_{\theta_T}(\zeta_T))\|=L\|\zeta_T -\zeta^*\|+L\|\phi_{\theta_T}(\zeta^*)-\phi_{\theta_T}(\zeta_T))\|\\
    \leq&(L+L\kappa_\xi)\|\zeta_T - \zeta^*\| \leq  \frac{\epsilon \kappa_\zeta}{1+\kappa_\zeta}
\end{align*}
where in the first inequality we use the smoothness Assumption \ref{assump:smooth}, and in the second inequality we use (1) in Lemma \ref{lem:optimum_function}. As a result, 
\begin{align*}
    \EE[\|\nt J(\pi_{\theta_T})\|]\leq&  \frac{\epsilon}{(\kappa_\xi +1)(\kappa_\zeta+1)} + \frac{\epsilon \kappa_\zeta}{1+\kappa_\zeta}+\epsilon_{func} + \epsilon_{reg}+\epsilon_{data} \\
    \leq& \epsilon +\epsilon_{func} + \epsilon_{reg}+\epsilon_{data}
\end{align*}
\end{proof}

In the following subsections, we will introduce the Projected-SREDA Algorithm revised from \citep{luo2020stochastic} and prove that it provide us a stationary points required by Theorem \ref{thm:equi_stationary}.

We choose $\Theta=\mathbb{R}^{\dims_\theta}$, $\Xi=\Xi_0$, and $Z =\{\zeta|\|\zeta\|\leq \rad'\}$, where $\Xi_0$ is defined in Property \ref{prop:radius} and $\rad'$ will be determined later. For simplicity, we use $\cL^D_-=-\cL^D$ to denote the minus of original loss function, which should be a non-convex-strongly-concave problem and aligns with the setting of \citep{luo2020stochastic}.

\subsection{Verification of the Assumptions in \citep{luo2020stochastic}}
In this section, we verify that Assumptions 1-5 in \citep{luo2020stochastic} are satisfied under our Assumption \ref{assump:smooth}, \ref{assump:feature_matrices} and \ref{assump:detailed_variance}.

\paragraph{Assumption 1}
\begin{align*}
    &\inf_{\theta\in\mathbb{R}^{\dims_\theta},\zeta\in Z}\max_{\xi\in\Xi}\cL^D_-(\theta, \zeta, \xi)\\
    \geq& - \max_{\theta\in\mathbb{R}^{\dims_\theta},\zeta\in Z,\xi\in\Xi}(1-\gamma)\|(\nu_D^\pi)\trans \bPhi_Q \xi\|+\|\zeta\|\|\bPhi_w\trans \bLambda^D R\|+ \|\zeta\| \|\M_\pi\| \|\xi\|+\frac{\lambda_Q}{2}\|\xi\|^2 \|\K_Q\| + \frac{\lambda_w}{2}\|\zeta\|\|\K_w\|\\
    \geq&-\max_{\theta\in\mathbb{R}^{\dims_\theta},\zeta\in Z,\xi\in\Xi}(1-\gamma)+\|\zeta\|+(1+\gamma)\|\zeta\|\|\xi\|++\frac{\lambda_Q}{2}\|\xi\|^2 + \frac{\lambda_w}{2}\|\zeta\|
\end{align*}
Because $\|\zeta\|$ and $\|\xi\|$ are bounded for arbitrary $\zeta,\xi\in Z\times\Xi$, Assumption 1 holds.

\paragraph{Assumption 2}
The proof is almost identical to the proof of Property \ref{prop:detailed_scscs}-(b) and we omit here. Assumption 2 holds by choosing $L$ according to \eqref{eq:choice_of_L}.
\paragraph{Assumption 3}
Under our linear function classes setting, it holds obviously.
\paragraph{Assumption 4}
Hold directly by choosing $\mu=\lambda_Q\eig_Q$.
\paragraph{Assumption 5}
Identicial to the Condition \ref{cond:variance}. We prove Condition \ref{cond:variance} holds in Appendix \ref{appx:properties} under our Assumptions.

\subsection{Useful Lemma}
In this subsection, we first prove several useful lemma.
\begin{lemma}
    Under Assumption \ref{assump:feature_matrices}, for arbitrary $\theta\in\Theta$ and $\xi\in\Xi$, the solution for $\max_{\zeta\in\mathbb{R}^{\dims_Z}}\cL^D(\theta, \zeta, \xi)$ has bounded $\ell_2$ norm.
\end{lemma}\label{lem:bounded_max_zeta}
\begin{proof}
Recall the loss function $\cL^D$:
\begin{align*}
    \cL^D(\pi, \zeta, \xi)= (1-\gamma)(\nu_D^\pi)\trans \bPhi_Q \xi+\zeta\trans \bPhi_w\trans \bLambda^D R -\zeta\trans \M_\pi \xi+\frac{\lambda_Q}{2}\xi\trans \K_Q \xi - \frac{\lambda_w}{2}\zeta\trans \K_w\zeta
\end{align*}
Taking the derivative w.r.t. $\zeta$ and set it to 0, we have:
\begin{align*}
    \zeta^* = \frac{1}{\lambda_w}\K_w^{-1}(\bPhi_w\trans\Lambda^D R-\M_\pi\xi)
\end{align*}
Given that $\|\xi\|\leq \rad_\xi$ for $\xi\in\Xi$, we have:
\begin{align*}
    \|\zeta^*\|\leq& \frac{1}{\lambda_w}\|\K_w^{-1}\|(\|\bPhi\trans\Lambda^D R\| + \|\M_\pi\|\|\xi\|)\leq\frac{1}{\lambda_w\eig_w}(1 + (1+\gamma)\rad_\xi)
\end{align*}
\end{proof}

In the following, we will use $R_0:=\frac{1}{\lambda_w\eig_w}(1 + (1+\gamma)\rad_\xi)$ as a shortnote. Next, we are ready to prove the following lemma which is crucial for the analysis of the effect of projection step.
\projectionError*
\begin{proof}
    First of all, if $\zeta^+_{k+1}\in Z$, then $\zeta_{k+1}-\zeta_{k+1}^+=0$, and the Lemma holds. Therefore, in the following, we only consider the case when $\zeta^+_{k+1}\notin Z$. Because $\zeta_k\in Z$, in the case, we must have $\|\vk^\zeta\|>0$.


    Because we are considering $Z$ is a high dimensional ball. For $\zeta^+_{k+1}\notin Z$, we have 
    \begin{align*}
        \zeta_{k+1}=P_Z(\zeta_{k+1}^+) = \zeta^+_{k+1}\frac{\rad'}{\|\zeta^+_{k+1}\|}
    \end{align*}
    which means,
    \begin{align*}
        \zeta_{k+1}-\zeta^+_{k+1} = (\frac{\rad'}{\|\zeta^+_{k+1}\|}-1) \zeta^+_{k+1}
    \end{align*}
    
    Denote $\zeta_k^* = \min_{\zeta\in Z}\cL_-^D(\theta_k, \zeta, \xi_k)$. Then we have:
    \begin{align*}
        \langle \nz \cL_-^D(\theta_k, \zeta_k, \xi_k), \zeta_k - \zeta_k^* \rangle \geq 0,~~~\|\zeta^*_k\|\leq R_0
    \end{align*}
    
    Then we have:    
    \begin{align*}
        &\frac{\langle \nz \cL_-^D(\theta_k, \zeta_k, \xi_k), \zeta_{k+1} - \zeta_{k+1}^+\rangle}{\|\nz \cL_-^D(\theta_k, \zeta_k, \xi_k)\|\|\vk^\zeta\|}\\
        =&(\frac{\rad'}{\|\zeta^+_{k+1}\|}-1)\frac{\langle \nz \cL_-^D(\theta_k, \zeta_k, \xi_k), \zeta_{k+1}^+\rangle}{\|\nz \cL_-^D(\theta_k, \zeta_k, \xi_k)\|\|\vk^\zeta\|}\\
        =&(\frac{\rad'}{\|\zeta^+_{k+1}\|}-1)\frac{\langle \nz \cL_-^D(\theta_k, \zeta_k, \xi_k), \zeta_k - \eta_k \vk^\zeta\pm\zeta^*_k\rangle}{\|\nz \cL_-^D(\theta_k, \zeta_k, \xi_k)\|\|\vk^\zeta\|}\\
        =&\underbrace{(\frac{\rad'}{\|\zeta^+_{k+1}\|}-1)}_{smaller~than~0}\underbrace{\frac{\langle \nz \cL_-^D(\theta_k, \zeta_k, \xi_k), \zeta_k - \zeta_k^* \rangle}{\|\nz \cL_-^D(\theta_k, \zeta_k, \xi_k)\|\|\vk^\zeta\|}}_{larger~than~0}+(\frac{\rad'}{\|\zeta^+_{k+1}\|}-1)\frac{\langle \nz \cL_-^D(\theta_k, \zeta_k, \xi_k), \zeta_k^* - \eta_k \vk^\zeta\rangle}{\|\nz \cL_-^D(\theta_k, \zeta_k, \xi_k)\|\|\vk^\zeta\|}\\
        \leq&(\frac{\rad'}{\|\zeta^+_{k+1}\|}-1)\frac{\langle \nz \cL_-^D(\theta_k, \zeta_k, \xi_k), \zeta_k^* - \eta_k \vk^\zeta\rangle}{\|\nz \cL_-^D(\theta_k, \zeta_k, \xi_k)\|\|\vk^\zeta\|}\\
        \leq& (1-\frac{\rad'}{\|\zeta^+_{k+1}\|})\frac{\|\zeta^*_k\|+\eta_k \|\vk^\zeta\|}{\|\vk^\zeta\|}\tag{$\|\zeta^+_{k+1}\|\geq \rad'$ and $\langle a, b \rangle \leq \|a\|\|b\|$}\\
        \leq&(1-\frac{\rad'}{\rad' + \eta_k\|\vk^\zeta\|}) \frac{R_0+\eta_k \|\vk^\zeta\|}{\|\vk^\zeta\|} \tag{$\|\zeta^+_{k+1}\|\leq \rad'+\eta_k\|\vk^\zeta\|$}\\
        =&\eta_k\frac{R_0+\eta_k \|\vk^\zeta\|}{\rad' + \eta_k\|\vk^\zeta\|}
    \end{align*}
    
    Because $\rad' =8\max\{R_0, 1\}$, and $\|\vk\|\geq \|\vk^\zeta\|$, we have
    \begin{align*}
        \frac{R_0+\eta_k\|\vk^\zeta\|}{\rad'+\eta_k \|\vk^\zeta\|}\leq& \frac{R_0+\eta_k\|\vk\|}{\rad'+\eta_k \|\vk\|}\leq \frac{R_0+\frac{\epsilon}{5\kappa_\xi L}}{\rad'+\frac{\epsilon}{5\kappa_\xi L}}
        \leq\frac{R_0+1}{\rad'+1}\leq \frac{1}{4}
    \end{align*}
    where in the third inequality we use the constraint that $\epsilon<1$ and the fact that $\kappa_\xi > 1, L>1$ in our setting.
\end{proof}

\subsection{Main Proofs for Theorem 3.1}
The proofs for Theorem \ref{thm:converge_rate_P_SREDA} is almost the same as those for the original SREDA algorithm. We will only show those key Lemmas or Theorems in \citep{luo2020stochastic} which need to be modified as a result of the additional projection step, and omit those untouched.
In the following, we will frequently use $x^+_{k+1}$ to denote $(\theta_{k+1}, \zeta^+_{k+1})$ before the projection and use $x_{k+1}$ to denote $(\theta_{k+1}, \zeta_{k+1})$ after the projection. 

First of all, the following condition still holds 
\begin{align*}
    \|x_{k+1}-x_k\|^2\leq \|x_{k+1}^+-x_k\|^2\leq \frac{\epsilon^2}{25\kappa_\xi^2}
\end{align*}
where the first inequality results from the property of projection and the second one holds because of the choice of learning rate $\eta_k$. The condition above corresponds to the $\|x_{k+1}-x_k\|^2\leq \epsilon^2_x$ in \citep{luo2020stochastic}. As a result, all the Lemmas and Theorems in the Appendix B of \citep{luo2020stochastic} still hold for our Projected-SREDA. Besides, because the PiSARAH will not be effected by our projection step, the results in Appendix C of \citep{luo2020stochastic} also holds.

Similarly, we consider the following decomposition:
\begin{align*}
    \cL^D_-(x_{k+1}, \xi_{k+1})-\cL^D_-(x_k, \xi_{k})=\underbrace{\cL^D_-(x_{k+1}, \xi_{k})-\cL^D_-(x_k, \xi_{k})}_{A_k}+\underbrace{\cL^D_-(x_{k+1}, \xi_{k+1})-\cL^D_-(x_{k+1}, \xi_{k})}_{B_k}
\end{align*}
Because the proof of Lemma 14 and Lemma 15 in \citep{luo2020stochastic} only depends on the previous lemmas, they still hold and we list them here.
\begin{lemma}\label{lem:bound_Bk}
    Under Assumptions of Theorem \ref{thm:converge_rate_P_SREDA}, we have $\EE[B_k]\leq \frac{134\epsilon^2}{\kappa_\xi L}$ for any $k\geq 1$
\end{lemma}
\begin{lemma}\label{lem:bound_grad_Phik}
    Under Assumptions of Theorem \ref{thm:converge_rate_P_SREDA}, we have 
    $$
    \EE\|\nabla_{\theta,\zeta}\max_{\xi\in\Xi}\cL^D_-(\theta,\zeta,\xi)\|\leq \EE\|\vk\|+\frac{15}{7}\epsilon
    $$
\end{lemma}

However, the final proof for Theorem \ref{thm:converge_rate_P_SREDA} can not be adapted from \citep{luo2020stochastic} directly because of the projection. In the following, we show our proof targeted at our Projected-SREDA:

\begin{restatable}{theorem}{ConvPSREDA}\label{thm:converge_rate_P_SREDA}
    For $\epsilon<1$, under Assumption \ref{assump:smooth}, \ref{assump:feature_matrices}, \ref{assump:detailed_variance}, with the following parameter choices:
    \begin{align*}
    &\rad'=8\max\{\frac{1}{\lambda_w\eig_w}(1 + (1+\gamma)\rad_\xi), 1\},\rad_\xi=\frac{1}{\lambda_w\lambda_Q \eig_Q+\eig_\M^2}((1-\gamma)\lambda_w + \frac{1+\gamma}{\eig_w}),\\
    &\eta_k=\min\Big(\frac{\epsilon}{5\kappa_\xi L\|\vk\|}, \frac{1}{10\kappa_\xi L}\Big),~~~\lambda=\frac{1}{8L},~~~S_1=\lceil\frac{2250}{19}\sigma^2\kappa_\xi^2\epsilon^2\rceil,~~~
    S_2=\lceil\frac{3687}{76}\kappa_\xi q\rceil,~~~q=\lceil\epsilon^{-1}\rceil,\\
    &\Delta \cL^D_- =  \max_{\theta\in\Theta, \zeta\in Z, \xi\in\Xi}  \cL^D_-(\theta,\zeta,\xi)-\cL^D_-(\theta_0,\zeta_0,\xi_0),~~~K=\lceil\frac{50\kappa_\xi L \Delta \cL^D_-}{\epsilon^2}\rceil~~~and~~~m=\lceil 1024\kappa_\xi\rceil
    \end{align*}
    and the same parameter choices for PiSADAH as in \citep{luo2020stochastic}, Algorithm \ref{alg:PLSO} outputs $\hat\theta, \hat\zeta$ such that
    \begin{align*}
        \EE[\|\ntz 
        \max_{\xi\in\Xi}\cL^D_-(\hat\theta,\hat\zeta, \xi)\|]\leq O(\epsilon)
    \end{align*}
    with $O(\kappa_\xi^3\epsilon^{-3})$ stochastic gradient evaluations.
\end{restatable}

\begin{proof}[Proof of Theorem \ref{thm:converge_rate_P_SREDA}]
    Based on the update rule of $\theta$ and $\zeta$ in Algorithm \ref{alg:P_SREDA}, we have:
    \begin{align*}
        A_k \leq& \langle \ntz \cL^D_-(x_k,\xi_k), x_{k+1}-x_k\rangle + \frac{L}{2}\|x_{k+1}-x_k\|^2\tag{Smoothness of $\cL^D_-$}\\
        \leq&\langle \ntz \cL^D_-(x_k,\xi_k), x_{k+1}^+-x_k\rangle+\langle \ntz \cL^D_-(x_k,\xi_k), x_{k+1}-x_{k+1}^+\rangle + \frac{L}{2}\|x_{k+1}^+-x_k\|^2\tag{the property of projection}\\
        =&-\eta_k\langle \ntz \cL^D_-(x_k,\xi_k), \vk\rangle+\langle \nz \cL^D_-(x_k,\xi_k), \zeta_{k+1}-\zeta_{k+1}^+\rangle + \frac{L\eta^2_k}{2}\|\vk\|^2\\
        \leq& -\eta_k\langle \ntz \cL^D_-(x_k,\xi_k), \vk\rangle+\frac{\eta_k}{4} \|\nz \cL^D_-(x_k,\xi_k)\| \|\vk^\zeta\| + \frac{L\eta^2_k}{2}\|\vk\|^2\tag{Lemma \ref{lem:small_angle}}\\
        \leq& -\eta_k\langle \ntz \cL^D_-(x_k,\xi_k), \vk\rangle+\frac{\eta_k}{4} \|\vk\|^2 + \frac{\eta_k}{4}\|\ntz \cL^D_-(x_k,\xi_k)-\vk\|\|\vk\| + \frac{L\eta^2_k}{2}\|\vk\|^2\tag{Triangle Ineq.; $\|\vk^\zeta\|\leq \|\vk\|; \|\nz\cL^D_-\|\leq \|\ntz\cL^D_-\|; \|\nz\cL^D_- -\vk^\zeta\|\leq \|\ntz\cL^D_- -\vk\|$}\\
        \leq& \frac{\eta_k}{2}\|\ntz \cL^D_-(x_k,\xi_k)-\vk\|^2-(\frac{\eta_k}{4}-\frac{L\eta_k^2}{2})\|\vk\|^2+ \frac{\eta_k}{4}\|\ntz \cL^D_-(x_k,\xi_k)-\vk\|\|\vk\|\tag{$\|\ntz \cL^D_-\|\geq 0$}
    \end{align*}

    The choice of the step size implies that
    \begin{align*}
        (\frac{\eta_k}{4}-\frac{L\eta_k^2}{2})\|\vk\|^2\geq& \min\{\frac{1}{40\kappa_\xi L}-\frac{1}{200\kappa_\xi^2 L}, \frac{\epsilon}{20\kappa_\xi L\|\vk\|}-\frac{\epsilon^2}{50\kappa_\xi^2 L \|\vk\|^2}\}\|\vk\|^2\\
        \geq&\min\{\frac{1}{50\kappa_\xi L},\frac{3\epsilon}{100\kappa_\xi L\|\vk\|} \}\|\vk\|^2\\
        \geq&\frac{\epsilon^2}{50\kappa_\xi L}\min\Big(\frac{\|\vk\|}{\epsilon}, \frac{\|\vk\|^2}{2\epsilon^2}\Big)\\
        \geq&\frac{\epsilon^2}{50\kappa_\xi L} (\frac{\|\vk\|}{\epsilon}-2))\tag{$\min(|x|, x^2/2)\geq |x|-2$}\\
        = & \frac{1}{50\kappa_\xi L}(\epsilon\|\vk\|-2\epsilon^2)
    \end{align*}

    Therefore,
    \begin{align*}
        \EE[A_k]\leq& \frac{1}{20\kappa_\xi L}\EE[\|\ntz \cL^D_-(x_k,\xi_k)-\vk\|^2]-\frac{1}{50\kappa_\xi L}(\epsilon\|\vk\|-2\epsilon^2)+ \frac{\epsilon}{20\kappa_\xi L}\EE[\|\ntz \cL^D_-(x_k,\xi_k)-\vk\|]\\
        \leq& \frac{1}{20\kappa_\xi L}\EE[\|\ntz \cL^D_-(x_k,\xi_k)-\vk\|^2]-\frac{1}{50\kappa_\xi L}(\epsilon\|\vk\|-2\epsilon^2)+ \frac{\epsilon}{20\kappa_\xi L}\sqrt{\EE[\|\ntz \cL^D_-(x_k,\xi_k)-\vk\|^2]}\\
        \leq& \frac{1}{20\kappa_\xi L}\cdot\frac{19}{1125}\kappa_\xi^{-2}\epsilon^2-\frac{1}{50\kappa_\xi L}(\epsilon\|\vk\|-2\epsilon^2)+ \frac{\epsilon}{20\kappa_\xi L}\cdot \sqrt{\frac{19}{1125}}\kappa_\xi^{-1}\epsilon\tag{Corollary 2 in \citep{luo2020stochastic}}\\
        \leq& \frac{\epsilon^2}{20\kappa_\xi L}-\frac{1}{50\kappa_\xi L}\epsilon\|\vk\|
    \end{align*}

    Therefore, combining with Lemma \ref{lem:bound_Bk} and taking average over $K$, we have
    \begin{align*}
        \frac{1}{K}\sum_{k=0}^{K-1}\EE[\cL^D_-(x_{k+1}, \xi_{k+1})-\cL^D_-(x_{k}, \xi_{k})]\leq &\frac{1}{K}\sum_{k=0}^{K-1}(\frac{\epsilon^2}{20\kappa_\xi L}-\frac{1}{50\kappa_\xi L}\epsilon\|\vk\|+\frac{134\epsilon^2}{\kappa_\xi L})\\
    \end{align*}
    Consequently, we have:
    \begin{align*}
        \frac{\epsilon}{50\kappa_\xi L}\frac{1}{K}\sum_{k=0}^{K-1}\|\vk\|        \leq& \frac{135\epsilon^2}{\kappa_\xi L} + \frac{\Delta \cL^D_-}{K}
    \end{align*}
    which means
    \begin{align*}
        \frac{1}{K}\sum_{k=0}^{K-1}\|\vk\|\leq 6750\epsilon + \frac{50\kappa_\xi L \Delta \cL^D_-}{K\epsilon }
    \end{align*}

    Under Assumptions \ref{assump:smooth} and \ref{assump:feature_matrices} and Condition that that both $Z$ and $\Xi$ have finite diameter, $\Delta \cL^D_-$ is a finite constant. By choosing $K=\lceil\frac{50\kappa_\xi L \Delta \cL^D_-}{\epsilon^2}\rceil$, we have:
    \begin{align*}
        \EE\|\ntz \max_{\xi\in\Xi}\cL^D_-(\hat\theta, \hat\zeta, \xi)\|=\frac{1}{K}\sum_{k=0}^{K-1}\EE\|\ntz \max_{\xi\in\Xi}\cL^D_-(\theta_k, \zeta_k, \xi)\|\leq \frac{1}{K}\sum_{k=0}^{K-1}(\EE\|\vk\| + \frac{15\epsilon}{7})\leq 6754\epsilon
    \end{align*}
    which finishes the proof.
\end{proof}

\section{Concrete Examples for Saddle-Point Solver Oracle}\label{appx:oracle_examples}

\subsection{Connection with Previous Methods}\label{appx:inner_oracle_related_work}
The inner optimization oracle in our second strategy essentially solves the off-policy policy evaluation problem, i.e., given a policy, compute its value function and  marginalized importance weighting function. 
Among the plethora of works studying off-policy policy evaluation with linear function approximation,  
\citep{liu2020finitesample} connected the GTD family and stochastic gradient optimization, and established finite-sample analysis for their off-policy algorithms. Their convergence rate is worse than ours, because their objective is only convex-strongly-concave, whereas our objective is strongly-convex-strongly-concave thanks to the regularization on the parameters of the $Q$ function.  

Besides, \citep{pmlr-v70-du17a} adapted variance-reduced stochastic gradient optimization algorithms for policy evaluation, which can be extended to off-policy setting. However, they focused on the finite-sum case and their algorithms will be inefficient for our purpose when the dataset is prohibitively large. Besides, they did not analyze the bias resulting from regularization and approximation error, which we do in this paper. 

\subsection{The First Example for Saddle-Point Solver Oracle}\label{appx:PLSO}

In the proof of Property \ref{prop:radius}, we show that in our linear setting, it's possible to derive a close form solution for $\zeta^*, \xi^*=\arg\max_{\zeta\in Z}\min_{\xi \in \Xi}\cL^D(\theta, \zeta, \xi)$ for arbitrary $\theta\in \Theta$:
\begin{align*}
    \zeta^*_\pi =& \Big(\lambda_w\lambda_Q  \K_w +  \M_\pi \K_Q^{-1} \M_\pi\trans\Big)^{-1}\Big(-(1-\gamma)\M_\pi \K_Q^{-1} \bPhi_Q\trans  \nu_D^\pi+\lambda_Q\bPhi_w\trans \bLambda^D R\Big),\\
    \xi^*_\pi =&  \Big(\lambda_w\lambda_Q  \K_Q +  \M_\pi\trans \K_w^{-1}\M_\pi \Big)^{-1}\Big((1-\gamma)\lambda_w \bPhi_Q\trans  \nu_D^\pi +\M\trans_\pi \K^{-1}_w \bPhi_w\trans \bLambda^D R\Big).
\end{align*}
Therefore, we can directly use our sample data to estimate $\K$, $\M$, $\bPhi\trans \bLambda^D R$ and $\bPhi\trans \nu_D^\pi$ and then estimate $(\zeta^*,\xi^*)$.
In this section, we provide an algorithm based on this idea in Alg. \ref{alg:PLSO}, and prove that such algorithm satisfies the requirement of the \textsc{Oracle} in Definition \ref{def:oracle_alg}.

\subsubsection{Estimation Error of Least-Square Oracle}
We first show a useful Lemma:
\begin{lemma}[Matrix Bernstein Theorem (Theorem 6.1.1 in \citep{tropp2015introduction})]
Consider a finite sequence $\{\BoldS_k\}$ of independent, random matrices with common dimension $d_1 \times d_2$. Assume that
\begin{align*}
    \EE \BoldS_k = 0~~and~~\|\BoldS_k\| \leq L~~for~each~index~k
\end{align*}
Introduce the random matrix
$$
\Z = \sum_k \BoldS_k
$$
Let $v(\Z)$ be the matrix variance statistic of the sum:
\begin{align*}
    v(\Z) =& \max\{\|\EE(\Z\Z^*)\|, \|\EE(\Z^*\Z)\|\}\\
    =&\max\{\|\sum_k \EE(\BoldS_k\BoldS_k^*)\|, \|\sum_k \EE(\BoldS_k^*\BoldS_k)\|\}
\end{align*}
Then,
\begin{align*}
    \EE\|\Z\|\leq \sqrt{2v(\Z)\log(d_1 + d_2)} + \frac{1}{3}L\log(d_1+d_2)
\end{align*}
Furthermore, for all $t\geq 0$
\begin{align*}
    \mathbb{P}\{\|\Z\|\geq t\}\leq (d_1 + d_2)\exp\Big(\frac{-t^2/2}{v(\Z)+Lt/3}\Big)
\end{align*}
\end{lemma}

For Algorithm \ref{alg:PLSO}, we have the following guarantee:
\begin{proposition}
    Suppose we have $N-(s,a,r,s',a',a_0)$ tuples independently sampling according to $d^D\times\pi$, and $N\geq \max\{2\sigma^2_\K+\frac{4\sigma_{\min}}{3}, 2\sigma^2_\M+\frac{8(1+\gamma)\sigma_{\min}}{3}\}\frac{4}{\sigma_{\min}^2} \log\frac{2\dims}{\delta}$.
    For arbitrary $1/5>\delta>0$, with probability at least $1-5\delta$, $\hK_w$, $\hK_Q$, $\hM$ are invertible and their smallest eigenvalues (for $\hM$ we consider the smallest sigular values) are larger than $\eig_w/2, \eig_Q/2$ and $\eig_M/2$, respectively, while the following conditions hold at the same time:
    \begin{align*}
        \|\K_w-\hK_w\|\leq& \frac{2}{3N_{all}}\log \frac{2\dims}{\delta}+\sqrt{\frac{2\sigma^2_\K}{N_{all}}\log \frac{2\dims}{\delta}}=O(\frac{\sigma_\K}{\sqrt{N_{all}}})\\
        \|\K_Q-\hK_Q\|\leq& \frac{2}{3N_{all}}\log \frac{2\dims}{\delta}+\sqrt{\frac{2\sigma^2_\K}{N_{all}}\log \frac{2\dims}{\delta}}=O(\frac{\sigma_\K}{\sqrt{N_{all}}})\\
        \|\M_\pi-\hM_\pi\|=\|\M\trans_\pi-\hM_\pi\trans\| \leq& \frac{4(1+\gamma)}{3N_{all}}\log \frac{2\dims}{\delta}+\sqrt{\frac{2\sigma^2_\M}{N_{all}}\log \frac{2\dims}{\delta}}=O(\frac{\sigma_\M}{\sqrt{N_{all}}})\\
        \|\hunu^\pi-\unu^\pi\| \leq& \frac{4}{3N_{all}}\log \frac{2\dims}{\delta}+\sqrt{\frac{2\sigma^2_\nu}{N_{all}}\log \frac{2\dims}{\delta}}=O(\frac{\sigma_\nu}{\sqrt{N_{all}}})\\       
        \|\huR-\uR\| \leq& \frac{4}{3N_{all}}\log \frac{2\dims}{\delta}+\sqrt{\frac{2\sigma^2_R}{N_{all}}\log \frac{2\dims}{\delta}}=O(\frac{\sigma_R}{\sqrt{N_{all}}})
    \end{align*}
\end{proposition}
\begin{proof}
First, we try to bound $\|\K_w-\hK_w\|$. We take a look at a series of random matrices 
\begin{align*}
    A_i = \K_w - \bphi_w(s_i, a_i)\bphi_w(s_i,a_i)\trans,~~~~i=1,2,...
\end{align*}
where $(s_i, a_i)$ are sampled from $d^D$ independently. Easy to verify that $A_i$ has the following properties as a result of $\|\bphi_w(s,a)\|\leq 1$:
\begin{align*}
    \EE[A_i] = 0,~~~~\|A_i\| =\|\K_w-\bphi_w(s,a)\trans\bphi_w(s_i,a_i)\trans\|\leq 1
\end{align*}


Under Assumption \ref{assump:detailed_variance}, according to the Matrix Bernstein Theorem and the union bound, we have:
\begin{align*}
    P(\frac{\|\sum_{i=1}^N A_i \|}{N_{all}}\geq \epsilon)\leq& P(\sigma_{\max}(\sum_{i=1}^N A_i)\geq N_{all}\epsilon)+P(\sigma_{\max}(\sum_{i=1}^N -A_i)\geq N_{all}\epsilon)\\
    \leq& 2\dims\exp(-\frac{N_{all}\epsilon^2/2}{\sigma_\K^2+\epsilon/3})
\end{align*}
which implies that, with probability $1-\delta$:
\begin{align*}
    \|\K_w-\hK_w\|\leq \frac{2}{3N_{all}}\log \frac{2\dims}{\delta}+\sqrt{\frac{2\sigma^2_\K}{N_{all}}\log \frac{2\dims}{\delta}}=O(\frac{\sigma_\K}{\sqrt{N_{all}}})
\end{align*}
The discussion for $\hK_Q$ is similar and we omit here. As for $\|\M_\pi-\hM_\pi\|$, notice that,
\begin{align*}
    \|\M_\pi-(\bphi_w(s,a)\bphi_Q(s,a)\trans -\gamma \bphi_w(s,a)\bphi_Q(s',a'))\|\leq 2(1+\gamma)\\
    \|\M_\pi\trans-(\bphi_Q(s,a)\bphi_w(s,a)\trans -\gamma \bphi_Q(s',a')\bphi_w(s,a))\|\leq 2(1+\gamma)
\end{align*}
Therefore, w.p. $1-\delta$,
\begin{align*}
   \|\M_\pi-\hM_\pi\|=\|\M_\pi\trans-\hM_\pi\trans\| \leq \frac{4(1+\gamma)}{3N_{all}}\log \frac{2\dims}{\delta}+\sqrt{\frac{2\sigma^2_\M}{N_{all}}\log \frac{2\dims}{\delta}}=O(\frac{\sigma_\M}{\sqrt{N_{all}}})
\end{align*}

As for the bounds for the difference of vectors $\|\hunu^\pi-\unu^\pi\|$ and $\|\huR-\uR\|$, since vectors are special cases of matrices (or we can concatenate the vector with a zero-vector to make a $n\times 2$ matrix and make the proof more rigorous), again, we can apply the same technique again. Notice that,
\begin{align*}
    \|\uR-\bphi_w(s,a)r(s,a)\|\leq 2\\
    \|\unu^\pi-\bphi_Q(s_0,a_0)\|\leq 2
\end{align*}
Besides, for random column vectors $x$ with bounded norm, we should have
\begin{align*}
    \|\EE_{x}[xx\trans]\| \leq tr(\EE_{x}[x x\trans]) = \EE_{x}[tr(x x\trans)] = \EE_{x}[x\trans x]=\|\EE_{x}[x\trans x]\|
\end{align*}
As a result, combining Assumption \ref{assump:detailed_variance}, we have
\begin{align*}
    \|\hunu^\pi-\unu^\pi\| \leq \frac{4}{3N_{all}}\log \frac{2\dims}{\delta}+\sqrt{\frac{2\sigma^2_\nu}{N_{all}}\log \frac{2\dims}{\delta}}=O(\frac{\sigma_\nu}{\sqrt{N_{all}}}),~~~w.p.~~1-\delta\\       
    \|\huR-\uR\| \leq \frac{4}{3N_{all}}\log \frac{2\dims}{\delta}+\sqrt{\frac{2\sigma^2_R}{N_{all}}\log \frac{2\dims}{\delta}}=O(\frac{\sigma_R}{\sqrt{N_{all}}}),~~~w.p.~~1-\delta
\end{align*}
Therefore, w.p. $1-2\delta$ the concentration results in this Proposition hold. Next, we derive the smallest eigenvalues of $\hK_w$, $\hK_Q$ and $\hM_\pi$ when $n$ is large enough. For arbitrary $x\in\mathbb{R}^{\dims\times 1}$ with $\|x\|=1$, we have:
\begin{align*}
    x\trans \hK_w x =& x\trans \K_w x + x\trans (\hK_w-\K_w) x\geq \eig_w -\|\K_w-\hat{\K}_w\|\\
    x\trans \hK_Q x =& x\trans \K_Q x + x\trans (\hK_Q-\K_Q) x\geq \eig_Q -\|\K_Q-\hat{\K}_Q\|\\
    |x\trans \hM_\pi x| \geq & |x\trans \M x| - |x\trans (\hM-\M) x|\geq \eig_M -\|\M-\hat{\M}\|\\
\end{align*}
Therefore, easy to verify that, when the concentration results hold,and $N\geq \max\{8\frac{\sigma^2_\K}{\eig_w^2}+\frac{16}{3\eig_w},8\frac{\sigma^2_\K}{\eig_Q^2}+\frac{16}{3\eig_Q}, 8\frac{\sigma^2_\K}{\eig_M^2}+\frac{32(1+\gamma)}{3\eig_M}\}\log\frac{2\dims}{\delta}$, 
$\hK_w$, $\hK_Q$, $\hM$ are invertible and their smallest eigenvalues (for $\hM$ we consider the smallest sigular values) are larger than $\eig_w/2, \eig_Q/2$ and $\eig_M/2$
\end{proof}

Next, we are ready to show that Algorithm \ref{alg:PLSO} is also an example of the Oracle Algorithm.
\begin{theorem}
    Algorithm \ref{alg:PLSO} satisfies the Oracle Condition \ref{def:oracle_alg} with $\beta=0$ and arbitrary $c>0$, with a proper choice of $N_{all}$.
\end{theorem}
\begin{proof}
Denote $\zeta^*, \xi^*$ as the saddle-point of $\cL^D$ given $\theta$, and use $\hat{\zeta^*}, \hat{\xi^*}$ to denote the $\zeta$ and $\xi$ in Algorithm \ref{alg:PLSO} before the projection. Given the proposition above, we are ready to bound $\|\zeta^* - \hat{\zeta^*}\|$ and $\|\xi^* - \hat{\xi^*}\|$. We list two properties below which we will use frequently later. Firstly, for arbitrary invertible $d\times d$ matrices $\A$ and $\B$, we have:
\begin{align*}
    \|\A^{-1}-\B^{-1}\|=\|\A^{-1}(\B-\A)\B^{-1}\|\leq \|\A^{-1}\|\|\B-\A\|\|\B^{-1}\|
\end{align*}
Secondly, for arbitrary $d\times d$ matrices $\X_1, \Y_1, \X_2, \Y_2$, we have:
\begin{align*}
    \|\X_1\Y_1-\X_2\Y_2\|= \|\X_1(\Y_1-\Y_2) + \Y_2(\X_1-\X_2)\|\leq \|\X_1\|\|\Y_1-\Y_2\|+\|\Y_2\|\|\X_1-\X_2\|
\end{align*}

Then, we have

\begin{align*}
    &\|\zeta^* - \hat{\zeta^*}\|\\
    \leq&(1-\gamma)\|\Big(\Big(\lambda_w\lambda_Q  \K_w +  \M_\pi \K_Q^{-1} \M_\pi\trans\Big)^{-1}\M_\pi\K_Q^{-1}-\Big(\lambda_w\lambda_Q \hK_w +  \hM_\pi \hK_Q^{-1} \hM_\pi\trans\Big)^{-1}\hM_\pi\hK_Q^{-1}\Big) \unu^\pi\|\\
        &+(1-\gamma)\|\Big(\lambda_w\lambda_Q \hK_w +  \hM_\pi \hK_Q^{-1} \hM_\pi\trans\Big)^{-1}\hM_\pi\hK_Q^{-1} (\unu^\pi - \hunu^\pi)\|\\
        &+\|\Big(\lambda_w\lambda_Q \hK_w +  \hM_\pi \hK_Q^{-1} \hM_\pi\trans\Big)^{-1}\lambda_Q(\uR-\huR)\|\\
        &+\|\Big(\Big(\lambda_w\lambda_Q  \K_w +  \M_\pi \K_Q^{-1} \M_\pi\trans\Big)^{-1}-\Big(\lambda_w\lambda_Q \hK_w +  \hM_\pi \hK_Q^{-1} \hM_\pi\trans\Big)^{-1}\Big)\lambda_Q \uR\|\\
    \leq&(1-\gamma)\|\Big(\lambda_w\lambda_Q  \K_w +  \M_\pi \K_Q^{-1} \M_\pi\Big)^{-1}\|\|\lambda_w\lambda_Q  \K_w +  \M_\pi \K_Q^{-1} \M_\pi-\lambda_w\lambda_Q \hK_w -  \hM_\pi \hK_Q^{-1} \hM_\pi\trans\|\\
    &\cdot\|\Big(\lambda_w\lambda_Q \hK_w +  \hM_\pi \hK_Q^{-1} \hM_\pi\trans\Big)^{-1}\|\|\M_\pi\K_Q^{-1}\|\\
        &+(1-\gamma)\|\Big(\lambda_w\lambda_Q \hK_w +  \hM_\pi \hK_Q^{-1} \hM_\pi\trans\Big)^{-1}\hM_\pi\hK_Q^{-1}\|\|\M_\pi\K_Q^{-1}-\hM_\pi\hK_Q^{-1}\|\\
        &+(1-\gamma)\frac{\sigma_\nu}{\sqrt N}\|\Big(\lambda_w\lambda_Q \hK_w +  \hM_\pi \hK_Q^{-1} \hM_\pi\trans\Big)^{-1}\hM_\pi\hK_Q^{-1}\|+\lambda_Q\frac{\sigma_R}{\sqrt{N_{all}}} \|\Big(\lambda_w\lambda_Q \hK_w + \hM_\pi\hK^{-1}_Q\hM_\pi\trans\Big)^{-1}\|\\
        &+\lambda_Q\|\Big(\lambda_w\lambda_Q  \K_w +  \M_\pi \K_Q^{-1} \M_\pi\Big)^{-1}\|\|\lambda_w\lambda_Q  \K_w +  \M_\pi \K_Q^{-1} \M_\pi-\lambda_w\lambda_Q \hK_w -  \hM_\pi \hK_Q^{-1} \hM_\pi\trans\|\\
        &\cdot\|\Big(\lambda_w\lambda_Q \hK_w +  \hM_\pi \hK_Q^{-1} \hM_\pi\trans\Big)^{-1}\|\\
    =&O\Big(\frac{1}{\eig_Q (\lambda_w\lambda_Q\eig_w+\eig_\M^2)^2\sqrt{N_{all}}}(\lambda_w\lambda_Q \sigma_\K + \frac{\sigma_\M}{\eig_Q}+\frac{\sigma_\K}{\eig^2_Q})+\frac{\sigma_\nu}{\eig_Q(\lambda_w\lambda_Q \eig_w + \eig_\M^2)\sqrt{N_{all}}}+\frac{\lambda_Q\sigma_R}{(\lambda_w\lambda_Q \eig_w + \eig_\M^2)\sqrt{N_{all}}}\\
    &+\frac{\lambda_Q}{(\lambda_w\lambda_Q\eig_w+\eig_\M^2)^2\sqrt{N_{all}}}(\lambda_w\lambda_Q \sigma_\K + \frac{\sigma_\M}{\eig_Q}+\frac{\sigma_\K}{\eig^2_Q})\Big)\log\frac{2\dims}{\delta}\\
    =&O\Big(\frac{1+\lambda_Q\eig_Q}{\eig_Q (\lambda_w\lambda_Q\eig_w+\eig_\M^2)^2}(\lambda_w\lambda_Q \sigma_\K + \frac{\sigma_\M}{\eig_Q}+\frac{\sigma_\K}{\eig^2_Q})+\frac{\sigma_\nu}{\eig_Q(\lambda_w\lambda_Q \eig_w + \eig_\M^2)}+\frac{\lambda_Q\sigma_R}{(\lambda_w\lambda_Q \eig_w + \eig_\M^2)}\Big)\sqrt{N_{all}}\log\frac{2\dims}{\delta}
\end{align*}
where we omit constant number in $O(\cdot)$.

Similarly, we have
\begin{align*}
    &\|\xi^* - \hat\xi^*\|\\
    \leq& \|(1-\gamma)\lambda_w \Big(\Big(\lambda_w\lambda_Q  \K_Q +  \M_\pi\trans \K_w^{-1}\M_\pi \Big)^{-1}-\Big(\lambda_w\lambda_Q  \hK_Q +  \hM_\pi\trans \hK_w^{-1}\hM_\pi\Big)^{-1} \Big)\unu^\pi\|\\
        &+\|(1-\gamma)\lambda_w \Big(\lambda_w\lambda_Q  \hK_Q +  \hM_\pi\trans \hK_w^{-1}\hM_\pi \Big)^{-1} (\unu^\pi-\hunu^\pi)\|\\
        &+\|\Big(\Big(\lambda_w\lambda_Q  \K_Q +  \M_\pi\trans \K_w^{-1}\M_\pi \Big)^{-1}\M\trans_\pi \K^{-1}_w-\Big(\lambda_w\lambda_Q  \hK_Q +  \hM_\pi\trans \hK_w^{-1}\hM_\pi \Big)^{-1}\hM\trans_\pi \hK^{-1}_w\Big)\uR\|\\
        &+\|\Big(\lambda_w\lambda_Q  \hK_Q +  \hM_\pi\trans \hK_w^{-1}\hM_\pi \Big)^{-1}\hM\trans_\pi \hK^{-1}_w (\uR-\huR)\|\\
    =&O\Big(\frac{\lambda_w\eig_w+1}{\eig_w(\lambda_w\lambda_Q\eig_Q+\eig_\M^2)^2}(\lambda_w\lambda_Q \sigma_\K + \frac{\sigma_\M}{\eig_w}+\frac{\sigma_\K}{\eig_w^2})+\frac{\sigma_R}{\eig_w(\lambda_w\lambda_Q \eig_Q + \eig_\M^2)}+\frac{\lambda_w\sigma_\nu}{(\lambda_w\lambda_Q \eig_Q + \eig_\M^2)}\Big)\sqrt{N_{all}}\log\frac{2\dims}{\delta}\\
\end{align*}

As a result, for arbitrary $1/5>\delta>0$, w.p. at least $1-5\delta$,
\begin{align*}
    &\sqrt{\|\zeta^*-\hat\zeta^*\|^2+\|\xi^*-\hat\xi^*\|^2}\leq \|\zeta^*-\hat\zeta^*\|+\|\xi^*-\hat\xi^*\|\\
    =&O\Big(\frac{1+\lambda_Q\eig_Q}{\eig_Q (\lambda_w\lambda_Q\eig_w+\eig_\M^2)^2\sqrt{N_{all}}}(\lambda_w\lambda_Q \sigma_\K + \frac{\sigma_\M}{\eig_Q}+\frac{\sigma_\K}{\eig^2_Q})+\frac{\sigma_\nu}{\eig_Q(\lambda_w\lambda_Q \eig_w + \eig_\M^2)\sqrt{N_{all}}}+\frac{\lambda_Q\sigma_R}{(\lambda_w\lambda_Q \eig_w + \eig_\M^2)\sqrt{N_{all}}}\\
    &+\frac{\lambda_w\eig_w+1}{\eig_w(\lambda_w\lambda_Q\eig_Q+\eig_\M^2)^2}(\lambda_w\lambda_Q \sigma_\K + \frac{\sigma_\M}{\eig_w}+\frac{\sigma_\K}{\eig_w^2})+\frac{\sigma_R}{\eig_w(\lambda_w\lambda_Q \eig_Q + \eig_\M^2)}+\frac{\lambda_w\sigma_\nu}{(\lambda_w\lambda_Q \eig_Q + \eig_\M^2)}\Big)\sqrt{N_{all}}\log\frac{2\dims}{\delta}\\
    \numberthis\label{eq:quality_dm}
\end{align*}

For simplicity, we use $C_{LS}$ to as a shorthand of
\begin{align*}
    O\Big(&\frac{1+\lambda_Q\eig_Q}{\eig_Q (\lambda_w\lambda_Q\eig_w+\eig_\M^2)^2}(\lambda_w\lambda_Q \sigma_\K + \frac{\sigma_\M}{\eig_Q}+\frac{\sigma_\K}{\eig^2_Q})+\frac{\sigma_\nu}{\eig_Q(\lambda_w\lambda_Q \eig_w + \eig_\M^2)}+\frac{\lambda_Q\sigma_R}{(\lambda_w\lambda_Q \eig_w + \eig_\M^2)}\\
    &+\frac{\lambda_w\eig_w+1}{\eig_w(\lambda_w\lambda_Q\eig_Q+\eig_\M^2)^2}(\lambda_w\lambda_Q \sigma_\K + \frac{\sigma_\M}{\eig_w}+\frac{\sigma_\K}{\eig_w^2})+\frac{\sigma_R}{\eig_w(\lambda_w\lambda_Q \eig_Q + \eig_\M^2)}+\frac{\lambda_w\sigma_\nu}{(\lambda_w\lambda_Q \eig_Q + \eig_\M^2)}\Big)
\end{align*}

Next, we try to convert the above high probability bound to an upper bound for expectation. We use $X$ to denote $\|\zeta^*-\hat\zeta^*\|^2+\|\xi^*-\hat\xi^*\|^2$ and treat it as a r.v.. Then, we have:
\begin{align*}
    P(X\leq \frac{C_{LS}^2}{N_{all}}\log^2\frac{2\dims}{\delta}) \geq 1-5\delta
\end{align*}
or equivalently,
\begin{align*}
    P(X\leq x) \geq 1 - 10\dims \exp(-\frac{\sqrt{x N_{all}}}{C_{LS}})
\end{align*}
Therefore, computing the expectation of the distribution described by the following C.D.F. can provide us an upper bound for $\EE[X]$: 
$$F_X(x)=P(X\leq x)=1-10\dims \exp(-\frac{\sqrt{x N_{all}}}{C_{LS}})$$
As a result,
\begin{align*}
    \EE[X]\leq& \int_{x=0}^\infty (1-F_X(x)){\rm d}x=\int_{x=0}^\infty 10\dims \exp(-\frac{\sqrt{x N_{all}}}{C_{LS}})  {\rm d}x\\
    =&10\dims \int_{x=0}^\infty\exp(-x){\rm d} \frac{C_{LS}^2}{N_{all}}x^2=20\dims \frac{C_{LS}^2}{N_{all}}
\end{align*}

which means that,
\begin{align}\label{eq:exp_quality_dm}
    \EE[\|\zeta^*-\hat\zeta^*\|^2+\|\xi^*-\hat\xi^*\|^2]=20d \frac{C_{LS}^2}{N_{all}}
\end{align}
Because $\zeta\in Z$ and $\xi\in\Xi$, the projection can only shrink the distance:
\begin{align*}
    \EE[\|\zeta^*-P_Z(\hat\zeta^*)\|^2+\|\xi^*-P_\Xi(\hat\xi^*)\|^2]\leq \EE[\|\zeta^*-\hat\zeta^*\|^2+\|\xi^*-\hat\xi^*\|^2]
\end{align*}

As a result, for arbitrary $c$, we can choose $N_{all}=20d\frac{C_{LS}^2}{\rc}$ in Algorithm \ref{alg:PLSO}, and Algorithm \ref{alg:PLSO} satisfies the oracle condition with $\beta=0$:
\begin{equation}
    \EE[\|P_Z(\zeta)-\zeta^*\|^2+\|P_\Xi(\xi)-\xi^*\|^2]\leq \rc
\end{equation}


\end{proof}
\subsection{The Second Example for Saddle-Point Solver Oracle}\label{appx:SVRE_Oracle}
In this section, we provide another example for the oracle in Definition \ref{def:oracle_alg} based on first-order optimization, which is inspired by SVRE\citep{chavdarova2019reducing}.

\subsubsection{Proofs for Algorithm 4}\label{appx:proofs_alg4}

For simplification, we will use $\omega=[\zeta,\xi]\in Z\times\Xi:=\Omega$ to denote the vector concatenated by $\zeta$ and $\xi$. 
Similarly, $g_t=[-g^\zeta_t, g^\xi_t]$, and $F_{\batch}(\omega)=\EE_{(s,a,r,s',a_0,a')\sim \batch}\{[-\nz\cL^{(s,a,r,s',a_0,a')}(\theta, \zeta, \xi), \nx\cL^{(s,a,r,s',a_0,a')}(\theta, \zeta, \xi)]\}$, where $\batch$ is the mini batch data sampled according to $d^D$, and $\ncL^{(s,a,r,s',a_0,a')}(\theta, \zeta, \xi)$ is the gradient computed with one sample $(s,a,r,s',a_0,a')$. We use $F(\omega):=\EE_{\batch\sim D}[F_{\batch}(\omega)]$ to denote the gradient expected over entire dataset distribution. Besides,
\begin{align*}
    \eta g_t =& [-\ez g^\zeta_t, \ex g^\xi_t];\quad\quad\quad\eta^2 \|\omega\|^2=\et^2\|\zeta\|^2+\ex^2\|\xi\|^2;\quad\quad\mu\|\omega\|^2=\mu_\zeta\|\zeta\|^2+\mu_\xi\|\xi\|^2\\
    \bar{L}^2\|w\|^2 =& \bar{L}^2_\zeta \|\zeta\|^2 + \bar{L}^2_\xi \|\xi\|^2;\quad\quad \eta^2 \bar{L}^2=\ez^2 \bar{L}_\zeta^2 + \ex^2 \bar{L}_\xi^2;\quad\quad\quad\quad\quad\eta\mu = \ez \mu_\zeta + \ex \mu_\xi
\end{align*}
The update rule for Algorithm \ref{alg:SVREB} can be summarized as
\begin{align*}
    {\rm Extrapolation}: &~~~\omega_{t+1/2}=P_{\Omega}(\omega_t - \eta g_{t})\\
    {\rm Update}: &~~~\omega_{t+1}=P_{\Omega}(\omega_t - \eta g_{t+1/2})\\
\end{align*}
Besides, in this section, the expectation $\EE$ concerns all the randomness starting from the beginning of the algorithm.

\begin{lemma}[Lemma 1 in \citep{chavdarova2019reducing}]\label{lem:reducing_lem1}
Let $\omega\in\Omega$ and $\omega^+:=P_{\Omega}(w+u)$, then for all $w'\in\Omega$, we have
\begin{align*}
    \|\omega^+-\omega'\|^2\leq \|\omega-\omega'\|^2+2u\trans (\omega^+-\omega')-\|\omega^+-\omega\|^2
\end{align*}
\end{lemma}

\begin{lemma}[Adapted from Lemma 3 in \citep{chavdarova2019reducing}]\label{lem:lem3_reducing}
For any $w\in\Omega$, when $t> 0$, we have
\begin{align*}
    \|\omega_{t+1}-\omega\|^2\leq     \|\omega_t-\omega\|^2-2\eta g_{t+1/2}\trans(\omega_{t+1/2}-\omega)+\eta^2 \|g_t-g_{t+1/2}\|-\|\omega_{t+1/2}-\omega_t\|^2
\end{align*}
and when $t=0$, we have
\begin{align*}
    \|\omega_1-\omega\|^2\leq \|\omega_0-\omega\|^2-2\eta g_0\trans(\omega_1-\omega)
\end{align*}
\end{lemma}
\begin{proof}
For $t=0$, by simply applying Lemma \ref{lem:reducing_lem1} for $(\omega, u, \omega^+, \omega')=(\omega_0, -\eta g_0\trans, \omega_1, \omega)$, we have:
\begin{align*}
    \|\omega_1-\omega\|^2    \leq&\|\omega_0-\omega\|^2-2\eta g_0\trans(\omega_1-\omega)-\|\omega_1-\omega_0\|^2\leq \|\omega_0-\omega\|^2-2\eta g_0\trans(\omega_1-\omega)
\end{align*}
For $t>0$, the proof is exactly the same as Lemma 3 in \citep{chavdarova2019reducing}

\end{proof}

\begin{lemma}[Bound $\|g_t-g_{t+1/2}\|^2$]\label{lem:bound_g}
For $t > 0$, we have:
\begin{align*}
    \EE[\|g_t-g_{t+1/2}\|^2] \leq& 10\EE[\|F_{\batch}(w_{t})-F_{\batch}(\omega^*)\|^2]+10\EE[\|F_{\batch}(\omega^*)-F_{\batch}(w_{t-1})\|^2] + 5\bar{L}^2\EE[\|w_t-w_{t+1/2}\|]
\end{align*}
\end{lemma}
\begin{proof}
For $t>0$:
\begin{align*}
    &\EE[\|g_t-g_{t+1/2}\|^2]\\
    =&\EE[\|F_{\batch}(w_{t})-F_{\batch}(w_{t-1})+m_{t}-F_{\batch'}(w_{t+1/2})+F_{\batch'}(w_{t-1})-m_{t}\|^2]\\
    =&\EE[\|F_{\batch}(w_{t})\pm F_{\batch}(w^*)-F_{\batch}(w_{t-1})-F_{\batch'}(w_{t+1/2})\pm F_{\batch'}(w_t)\pm F_{\batch'}(w^*)+F_{\batch'}(w_{t-1})\|^2]\\
    \leq & 5\EE[\|F_{\batch}(w_{t})-F_{\batch}(\omega^*)\|^2]+5\EE[\|F_{\batch}(\omega^*)-F_{\batch}(w_{t-1})\|^2]\\
    &+5\EE[\|F_{\batch'}(w_{t+1/2})-F_{\batch'}(w_t)\|^2]]+5\EE[\|F_{\batch'}(w_{t})-F_{\batch'}(\omega^*)\|^2]+5\EE[\|F_{\batch'}(\omega^*)-F_{\batch'}(w_{t-1})\|^2]\\
    =&10\EE[\|F_{\batch}(w_{t})-F_{\batch}(\omega^*)\|^2]+10\EE[\|F_{\batch}(\omega^*)-F_{\batch}(w_{t-1})\|^2]+5\EE[\|F_{\batch'}(w_{t+1/2})-F_{\batch'}(w_t)\|^2]]
\end{align*}
where in the inequality we use the extended Young's inequality; in the last equation we use the fact that
\begin{align*}
    \EE_{\batch \sim D}[\|F_{\batch}(w_{t})-F_{\batch}(w)\|^2]=\EE_{\batch' \sim D}[\|F_{\batch'}(w_{t})-F_{\batch'}(w)\|^2],~~~~~\forall w\in \Omega
\end{align*}
Besides, according to Assumption \ref{cond:smooth_SVRE}
\begin{align*}
    \EE[\|F_{\batch'}(w_{t+1/2})-F_{\batch'}(w_t)\|^2]]
    \leq \bar{L}^2\EE[\|w_t-w_{t+1/2}\|]
\end{align*}
As a result,
\begin{align*}
    \EE[\|g_t-g_{t+1/2}\|^2] \leq& 10\EE[\|F_{\batch}(w_{t})-F_{\batch}(\omega^*)\|^2]+10\EE[\|F_{\batch}(\omega^*)-F_{\batch}(w_{t-1})\|^2] + 5\bar{L}^2\EE[\|w_t-w_{t+1/2}\|]
\end{align*}
\end{proof}

\begin{proposition}\label{prop:strong_mono}
With Property \ref{prop:detailed_scscs}, for arbitrary $\theta$, the operator $F(\omega)$ satisfying:
\begin{align*}
    \Big(F(\omega_1)-F(\omega_2)\Big)\trans \Big(\omega_1 - \omega_2\Big)\geq \mu\|\omega_1-\omega_2\|^2 
\end{align*}
\end{proposition}
\begin{proof}
Based on Property \ref{prop:detailed_scscs}, we have:
\begin{align*}
    -\cL^D(\theta, \zeta_1,\xi_2)\geq& -\cL^D(\theta, \zeta_2,\xi_2)-\nz \cL^D(\theta, \zeta_1,\xi_1)\trans (\zeta_2 - \zeta_1)+\frac{\mu_\zeta}{2}\|\zeta_2-\zeta_1\|^2\\
    -\cL^D(\theta, \zeta_2,\xi_1)\geq& -\cL^D(\theta, \zeta_2,\xi_2)-\nz \cL^D(\theta, \zeta_2,\xi_2)\trans (\zeta_1 - \zeta_2)+\frac{\mu_\zeta}{2}\|\zeta_2-\zeta_1\|^2\\
    \cL^D(\theta, \zeta_1,\xi_2)\geq& \cL^D(\theta, \zeta_1,\xi_1)+\nx \cL^D(\theta, \zeta_1,\xi_1)\trans (\xi_2 - \xi_1)+\frac{\mu_\xi}{2}\|\xi_2-\xi_1\|^2\\
    \cL^D(\theta, \zeta_2,\xi_1)\geq& \cL^D(\theta, \zeta_2,\xi_2)+\nx \cL^D(\theta, \zeta_2,\xi_2)\trans (\xi_1 - \xi_2)+\frac{\mu_\xi}{2}\|\xi_2-\xi_1\|^2
\end{align*}
Sum up and we can obtain
\begin{align*}
    &\Big(F(\omega_1)-F(\omega_2)\Big)\trans \Big(\omega_1 - \omega_2\Big):=\Big(F(\zeta_1, \xi_1)-F(\zeta_2, \xi_2)\Big)\trans \Big([\zeta_1, \xi_1] - [\zeta_2, \xi_2]\Big) \\
    =& -\Big(\nz \cL^D(\theta, \zeta_1,\xi_1)-\nz \cL^D(\theta, \zeta_2,\xi_2)\Big)\trans(\zeta_1-\zeta_2)+\Big(\nx \cL^D(\theta, \zeta_1,\xi_1)-\nx \cL^D(\theta, \zeta_2,\xi_2)\Big)\trans(\xi_1-\xi_2)\\
    \geq& \mu_\zeta\|\zeta_2-\zeta_1\|^2 + \mu_\xi\|\xi_2-\xi_1\|^2:=\mu\|\omega_1-\omega_2\|^2 
\end{align*}
\end{proof}

\ConvergeRateAlgFour*
\begin{proof}
When $t>0$, from Lemma \ref{lem:lem3_reducing}, we have
\begin{align*}
    \|\omega_{t+1}-\omega^*\|^2\leq& \|\omega_t - \omega^*\|^2 -2\eta g_{t+1/2}\trans(\omega_{t+1/2}-\omega^*)-\|\omega_{t+1/2}-\omega_t\|^2+\eta^2 \|g_t-g_{t+1/2}\|^2
\end{align*}

Next, we use $P_{t+1}$ to denote $\EE[\|\omega_{t+1}-\omega^*\|^2]+\tau \EE[\|F_{\batch}(\omega^*)-F_{\batch}(w_t)\|^2]$, where $\tau$ will be determined later, then we have

\begin{align*}
    P_{t+1}=&\expect{\|\omega_{t+1}-\omega^*\|^2}+\tau \EE[\|F_{\batch}(\omega^*)-F_{\batch}(w_t)\|^2]\\
    \leq& \expect{\|\omega_t - \omega^*\|^2} -2\eta \EE[F(\omega_{t+1/2})\trans(\omega_{t+1/2}-\omega^*)] -\EE[\|\omega_{t+1/2}-\omega_t\|^2]\\
    &+\eta^2\EE[ \|g_t-g_{t+1/2}\|^2]+\tau \EE[\|F_{\batch}(\omega^*)-F_{\batch}(w_t)\|^2]\\
    &-2\eta \EE[(m_t-F(\omega_{t-1}))\trans(\omega_{t+1/2}-\omega^*)]\tag{$\EE[g\trans_{t+1/2}(\omega_{t+1/2}-\omega^*)]=\EE[(F(\omega_{t+1/2})-F(\omega_{t-1})+m_t)\trans(\omega_{t+1/2}-\omega^*)]$}\\
    \leq & \expect{\|\omega_t - \omega^*\|^2} -2\eta \EE[F(\omega_{t+1/2})\trans(\omega_{t+1/2}-\omega^*)] -(1-5\eta^2\bar{L}^2)\EE[\|\omega_{t+1/2}-\omega_t\|^2]\\
    & + (\tau+10\eta^2)\EE[\|F_{\batch}(w_{t})-F_{\batch}(\omega^*)\|^2]+10\eta^2\EE[\|F_{\batch}(\omega^*)-F_{\batch}(w_{t-1})\|^2] \\
    &+ {2 \eta \sqrt{\EE[\|m_t-F(\omega_{t-1})\|^2]\EE[\|\omega_{t+1/2}-\omega^*\|^2]}}\tag{Lemma \ref{lem:bound_g} and Cauthy Inequality: $\EE[a\trans b|c]\leq \sqrt{\EE[\|a\|^2|c]\EE[\|b\|^2 |c]}$}\\
    \leq & \expect{\|\omega_t - \omega^*\|^2} -2\eta \EE[F(\omega_{t+1/2})\trans(\omega_{t+1/2}-\omega^*)] -(1-5\eta^2\bar{L}^2)\EE[\|\omega_{t+1/2}-\omega_t\|^2]\\
    & + (\tau+10\eta^2)\EE[\|F_{\batch}(w_{t})-F_{\batch}(\omega^*)\|^2]+10\eta^2\EE[\|F_{\batch}(\omega^*)-F_{\batch}(w_{t-1})\|^2] \\
    &+ {\frac{8 \eta}{\mu} \EE[\|m_t-F(\omega_{t-1})\|^2]+\frac{\mu\eta}{8} \EE[\|\omega_{t+1/2}-\omega^*\|^2]}\tag{$2\sqrt{|a\trans b|}\leq \|a\|^2+\|b\|^2$}\\
    \leq & \expect{\|\omega_t - \omega^*\|^2} -2\eta \EE[F(\omega_{t+1/2})\trans(\omega_{t+1/2}-\omega^*)]-(1-25\eta^2\bar{L}^2-2\tau\bar{L}^2)\EE[\|\omega_{t+1/2}-\omega_t\|^2] \\
    &+ (2\tau +20\eta^2)\EE[\|F_{\batch}(w_{t+1/2})-F_{\batch}(\omega^*)\|^2]+10\eta^2\EE[\|F_{\batch}(\omega^*)-F_{\batch}(w_{t-1})\|^2]\\
    &+\frac{8\sigma^2}{|\batch|}(\frac{\ez}{\mz}+\frac{\ex}{\mx})+\frac{\mu\eta}{4}(\EE[\|\omega_{t+1/2}-\omega_t\|^2+\EE[\|\omega_{t}-\omega^*\|^2]) \tag{Condition \ref{cond:variance}; Young's Inequality; $\EE[\|F_{\batch}(\omega_{t+1/2})-F_{\batch}(\omega_t)\|^2]\leq \bar{L}^2\EE[\|\omega_{t+1/2}-\omega_t\|^2]$}\\
    \leq & \expect{\|\omega_t - \omega^*\|^2} -2\eta \EE[F(\omega_{t+1/2})\trans(\omega_{t+1/2}-\omega^*)]-(1-25\eta^2\bar{L}^2-2\tau\bar{L}^2)\EE[\|\omega_{t+1/2}-\omega_t\|^2] \\
    &+ (2\tau\bar{L}+20\eta^2\bar{L})\EE[(F_{\batch}(\omega^*)-F_{\batch}(w_{t+1/2}))\trans(\omega^*-w_{t+1/2})] +10\eta^2\EE[\|F_{\batch}(\omega^*)-F_{\batch}(w_{t-1})\|^2]\\
    &+\frac{8\sigma^2}{|\batch|}(\frac{\ez}{\mz}+\frac{\ex}{\mx})+\frac{\mu\eta}{4}(\EE[\|\omega_{t+1/2}-\omega_t\|^2+\EE[\|\omega_{t}-\omega^*\|^2])\tag{Assumption \ref{cond:smooth_SVRE}}\\
    = & \expect{\|\omega_t - \omega^*\|^2} -(2\eta-20\bar{L}\eta^2-2\tau\bar{L}) \EE[(F(\omega_{t+1/2})-F(w^*))\trans(\omega_{t+1/2}-\omega^*)]\\
    &-(1-25\eta^2\bar{L}^2-2\tau\bar{L}^2)\EE[\|\omega_{t+1/2}-\omega_t\|^2] +10\eta^2 \EE[\|F_{\batch}(\omega^*)-F_{\batch}(w_{t-1})\|^2] \\
    &+\frac{8\sigma^2}{|\batch|}(\frac{\ez}{\mz}+\frac{\ex}{\mx})+\frac{\mu\eta}{4}(\EE[\|\omega_{t+1/2}-\omega_t\|^2+\EE[\|\omega_{t}-\omega^*\|^2])
\end{align*}

By Prop. \ref{prop:strong_mono}, we have:
\begin{align}
    (F(w^*)-F(\omega_{t+1/2}))\trans(\omega^*-\omega_{t+1/2})\geq \mu\|\omega^*-\omega_{t+1/2}\|^2\geq \frac{\mu}{2}\|w_t-\omega^*\|^2-\mu\|w_{t+1/2}-w_t\|^2\label{eq:strong_convexity}
\end{align}
By choosing $0<\eta_\zeta \leq \frac{1}{50\max\{\bar{L}_\zeta, \mu_\zeta\}}$, $0<\eta_\xi \leq \frac{1}{50\max\{\bar{L}_\xi, \mu_\xi\}}$, and $\tau=15\eta^2$, we know
\begin{align*}
    2\eta-20\bL\eta^2-2\tau\bL=2\eta-50\bL\eta^2\geq 0
\end{align*}
As a result, we can use \eqref{eq:strong_convexity} to get:

\begin{align*}
    P_{t+1}\leq& (1-\mu\eta+10\mu\eta^2\bar{L}+\tau\mu\bar{L}+\frac{\mu\eta}{4})\expect{\|\omega_t - \omega^*\|^2}\\
    &-(1-25\eta^2\bar{L}^2-2\tau\bar{L}^2-2\mu\eta+20\mu\bar{L}\eta^2+2\mu\tau\bar{L}-\frac{\mu\eta}{4})\EE[\|\omega_{t+1/2}-\omega_t\|^2]\\ 
    &+\frac{10\eta^2}{\tau}\tau\EE[\|F_{\batch}(\omega^*)-F_{\batch}(w_{t-1})\|^2]+\frac{8\sigma^2}{|\batch|}(\frac{\ez}{\mz}+\frac{\ex}{\mx})\\
    \leq&\underbrace{(1-\frac{3}{4}\mu\eta+25\mu\eta^2\bar{L})}_{p_1}\expect{\|\omega_t - \omega^*\|^2}+\underbrace{(55\eta^2\bar{L}^2+\frac{9}{4}\mu\eta-50\mu\bar{L}\eta^2-1)}_{p_2}\EE[\|\omega_{t+1/2}-\omega_t\|^2]\\ &+\frac{2}{3}\tau\EE[\|F_{\batch}(\omega^*)-F_{\batch}(w_{t-1})\|^2]+\frac{8\sigma^2}{|\batch|}(\frac{\ez}{\mz}+\frac{\ex}{\mx})\\
\end{align*}

since $0<\eta \mu \leq 1/50$ and $0<\eta\bar{L}  \leq 1/50$
\begin{align*}
    p_1 \leq& 1-\frac{3}{4}\mu\eta + \frac{25\mu\eta}{50}=1-\frac{\mu\eta}{4} \\
    p_2 \leq& \frac{11}{500}+\frac{9}{200}-1\leq \frac{11}{500}+\frac{9}{200}-1\leq 0
\end{align*}
As a result
\begin{align*}
    P_{t+1}\leq& (1-\frac{\mu\eta}{4})\EE[\|w_t-\omega^*\|^2]+\frac{2}{3}\tau \EE[\|F_{\batch}(\omega^*)-F_{\batch}(w_{t-1})\|^2]+\frac{8\sigma^2}{|\batch|}(\frac{\ez}{\mz}+\frac{\ex}{\mx})\\
    \leq& \Big(1-\min\{\frac{\mu\eta}{4}, \frac{1}{3}\}\Big)P_t+\frac{8\sigma^2}{|\batch|}(\frac{\ez}{\mz}+\frac{\ex}{\mx})\\
    =& \Big(1-\frac{\mu\eta}{4}\Big)P_t+\frac{8\sigma^2}{|\batch|}(\frac{\ez}{\mz}+\frac{\ex}{\mx})\tag{$\mu\eta\leq 1/50$}\\
    \leq & \Big(1-\frac{\mu\eta}{4}\Big)^{t}P_1+\frac{8\sigma^2}{\min\{\frac{\mz\ez}{4},\frac{\mx\ex}{4}\}|\batch|}(\frac{\ez}{\mz}+\frac{\ex}{\mx})\\
    =& \Big(1-\frac{\mu\eta}{4}\Big)^{t}(\EE[\|w_1-\omega^*\|]+\tau \EE[\|F_{\batch}(\omega^*)-F_{\batch}(\omega_0)\|^2])+\frac{8\sigma^2}{\min\{\frac{\mz\ez}{4},\frac{\mx\ex}{4}\}|\batch|}(\frac{\ez}{\mz}+\frac{\ex}{\mx})
\end{align*}
Next, we take a look at $\EE[\|w_1-\omega^*\|]$, from Lemma \ref{lem:lem3_reducing}, we have:
\begin{align*}
    &\EE[\|\omega_1-\omega^*\|^2]\\
    \leq& \EE[\|\omega_0-\omega^*\|^2-2\eta g_0\trans(\omega_1-\omega^*)]\\
    =&\EE[\|\omega_0-\omega^*\|^2]+2\eta \EE[(F(\omega_0)-g_0)\trans (\omega_1-\omega^*)]+2\eta \EE[(F(\omega^*)-F(\omega_0))\trans(\omega_1-\omega^*)]\\
    \leq&\EE[\|\omega_0-\omega^*\|^2]+2\eta \EE[\|F(\omega_0)-g_0\|^2]\EE[\|\omega_1-\omega^*\|^2]+2\eta \EE[(F(\omega^*)-F(\omega_0))\trans(\omega_1-\omega^*)]\\
    \leq&\EE[\|\omega_0-\omega^*\|^2]+\frac{2\eta \sigma^2}{|\batch|}\EE[\|\omega_1-\omega^*\|]+ 3\eta^2\EE[\|F(\omega^*)-F(\omega_0)\|^2]+ \frac{1}{3}\EE[\|\omega_1-\omega^*\|^2]\\
    \leq&\EE[\|\omega_0-\omega^*\|^2]+\frac{\mu\eta}{4}\EE[\|\omega_1-\omega^*\|]+3\eta^2\EE[\|F(\omega^*)-F(\omega_0)\|^2]+ \frac{1}{3}\EE[\|\omega_1-\omega^*\|^2]\\
    \leq&\EE[\|\omega_0-\omega^*\|^2]+\frac{1}{2}\EE[\|\omega_1-\omega^*\|]+ 3\eta^2\EE[\|F(\omega^*)-F(\omega_0)\|^2]\tag{$\eta\mu<1/50$}
\end{align*}
Therefore, 
$$
\EE[\|\omega_1-\omega\|^2]\leq 2\EE[\|\omega_0-\omega^*\|^2] + 6\eta^2\EE[\|F(\omega^*)-F(\omega_0)\|^2] 
$$
Finally, using the fact that $\EE[\|F_{\batch}(\omega^*)-F_{\batch}(\omega_0)\|^2]\leq\bar{L}^2\EE[\|\omega^*-\omega_0\|^2]$, we have
\begin{align*}
    &\EE[\|\omega_{t+1}-\omega^*\|^2] -\frac{8\sigma^2}{\min\{\frac{\mz\ez}{4},\frac{\mx\ex}{4}\}|\batch|}(\frac{\ez}{\mz}+\frac{\ex}{\mx})\leq P_{t+1}-\frac{8\sigma^2}{\min\{\frac{\mz\ez}{4},\frac{\mx\ex}{4}\}|\batch|}(\frac{\ez}{\mz}+\frac{\ex}{\mx})\\
    \leq &\Big(1-\frac{\mu\eta}{4}\Big)^{t}(2\EE[\|\omega_0-\omega^*\|^2] + 6\eta^2\EE[\|F(\omega^*)-F(\omega_0)\|^2] +\tau \EE[\|F_{\batch}(\omega^*)-F_{\batch}(\omega_0)\|^2])\\
    = & \Big(1-\frac{\mu\eta}{4}\Big)^{t}(2\EE[\|\omega_0-\omega^*\|^2] + 6\eta^2\bL^2\EE[\|\omega^*-\omega_0\|^2] +15\eta^2\bL^2 \EE[\|\omega^*-\omega_0\|^2])\\
    \leq&\Big(1-\frac{\mu\eta}{4}\Big)^{t}(2+\frac{6}{2500}+\frac{15}{2500})\EE[\|\omega^*-\omega_0\|^2]\\
    \leq & \frac{201}{100}\Big(1-\frac{\mu\eta}{4}\Big)^{t}\EE[\|\omega^*-\omega_0\|^2]
\end{align*}
which finishes the proof.
\end{proof}


\section{Missing details for Algorithm 2}\label{appx:proofs_alg2}
In the following, we will use $\cL^D_t$, $\cL^B_t$ and $\cL^{D*}_t$ as shortnotes of $\cL^D(\theta_t, \zeta_t, \xi_t)$, $\cL^B(\theta_t, \zeta_t, \xi_t)$ and $\cL^D(\theta_t, \zeta_t^*, \xi_t^*)$, where $\zeta_t^*, \xi_t^*$ is the only one saddle point of $\cL^D(\theta_t, \zeta, \xi)$. Besides, we use $\nt \cL^D_t$ and $\nt \cL^B_t$ as a shortnote of the gradient averaged over $d^D$ and the gradient averaged over batch, respectively.

\begin{lemma}\label{lem:sd_shift} 
Denote $(\zeta_1^*,\xi_1^*)$ and $(\zeta_2^*,\xi_2^*)$ as the saddle-point of $\max_{\zeta\in Z}\min_{\xi\in\Xi} \cL^D(\theta_1,\zeta,\xi)$ and $\max_{\zeta\in Z}\min_{\xi\in\Xi} \cL^D(\theta_2,\zeta,\xi)$ respectively.
Under Assumptions in Section \ref{sec:assumptions} and Condition \ref{prop:radius}, we have
\begin{align*}
    \|\zeta_1^*-\zeta_2^*\|\leq \kappa_\zeta(\kappa_\xi+1)\|\theta_1-\theta_2\|,\quad \|\xi_1^*-\xi_2^*\|\leq \kappa_\xi(\kappa_\zeta+1)\|\theta_1-\theta_2\|
\end{align*}
\end{lemma}
\begin{proof}
With Condition \ref{prop:radius}, we have
\begin{align}
    \|\nz \cL^D(\theta_2,\zeta_1^*,\xi_1^*)\|=\|\nz \cL^D(\theta_1,\zeta_1^*,\xi_1^*)-\nz \cL^D(\theta_2,\zeta_1^*,\xi_1^*)\|\leq L\|\theta_1-\theta_2\|\label{eq:bound_gz_with_dt}\\
    \|\nx \cL^D(\theta_2,\zeta_1^*,\xi_1^*)\| = \|\nx \cL^D(\theta_1,\zeta_1^*,\xi_1^*)-\nx \cL^D(\theta_2,\zeta_1^*,\xi_1^*)\|\leq L\|\theta_1-\theta_2\|
\end{align}
Recall in Lemma \ref{lem:optimum_function}, we know $\Phi_{\theta_2}(\zeta)$ should be a $\mu_\zeta$-strongly-concave function. Then, we have
\
\begin{align*}
    \|\zeta_1^*-\zeta_2^*\|\leq&\frac{1}{\mu_\zeta}\|\nz\Phi_{\theta_2}(\zeta_1^*)\|= \frac{1}{\mu_\zeta}\|\nz \cL^D(\theta_2, \zeta_1^*, \phi_{\theta_2}(\zeta_1^*))\|\\
    \leq & \frac{1}{\mu_\zeta}\|\nz \cL^D(\theta_2, \zeta_1^*, \phi_{\theta_2}(\zeta_1^*))-\nz \cL^D(\theta_2, \zeta_1^*, \xi_1^*)\| + \frac{1}{\mu_\zeta}\|\nz \cL^D(\theta_2, \zeta_1^*, \xi_1^*))\|\\
    \leq & \frac{1}{\mu_\zeta}\|\nz \cL^D(\theta_2, \zeta_1^*, \phi_{\theta_2}(\zeta_1^*))-\nz\cL^D(\theta_2, \zeta_1^*, \xi_1^*)\| + \frac{L}{\mu_\zeta}\|\theta_1-\theta_2\|\\
    \leq & \frac{L}{\mu_\zeta}\|\phi_{\theta_2}(\zeta_1^*)-\xi_1^*\|+\frac{L}{\mu_\zeta}\|\theta_1-\theta_2\|\\
    \leq & \frac{L}{\mu_\zeta\mu_\xi}\|\nx \cL^D(\theta_2, \zeta_1^*,\xi_1^*)\|+\frac{L}{\mu_\zeta}\|\theta_1-\theta_2\|\\
    \leq & \kappa_\zeta(\kappa_\xi+1)\|\theta_1-\theta_2\|
\end{align*}
where in the first step, we use Lemma \ref{lem:grad_norm}; in the fourth and fifth inequalities, we use the Property \ref{prop:scscs}; 
in the last inequality, we use Eq.\eqref{eq:bound_gz_with_dt} again.

We can give a similarly discussion for $\|\xi_1^*-\xi_2^*\|$:
\begin{align*}
    \|\xi_1^*-\xi_2^*\|\leq&\frac{1}{\mu_\xi}\|\nx\Psi_{\theta_2}(\xi_1^*)\|= \frac{1}{\mu_\xi}\|\nx \cL^D(\theta_2, \psi_{\theta_2}(\xi_1^*), \xi_1^*)\|\\
    \leq & \frac{1}{\mu_\xi}\|\nx \cL^D(\theta_2, \psi_{\theta_2}(\xi_1^*), \xi_1^*)-\nx \cL^D(\theta_2, \zeta_1^*, \xi_1^*)\| + \frac{1}{\mu_\xi}\|\nx \cL^D(\theta_2, \zeta_1^*, \xi_1^*))\|\\
    \leq & \frac{1}{\mu_\xi}\|\nx \cL^D(\theta_2, \psi_{\theta_2}(\xi_1^*), \xi_1^*)-\nx\cL^D(\theta_2, \zeta_1^*, \xi_1^*)\| + \frac{L}{\mu_\xi}\|\theta_1-\theta_2\|\\
    \leq & \frac{L}{\mu_\xi}\|\zeta_1^*-\psi_{\theta_2}(\xi_1^*)\|+\frac{L}{\mu_\xi}\|\theta_1-\theta_2\|\\
    \leq & \frac{L}{\mu_\xi\mu_\zeta}\|\nz \cL^D(\theta_2, \zeta_1^*,\xi_1^*)\|+\frac{L}{\mu_\xi}\|\theta_1-\theta_2\|\\
    \leq & \kappa_\xi(\kappa_\zeta+1)\|\theta_1-\theta_2\|
\end{align*}

\end{proof}

\RelateShift*
\begin{proof}
We will use $\Delta_t(\zeta,\xi)$ to denote $\EE[\|\zeta-\zeta_t^*\|^2+\|\xi-\xi_t^*\|^2]$. We first study some useful properties of $\Delta_t(\zeta,\xi)$. 
\paragraph{Property 1} For $t \geq 1$
\begin{align*}
    \Delta_t(\zeta_{t-1}^*,\xi_{t-1}^*)=&\expect{\|\zeta^*_{t}-\zeta_{t-1}^*\|^2+\|\xi^*_{t}-\xi_{t-1}^*\|^2}\leq \Czx\expect{\|\theta_{t}-\theta_{t-1}\|^2}=\et^2\Czx\expect{\| \gt^{t-1}\|^2}
\end{align*}
where in the inequality, we use Lemma \ref{lem:sd_shift}; and the last equality results from the update rule $\theta_t=\theta_{t-1}+\et \gt^{t-1}$

\paragraph{Property 2} For $t\geq 0$,
\begin{align*}
    \Delta_t(\zeta_t,\xi_t)\leq& \frac{\beta}{2} \Delta_t(\zeta_{t-1},\xi_{t-1})+\rc    =\frac{\beta}{2}\expect{\|\zeta_{t-1}-\zeta_t^*\|^2+\|\xi_{t-1}-\xi_t^*\|^2}+\rc\\
    \leq&\beta\expect{\|\zeta_{t-1}-\zeta_{t-1}^*\|^2+\|\xi_{t-1}-\xi_{t-1}^*\|^2+\|\zeta_{t}^*-\zeta_{t-1}^*\|^2+\|\xi^*_{t}-\xi_{t-1}^*\|^2}+\rc\\
    =&\beta\Delta_{t-1}(\zeta_{t-1},\xi_{t-1})+\beta\Delta_t(\zeta_{t-1}^*,\xi_{t-1}^*)+\rc\\
    \leq & \beta^t\Delta_0(\zeta_0,\xi_0)+\sum_{\tau=1}^{t}\beta^{t-\tau+1}\Delta_\tau(\zeta^*_{\tau-1},\xi^*_{\tau-1})+\sum_{\tau=0}^{t-1}\beta^{\tau}\rc\\
    \leq & \beta^{t+1}d^2+\et^2\Czx\sum_{\tau=0}^{t-1}\beta^{t-\tau}\expect{\| \gt^{\tau}\|^2}+\sum_{\tau=0}^{t}\beta^{\tau}\rc\\
    \leq&\beta^{t+1}d^2+\et^2\Czx\sum_{\tau=0}^{t-1}\beta^{t-\tau}\expect{\| \gt^{\tau}\|^2}+\frac{\rc}{1-\beta}
\end{align*}
where the first inequality is because of the property of the Oracle; for the second inequality we use Young's inequality; In the last step, we use
\begin{align*}
    \Delta_0(\zeta_0,\xi_0)=\EE[\|\zeta_0-\zeta^*_0\|^2+\|\xi_0-\xi^*_0\|^2]\leq \frac{\beta}{2}\EE[\|\zeta_{-1}-\zeta^*_0\|^2+\|\xi_{-1}-\xi^*_0\|^2]+\rc\leq \beta d^2+\rc
\end{align*}
With the two properties above, we can bound:
\begin{align*}
    &\expect{\|\zeta_{t+1}-\zeta_t\|^2+\|\xi_{t+1}-\xi_t\|^2}\\
    \leq &3\expect{\|\zeta_{t+1}-\zeta_{t+1}^*\|^2+\|\xi_{t+1}-\xi^*_{t+1}\|^2+\|\zeta^*_{t+1}-\zeta_t^*\|^2+\|\xi^*_{t+1}-\xi_t^*\|^2+\|\zeta^*_{t}-\zeta_t\|^2+\|\xi^*_{t}-\xi_t\|^2}\\
    =&3\Delta_{t+1}(\zeta_{t+1},\xi_{t+1})+3\Delta_{t+1}(\zeta_{t}^*,\xi_{t}^*)+3\Delta_{t}(\zeta_{t},\xi_{t})\\
    \leq&3\beta^{t+2}d^2+3\et^2\Czx\sum_{\tau=0}^{t}\beta^{t-\tau+1}\expect{\|\gt^\tau\|^2}+3\et^2\Czx\expect{\| \gt^{t}\|^2}+3\beta^{t+1}d^2+3\et^2\Czx\sum_{\tau=0}^{t-1}\beta^{t-\tau}\expect{\| \gt^\tau\|^2}\\
    &+\frac{6\rc}{1-\beta}\\
    =&3(1+\beta)\beta^{t+1}d^2+3\et^2\Czx\sum_{\tau=0}^{t}(1+\beta)\beta^{t-\tau}\expect{\|\gt^\tau\|^2}+\frac{6\rc}{1-\beta}\\
    \leq& 6\beta^{t+1}d^2+6\et^2\Czx\sum_{\tau=0}^{t}\beta^{t-\tau}\expect{\|\gt^\tau\|^2}+\frac{6\rc}{1-\beta}
\end{align*}
where for the first one we use an extended version of Young's inequality $\|\sum_{i=1}^k x_i\|^2\leq k\sum_{i=1}^k\|x_i\|^2$; in the second inequality, we use the Property 1 and 2 to give the upper bound; in the third inequality, we use the fact that $0<\beta\leq 1$. 
\end{proof}

\begin{lemma}\label{lem:diff_update_vs_truegrad}
Define $L^{D*}(\theta):=\max_{\zeta\in Z}\min_{\xi\in\Xi} \cL^D(\theta, \zeta,\xi)$.
Under the same condition of Lemma \ref{lem:recursive_split} above, with an additional constraint $\beta\leq (1-\alpha)^2/2$, for $t\geq 0$, we have:
\begin{align*}
    \EE[\|\gt^{t+1}-\nt \cL^{D*}(\theta_{t+1})\|^2]\leq&2\lambda^{t+1}\EE[\|\gt^0 -\nt\cL^D_0\|^2] +2\frac{\alpha^2\sigma^2}{(1-\lambda)|B|}\\
    &+4L^2 \Big(\beta^{t+2}d^2+\frac{\rc}{1-\beta}+\frac{3\lambda c}{(1-\beta)(1-\lambda)}+\frac{3d^2 \beta \lambda^{t+2}}{\lambda-\beta}\Big)\\
    &+2\et^2 L^2\Expect{\sum_{i=0}^{t} \Big(\lambda^{t-i+1}\Big(6\Czx \frac{\lambda}{\lambda-\beta}+1\Big)+2\Czx\beta^{t-i+1}\Big)\|g_\theta^i\|^2}
\end{align*}
where $\epsilon_{data},\epsilon_{func},\epsilon_{reg}$ are the same as those in Theorem \ref{thm:biasedness}, and $\lambda$ as a shortnote of $2(1-\alpha)^2$.
\end{lemma}

\begin{proof}
Recall that we will use $\nt \cL^B_t$, $\nt \cL^D_t$ and $\nt \cL^{D*}_t$ as a shortnote of $\nt \cL^B(\theta_t,\zeta_t,\xi_t)$, $\nt \cL^D(\theta_t,\zeta_t,\xi_t)$, $\nt \cL^{D*}(\theta_t)$ respectively, and in Property \ref{cond:variance}, we have already shown the variance of batch gradient can be bounded. First we can use the Young's inequality to obtain
\begin{align*}
    \EE[\|\gt^{t+1}-\nt\cL^{D*}_{t+1}\|^2]\leq 2\underbrace{\EE[\|\gt^{t+1}-\nt\cL^D_{t+1}\|^2]}_{p_1}+2\underbrace{\EE[\|\nt\cL^D_{t+1}-\nt\cL^{D*}_{t+1}\|^2]}_{p_2}\\
\end{align*}
Since the first term has already been bounded in Theorem \ref{thm:biasedness}. Next, we bound $p_1$ and $p_2$:

\paragraph{Upper bound $p_1$}
We again use $C_{\zeta,\xi}$ as a short note of $\kappa^2_\mu(\kappa_\xi+1)^2+\kappa^2_\xi(\kappa_\zeta+1)^2$. From Lemma \ref{lem:recursive_split}, we know that,
\begin{align*}
    \expect{\|\zeta_{t+1}-\zeta_t\|^2+\|\xi_{t+1}-\xi_t\|^2}\leq&6\beta^{t+1}d^2+6\et^2\Czx\sum_{\tau=0}^{t}\beta^{t-\tau}\expect{\|\gt^\tau\|^2}+\frac{6\rc}{1-\beta}
\end{align*}
In the following, we will use $\lambda$ as a shortnote of $2(1-\alpha)^2$.
\begin{align*}
    &p_1=\EE[\|\gt^{t+1}-\nt \cL^D_{t+1}\|^2]\\
    =&\EE\Big[\Big\|(1-\alpha)\gt^t+\alpha\nt\cL^B_{t+1}-\nt \cL^D_{t+1}\pm (1-\alpha) \nt \cL^D_t\Big\|^2\Big]\\
    =&\EE\Big[\Big\|(1-\alpha)(\gt^t-\nt \cL^D_t)+\alpha (\nt \cL^B_{t+1}-\nt \cL^D_{t+1})+(1-\alpha)(\nt \cL^D_{t}-\nt \cL^D_{t+1})\Big\|^2\Big]\\
    =&\EE\Big[\Big\|(1-\alpha)(\gt^t-\nt \cL^D_t)+(1-\alpha)(\nt \cL^D_{t}-\nt \cL^D_{t+1})\Big\|^2\Big]+\EE\Big[\Big\|\alpha(\nt \cL^B_{t+1}-\nt \cL^D_{t+1})\Big\|^2\Big]\tag{Drop 0 expectation}\\
    \leq &\lambda\EE\Big[\Big\|(\gt^t-\nt \cL^D_t)\Big\|^2\Big]+\lambda\EE\Big[\Big\|(\nt \cL^D_{t}-\nt \cL^D_{t+1})\Big\|^2\Big]+\frac{\alpha^2\sigma^2}{|B|}\tag{Young's Ineq.}\\
    \leq& \lambda\EE\Big[\Big\|(\gt^t-\nt \cL^D_t)\Big\|^2\Big] + \lambda L^2\EE[\|\zeta_{t+1}-\zeta_t\|^2+\|\xi_{t+1}-\xi_{t}\|^2+\|\theta_{t+1}-\theta_t\|^2]+\frac{\alpha^2\sigma^2}{|B|} \tag{Smoothness of $\cL^D$}\\
    \leq& \lambda^{t+1}\EE[\|\gt^0 -\nt\cL^D_0\|^2]+\frac{\alpha^2\sigma^2}{|B|}\frac{1-\lambda^{t+1}}{1-\lambda}+\Expect{\sum_{i=0}^{t}\lambda^{t-i+1} L^2\Big(\|\zeta_{i+1}-\zeta_i\|^2+\|\xi_{i+1}-\xi_i\|^2+\|\theta_{i+1}-\theta_i\|^2\Big)}\\
    \leq& \lambda^{t+1}\EE[\|\gt^0 -\nt\cL^D_0\|^2]+\frac{\alpha^2\sigma^2}{(1-\lambda)|B|}\\
    &+\Expect{\sum_{i=0}^{t}\lambda^{t-i+1} L^2\Big(6\beta^{i+1}d^2+6\et^2\Czx\sum_{\tau=0}^{i}\beta^{i-\tau}\expect{\|\gt^\tau\|^2}+\frac{6\rc}{1-\beta}+\et^2\|g_\theta^i\|^2\Big)}\\
    \leq& \lambda^{t+1}\EE[\|\gt^0 -\nt\cL^D_0\|^2]+\Expect{\sum_{i=0}^{t} \et^2 L^2\Big(6\Czx \sum_{\tau=i}^t\lambda^{t-\tau+1}\beta^{\tau-i}+\lambda^{t-i+1}\Big)\|g_\theta^i\|^2}\\
    &+\frac{\alpha^2\sigma^2}{(1-\lambda)|B|}+\frac{6\lambda cL^2}{(1-\beta)(1-\lambda)}+\frac{6L^2d^2 \beta \lambda^{t+2}}{\lambda-\beta}\\
    \leq& \lambda^{t+1}\EE[\|\gt^0 -\nt\cL^D_0\|^2]+\Expect{\sum_{i=0}^{t} \et^2 L^2\lambda^{t-i+1}\Big(6\Czx \frac{\lambda}{\lambda-\beta}+1\Big)\|g_\theta^i\|^2}+\frac{\alpha^2\sigma^2}{(1-\lambda)|B|}\\
    &+\frac{6\lambda cL^2}{(1-\beta)(1-\lambda)}+\frac{6L^2d^2 \beta \lambda^{t+2}}{\lambda-\beta}
\end{align*}
where the fourth equality because $\EE[\nt \cL^B_{t}]=\nt \cL^D_{t}$ holds for all $t$ and so the cross terms has 0 expectation; the first inequality is because variance is less than the second momentum; the second inequality we apply Assumption \ref{assump:smooth}; in the last but two inequality, we apply the summation formula of equal ratio sequence and use the fact that $0<\alpha\leq 1, \beta \leq 1$; in the last step, we use our condition $\beta \leq (1-\alpha)^2/2$

\paragraph{Upper bound $p_2$}
Next, we give an upper bound for $p_2$. From the Property 2 in Lemma \ref{lem:recursive_split}, we know that
\begin{align*}
    \Delta_{t+1}(\zeta_{t+1},\xi_{t+1})=&\expect{\|\zeta_{t+1}-\zeta^*_{t+1}\|^2}+ \expect{\|\xi_{t+1}-\xi^*_{t+1}\|^2}\\
    \leq& \beta^{t+2}d^2+\et^2\Czx\sum_{\tau=0}^{t}\beta^{t-\tau+1}\expect{\| \gt^{\tau}\|^2}+\frac{\rc}{1-\beta}
\end{align*}
As a result
\begin{align*}
    p_2 =& \EE[\|\nt\cL^D_{t+1}-\nt\cL^{D*}_{t+1}\|^2]\leq 2L^2\EE[\|\zeta_{t+1}-\zeta^*_{t+1}\|^2+\|\xi_{t+1}-\xi^*_{t+1}\|^2]\\
    \leq & 2L^2\Big(\beta^{t+2}d^2+\et^2\Czx\sum_{\tau=0}^{t}\beta^{t-\tau+1}\expect{\| \gt^{\tau}\|^2}+\frac{\rc}{1-\beta})
\end{align*}
Combine these two results we can finish the proof:

\begin{align*}
    &\EE[\|\gt^{t+1}-\nt \cL^{D*}(\theta_{t+1})\|^2]\leq 2p_1+2p_2\\
    \leq & 2\lambda^{t+1}\EE[\|\gt^0 -\nt\cL^D_0\|^2]+2\Expect{\sum_{i=0}^{t} \et^2 L^2\lambda^{t-i+1}\Big(6\Czx \frac{\lambda}{\lambda-\beta}+1\Big)\|g_\theta^i\|^2}\\
    &+2\frac{\alpha^2\sigma^2}{(1-\lambda)|B|}+\frac{12\lambda cL^2}{(1-\beta)(1-\lambda)}+\frac{12L^2d^2 \beta \lambda^{t+2}}{\lambda-\beta}+4L^2\Big(\beta^{t+2}d^2+\et^2\Czx\sum_{\tau=0}^{t}\beta^{t-\tau+1}\expect{\| \gt^{\tau}\|^2}+\frac{\rc}{1-\beta})\\
    \leq & 2\lambda^{t+1}\EE[\|\gt^0 -\nt\cL^D_0\|^2] +2\frac{\alpha^2\sigma^2}{(1-\lambda)|B|}+4L^2 \Big(\beta^{t+2}d^2+\frac{\rc}{1-\beta}+\frac{3\lambda c}{(1-\beta)(1-\lambda)}+\frac{3d^2 \beta \lambda^{t+2}}{\lambda-\beta}\Big)\\
    &+2\et^2 L^2\Expect{\sum_{i=0}^{t} \Big(\lambda^{t-i+1}\Big(6\Czx \frac{\lambda}{\lambda-\beta}+1\Big)+2\Czx\beta^{t-i+1}\Big)\|g_\theta^i\|^2}
\end{align*}
\end{proof}

\begin{restatable}{proposition}{PropSmoothJPi}\label{prop:smooth_LDstar}
    Under Assumption \ref{assump:bound_pi_grad} and \ref{assump:feature_matrices}, define $L^{D*}(\theta):=\max_{\zeta\in Z}\min_{\xi\in\Xi} \cL^D(\theta, \zeta,\xi)$ is $\Lzx$ smooth with $\Lzx:=L(1+\kappa_\zeta(\kappa_\xi+1)+\kappa_\zeta(\kappa_\zeta+1))$.
\end{restatable}    
\begin{proof}
For arbitrary $\theta_1,\theta_2\in\Theta$, denote $(\zeta_1^*,\xi^*_1)$ and $(\zeta_2^*,\xi^*_2)$ as the saddle-point of $\arg\max_{\zeta\in Z}\min_{\xi\in\Xi}\cL^D(\theta,\zeta,\xi)$ when $\theta=\theta_1$ and $\theta=\theta_2$, respectively, according to Lemma \ref{lem:sd_shift} we have:
\begin{align*}
    \|\nt L^{D*}(\theta_1) - \nt L^{D*}(\theta_2)\| \leq& L(\|\theta_1-\theta_2\|+\|\zeta_1-\zeta_2\|+\|\xi_1-\xi_2\|)\\
    \leq& L(1+\kappa_\zeta(\kappa_\xi+1)+\kappa_\zeta(\kappa_\zeta+1))\|\theta_1-\theta_2\|
\end{align*}
\end{proof}

\begin{restatable}{theorem}{AlgTwoConvergeRate}\label{thm:alg2_converage_rate}
    Under Assumption \ref{assump:bound_pi_grad}, \ref{assump:feature_matrices} and \ref{assump:dataset}, given arbitrary $\epsilon$, 
    Algorithm \ref{alg:SPIM_Oracle} will return us a policy $\pi_{\hat\theta}$, satisfying
    \begin{align*}
        \EE[\|\nt J(\pi_{\hat\theta})\|]\leq \epsilon+\epsilon_{reg} + \epsilon_{data}+\epsilon_{func}
    \end{align*}
    if the hyper-parameters in Alg. \ref{alg:SPIM_Oracle} satisfy the following constraints:
    \begin{align*}
        T=&\lceil\max \{96, \frac{16\Delta L^{D*} \Lzx}{\epsilon^2}, \frac{16\Delta L^{D*} L\sqrt{28\Czx+2}}{\epsilon^2 }\}\rceil=O(\epsilon^{-2}),\quad |B|=\lceil \frac{36\sigma^2}{\epsilon^2}]=O(\epsilon^{-2})\rceil,\\
        \alpha=0.5,&\quad \lambda=2(1-\alpha)^2=0.5,\quad\beta= \min\{\frac{\epsilon^2}{L^2d^2},(1-\lambda)^2,1/2\},\quad \et=\min\{\frac{1}{2\Lzx}, \sqrt{\frac{1-\lambda}{6L^2(14\Czx+1)}}\}.
    \end{align*}
    where $\Delta L^{D*}:=\max_{\theta\in\Theta}\cL^{D*}(\theta)-\cL^{D*}(\theta_0)$, $\Czx=\kappa^2_\mu(\kappa_\xi+1)^2+\kappa^2_\xi(\kappa_\zeta+1)^2$, and $c_{oracle}$ is an independent constant.
    
    Besides, the total gradient computation for O-SRM to obtain $\hat\theta$ should be 
    $O(\epsilon^{-4})$, if: 
    
    (1) either we choose Algorithm \ref{alg:PLSO} in Appendix \ref{appx:PLSO} as the \textsc{Oracle} and set the hyper-parameter $|N_{all}|$ in Algorithm \ref{alg:PLSO} to be
    \begin{align*}
        |N_{all}|=480\frac{d(2L^2+1)C_{LS}^2}{\epsilon^2}=O(\epsilon^{-2})
    \end{align*}
    where $C_{LS}$ is a constant defined in Appendix \ref{appx:PLSO},

    (2) or we choose Algorithm \ref{alg:SVREB} in Appendix \ref{appx:SVRE_Oracle} as the \textsc{Oracle} and set the hyper-parameters in Algorithm \ref{alg:SVREB} to be:
    \begin{align*}
        &|N|=\lceil\frac{240(L^2+1)\sigma^2}{\min\{\frac{\mz\ez}{4},\frac{\mx\ex}{4}\}\epsilon^2}(\frac{\ez}{\mz}+\frac{\ex}{\mx})\rceil,\quad K=c_{oracle}\log(\frac{1}{\beta}),\\
        &\ez\leq \frac{1}{50 \max\{\bar{L}_\zeta, \mu_\zeta\}},\quad\ex\leq \frac{1}{50 \max\{\bar{L}_\xi, \mu_\xi\}}.
    \end{align*}
    where $\bar{L}_\zeta$ and $\bar{L}_\xi$ are defined in Condition \ref{cond:smooth_SVRE}.
\end{restatable}

\begin{proof}
    Recall the definition of $\cL^{D*}$ in Prop \ref{prop:smooth_LDstar} and $\cL^{D*}$ is $L+C_{\zeta,\xi}$-smooth.
\begin{align*}
    L^{D*}(\theta_{T+1})=& L^{D*}(\theta_T+\et \gt^T)\\
    \geq & L^{D*}(\theta_T)+\et (\gt^T)\trans \nt L^{D*}(\theta_T)-\frac{\et^2 \Lzx}{2}\|\gt^T\|^2\\
    =&L^{D*}(\theta_T)+\frac{\et}{2}\|\nt L^{D*}(\theta_T)\|^2-\frac{\et}{2}\|\gt^T-\nt L^{D*}(\theta_T)\|^2+(\frac{\et}{2}-\frac{\et^2 \Lzx}{2})\|\gt^T\|^2\\
    \geq& L^{D*}(\theta_T)+\frac{\et}{2}\|\nt L^{D*}(\theta_T)\|^2-\frac{\et}{2}\|\gt^T-\nt L^{D*}(\theta_T)\|^2+\frac{\et}{4}\|\gt^T\|^2\\
    \geq & L^{D*}(\theta_0)+\frac{\eta_\theta}{2}\sum_{t=0}^T\|\nt L^{D*}(\theta_T)\|^2-\frac{\eta_\theta}{2}\Big(\underbrace{\sum_{t=0}^T\|\gt^t -\nt L^{D*}(\theta_T)\|^2-\frac{1}{2}\|\gt^t\|^2\Big)}_{p}
\end{align*}
where in the second equation, we use the fact that $(\gt^T)\trans \nt L^{D*}(\theta_T)=\frac{1}{2}\|\nt L^{D*}(\theta_T)\|^2+\frac{1}{2}\|\gt^T\|^2-\frac{1}{2}\|\gt^T-\nt L^{D*}(\theta_T)\|^2$; in the second inequality, we add a constraint for $\et$ that $\et\leq \frac{1}{2\Lzx}$;

Next, we give a upper bound for $p$ with Lemma \ref{lem:diff_update_vs_truegrad}:

\begin{align*}
    &p=\sum_{t=0}^T\|\gt^\tau -\nt \cL^{D*}(\theta_t)\|^2-\frac{1}{2}\|\gt^t\|^2\\
    \leq& \sum_{t=0}^T\Big\{2\lambda^{t+1}\EE[\|\gt^0 -\nt\cL^D_0\|^2] +2\frac{\alpha^2\sigma^2}{(1-\lambda)|B|}+4L^2 \Big(\beta^{t+2}d^2+\frac{\rc}{1-\beta}+\frac{3\lambda c}{(1-\beta)(1-\lambda)}+\frac{3d^2 \beta \lambda^{t+2}}{\lambda-\beta}\Big)\\
    &+2\et^2 L^2\Expect{\sum_{i=0}^{t} \Big(\lambda^{t-i+1}\Big(6\Czx \frac{\lambda}{\lambda-\beta}+1\Big)+2\Czx\beta^{t-i+1}\Big)\|g_\theta^i\|^2}-\frac{1}{2}\expect{\|\gt^t\|^2}\Big\}\\
    \leq& 2(T+1)\frac{\alpha^2\sigma^2}{(1-\lambda)|B|}+4(T+1)\rc L^2(\frac{1}{1-\beta}+\frac{3\lambda}{(1-\beta)(1-\lambda)})+\frac{2(T+1)}{1-\lambda}\EE[\|\gt^0 -\nt\cL^D_0\|^2]\\
        &+4L^2d^2(\frac{\beta^2}{1-\beta}+\frac{3\beta\lambda^2}{(1-\lambda)(\lambda-\beta)})+\sum_{t=0}^T \expect{\|\gt^t\|^2} \Big\{-\frac{1}{2}+2\et^2 L^2\sum_{i=0}^{T-t+1} \Expect{\lambda^{i}\Big(6\Czx \frac{\lambda}{\lambda-\beta}+1\Big)+2\Czx\beta^{i}}\Big\}\\
    \leq& 2(T+1)\frac{\alpha^2\sigma^2}{(1-\lambda)|B|}+4(T+1)\rc L^2(\frac{1}{1-\beta}+\frac{3\lambda}{(1-\beta)(1-\lambda)})+\frac{2(T+1)}{1-\lambda}\EE[\|\gt^0 -\nt\cL^D_0\|^2]\\
            &+4L^2d^2(\frac{\beta^2}{1-\beta}+\frac{3\beta\lambda^2}{(1-\lambda)(\lambda-\beta)})+\sum_{t=0}^T \expect{\|\gt^t\|^2} \Big\{-\frac{1}{2}+2\et^2 L^2\Big(\frac{1}{1-\lambda}\Big(6\Czx \frac{\lambda}{\lambda-\beta}+1\Big)+2\Czx\frac{1}{1-\beta}\Big)\Big\}\\
    \leq& 2(T+1)\frac{\alpha^2\sigma^2}{(1-\lambda)|B|}+4(T+1)\rc L^2(\frac{1}{1-\beta}+\frac{3\lambda}{(1-\beta)(1-\lambda)})+\frac{2(T+1)}{1-\lambda}\EE[\|\gt^0 -\nt\cL^D_0\|^2]\\
            &+4L^2d^2(\frac{\beta^2}{1-\beta}+\frac{6\beta\lambda}{1-\lambda})+\sum_{t=0}^T \expect{\|\gt^t\|^2} \Big\{-\frac{1}{2}+\frac{2\et^2 L^2}{1-\lambda}\Big(14\Czx+1\Big)\Big\}\tag{$\beta\leq (1-\alpha)^2=\frac{\lambda}{2}\leq \lambda$}\\
    \leq& 2(T+1)\frac{\alpha^2\sigma^2}{(1-\lambda)|B|}+4(T+1)\rc L^2(\frac{1}{1-\beta}+\frac{3\lambda}{(1-\beta)(1-\lambda)})+\frac{2(T+1)}{1-\lambda}\EE[\|\gt^0 -\nt\cL^D_0\|^2]\\
            &+4L^2d^2(\frac{\beta^2}{1-\beta}+\frac{6\beta\lambda}{1-\lambda}).
\end{align*}

In the fourth inequality, we add the following constraint to drop the terms containing $\|\gt\|$:
\begin{align}
\et\leq \sqrt{\frac{1-\lambda}{4L^2(14\Czx+1)}}.\label{eq:et_constraint2}
\end{align}
Therefore,

\begin{align*}
    \frac{1}{T+1}\sum_{t=0}^T\|\nt \cL^{D*}(\theta_\tau)\|^2    \leq& \frac{2}{(T+1)\et}(\cL^{D*}(\theta_T)-\cL^{D*}(\theta_0))+\frac{1}{T+1}\sum_{\tau=0}^T\Big(\|\gt^\tau -\nt \cL^{D*}(\theta_\tau)\|^2-\frac{1}{2}\|\gt^\tau\|^2\Big)\\
    \leq& {4\rc L^2(\frac{1}{1-\beta}+\frac{3\lambda}{(1-\beta)(1-\lambda)})}+{\frac{2\Delta\cL^{D*}}{(T+1)\et}}\\
    &+{\frac{2}{1-\lambda}\EE[\|\gt^0 -\nt\cL^D_0\|^2]}+{2\frac{\alpha^2\sigma^2}{(1-\lambda)|B|}}+{4L^2d^2(\frac{\beta^2}{1-\beta}+\frac{6\beta\lambda}{1-\lambda})}\\
    \leq& \underbrace{4\rc L^2\frac{(6\lambda+1)}{1-\lambda}}_{p_0}+\underbrace{\frac{2\Delta\cL^{D*}}{(T+1)\et}}_{p_1}+\underbrace{\frac{2(1+\alpha^2)}{1-\lambda}\frac{\sigma^2}{|B|}}_{p_2}+\underbrace{24L^2d^2\frac{\beta\lambda}{(T+1)(1-\lambda)}}_{p_3}. \tag{$\beta\leq 1/2,\beta\leq \lambda$, $\EE[g^0_\theta]=\cL^D_0$}
\end{align*}

\paragraph{The choice of $\alpha$}
In order to guarantee that $\lambda =2(1-\alpha)^2 < 1$ and $\alpha < 1$, we first need $\alpha\in(1-\frac{\sqrt{2}}{2}, 1)$. Easy to observe that, the coefficient of $p_2$:
\begin{align*}
    \frac{1+\alpha^2}{1-\lambda}=\frac{1+\alpha^2}{1-2(1-\alpha)^2}=\frac{1/2+2\alpha}{-1+4\alpha-2\alpha^2}-1/2.
\end{align*}
By taking the derivative w.r.t. $\alpha$, we have:
$$
(\frac{1+\alpha^2}{1-2(1-\alpha)^2})'=\frac{2(-1+4\alpha-2\alpha^2)-(1/2+2\alpha)(4-4\alpha)}{(-1+4\alpha-2\alpha^2)^2}=\frac{4\alpha^2+2\alpha-4}{(-1+4\alpha-2\alpha^2)^2}.
$$
We can observe easily that the zero point of the first-order derivative should located in interval $(\frac{1}{2},1)$.
Besides, $p_0,p_1,p_3$ are decreasing as $\alpha$ approaching 1 (i.e. $\lambda$ approaching 0). Therefore, the optimal choice of $\alpha$ should locates in $(\frac{1}{2},1)$.

As we can see, the introduction of $\alpha$ is to balancing the bias of momentum term and the variance in the current step to provide a better estimation of the gradient in current step. In the following, for simplicity, we directly choose $\alpha=1/2$.

Next, we want to carefully choose other hyper-parameters to make sure $p_0, p_1, p_2, p_3\leq \epsilon^2/4$. We consider $\beta\leq \min\{\frac{\epsilon^2}{L^2d^2},(1-\lambda)^2,1/2\}$. 
\paragraph{Control $p_0$}
Since $\beta<1/2$, we know
\begin{align*}
    p_0\leq 32\rc L^2.
\end{align*}
To make sure $p_0\leq \epsilon^2/4$, we need
\begin{align*}
    \rc \leq \frac{\epsilon^2}{128 L^2}.
\end{align*}
If we choose Algorithm \ref{alg:PLSO} as \textsc{Oracle}, the hypyer-parameter $|N|$ (total sample size) in Algorithm \ref{alg:PLSO} should be set to:
\begin{align*}
    |N|\geq 20d\frac{C_{LS}^2}{\rc} = O(\epsilon^{-2})
\end{align*}
If we choose Algorithm \ref{alg:SPIM_Oracle}, the hyper-parameter $|N|$ (mini batch size) and $K$ (total iteration) in Algorithm \ref{alg:PLSO} should be set to:
\begin{align*}
    \frac{8\sigma^2}{\min\{\frac{\mz\ez}{4},\frac{\mx\ex}{4}\}|\batch|}(\frac{\ez}{\mz}+\frac{\ex}{\mx})
    \leq  \frac{\epsilon^2}{128L^2},~~~~K=c_{oracle}\log(1/\beta)=O(\log\frac{1}{\epsilon}).
\end{align*}
where $c_{oracle}$ is an independent constant. 
Therefore,
\begin{align*}
    |\batch|\geq  \frac{240(L^2+1)\sigma^2}{\min\{\frac{\mz\ez}{4},\frac{\mx\ex}{4}\}\epsilon^2}(\frac{\ez}{\mz}+\frac{\ex}{\mx})=O(\epsilon^{-2}).
\end{align*}

\paragraph{Control $p_1$}
Since we have two constrains on $\et$: $\et\leq\frac{1}{2\Lzx}$ and $\et\leq \sqrt{\frac{1-\lambda}{4L^2(14\Czx+1)}}$. To make sure $p_1\leq \frac{\epsilon^2}{4}$, we need,
\begin{align*}
    \frac{2\Delta L^{D*}}{(T+1)}\max\{\sqrt{\frac{4L^2(14\Czx+1)}{1-\lambda}}, 2\Lzx\} \leq \frac{\epsilon^2}{4}.
\end{align*}
Therefore, we requirement
\begin{equation}
    T \geq \max\{\frac{16\Delta L^{D*} \Lzx}{\epsilon^2}, \frac{16\Delta L^{D*} L\sqrt{28\Czx+2}}{\epsilon^2 }\}.\label{eq:control_p1_1}
\end{equation} 

\paragraph{Control $p_2$}
Recall that we choose $\alpha=1/2$, we need:
\begin{align*}
    p_2=\frac{2(1+\alpha)^2}{1-\lambda}\frac{\sigma^2}{|B|}=9\frac{\sigma^2}{|B|}\leq \frac{\epsilon^2}{4}.
\end{align*}
which requires:
\begin{align}
    |B|\geq \frac{36\sigma^2}{\epsilon^2}\label{eq:control_p2}
\end{align}

\paragraph{Control $p_3$}
Because $K$ (iteration number of $Oracle$ Algorithm) only depends on $\log \frac{1}{\beta}$, we choose $\beta \leq \frac{\epsilon^2}{L^2d^2}$. Then, as long as $T \geq 168$, we have:
\begin{align}
    p_3=24L^2d^2\frac{\beta\lambda}{(T+1)(1-\lambda)}=&\frac{24\beta L^2d^2}{T+1}\leq \frac{\epsilon^2}{4}\label{eq:control_p3}
\end{align}
Combine \eqref{eq:control_p1_1}-\eqref{eq:control_p3}, we need
\begin{align*}
    T\geq \max \{96, \frac{16\Delta L^{D*} \Lzx}{\epsilon^2}, \frac{16\Delta L^{D*} L\sqrt{28\Czx+2}}{\epsilon^2 }\},\quad |B|\geq\frac{36\sigma^2}{\epsilon^2};
\end{align*}
As a result, if we use Algorithm \ref{alg:PLSO} as \textsc{Oracle}, the total computation before obtaining $\theta_{out}$ should be
\begin{align*}
    |B|\cdot(T+1) + |\batch|\cdot T=O(\epsilon^{-2})+O(\epsilon^{-4})+O(\epsilon^{-2}) \cdot O(\epsilon^{-2})=O(\epsilon^{-4})
\end{align*}
and if we use Algorithm \ref{alg:SVREB} as \textsc{Oracle}, it should be
\begin{align*}
    |B|\cdot (T+1) + |\batch|\cdot K\cdot T=O(\epsilon^{-2})+O(\epsilon^{-4})+O(\epsilon^{-2})\cdot O(\log\frac{1}{\epsilon}) \cdot O(\epsilon^{-2})=O(\epsilon^{-4})
\end{align*}
where we omit $\log$ terms in $O(\cdot)$.

\paragraph{Guarantee of $\EE[\|\nt J(\pi_{\hat\theta})\|]$}
By applying the hyper-parameters discussed above, we have:
\begin{align*}
    \EE[\|\nt J(\pi_{\hat\theta})\|]\leq&\EE[\|\nt \cL^{D*}(\hat\theta)\|]+\epsilon_{reg}+\epsilon_{func}+\epsilon_{data} \\
    \leq & \sqrt{\EE[\|\nt \cL^{D*}(\hat\theta)\|^2]}+\epsilon_{reg}+\epsilon_{func}+\epsilon_{data}\\
    =&\sqrt{\frac{1}{T+1}\sum_{t=0}^{T}\EE[\|\nt \cL^{D*}(\theta_t)\|^2]}+\epsilon_{reg}+\epsilon_{func}+\epsilon_{data}\\
    \leq& \epsilon+\epsilon_{reg}+\epsilon_{func}+\epsilon_{data}
\end{align*}
\end{proof}
\end{document}